\documentclass{article}

\usepackage[preprint]{neurips_2026}

\usepackage[utf8]{inputenc} % allow utf-8 input
\usepackage[T1]{fontenc}    % use 8-bit T1 fonts
\usepackage[colorlinks = True, linkcolor = violet, citecolor = blue!50!white]{hyperref}
\usepackage{url}            % simple URL typesetting
\usepackage{booktabs}       % professional-quality tables
\usepackage{amsfonts}       % blackboard math symbols
\usepackage{nicefrac}       % compact symbols for 1/2, etc.
\usepackage{microtype}      % microtypography
\usepackage{xcolor}         % colors
\usepackage{algorithm}
\usepackage{algorithmic}
\usepackage{wrapfig}
\usepackage{caption}
\usepackage{comment}

\usepackage{amsmath}
\usepackage{amssymb}
\usepackage{mathtools}
\usepackage{amsthm}
\usepackage{thm-restate}
\usepackage{pifont}
\usepackage{bbm}
\usepackage[capitalize,noabbrev]{cleveref}

\usepackage{colortbl}
\usepackage{tabularx}
\usepackage{tcolorbox}
\tcbuselibrary{skins}

\usepackage{tikz}
\usetikzlibrary{positioning} % add "quotes" only if needed

\tcbset{
  tab2/.style={
    enhanced,
    fonttitle=\bfseries,
    colback=white,
    colframe=gray,
    colbacktitle=white,
    coltitle=black,
    center title
  }
}

\newtcolorbox{mybox}[1]{colback=lb!5!white,colframe=lb!70!black,fonttitle=\bfseries,title=#1}
\newtcolorbox{resbox}{standard jigsaw,opacityback=0,boxrule=1pt,top=0.9ex,bottom=0.9ex,
  colbacktitle=teal!10!white,colframe=teal!70!white}

\usepackage{empheq}
\newtcbox{\mymath}[1][]{nobeforeafter, tcbox raise base, colframe=teal!60!black,
  colback=teal!10, boxrule=1pt, #1}

\theoremstyle{plain}
\newtheorem{theorem}{Theorem}[section]

\newtheorem{lemma}[theorem]{Lemma}
\newtheorem{corollary}[theorem]{Corollary}
\theoremstyle{definition}

\newtheorem{assumption}[theorem]{Assumption}

\theoremstyle{remark}

\newcommand\x{{\bf{x}}}
\newcommand\y{{\bf{y}}}
\newcommand\z{{\bf{z}}}
\newcommand\vv{{\bf{v}}}
\newcommand\kk{{\bf{k}}}

\newcommand\w{{\bf{w}}}

\newcommand\X{{\bf{X}}}
\newcommand\K{{\bf{K}}}
\newcommand\LL{{\bf{L}}}

\newcommand\I{{\bf{I}}}

\newcommand\blambda{\boldsymbol{\lambda}}

\newcommand\btheta{{\boldsymbol{\theta}}}

\newcommand\bTheta{{\boldsymbol{\Theta}}}
\newcommand\bseta{{\boldsymbol{\eta}}}

\newcommand\indicator{\mathbbm{1}}

\newcommand\normal{{\mathcal{N}}}
\newcommand\dataset{{\mathcal{D}}}
\newcommand\hilbert{{\mathcal{H}}}
\newcommand\calE{{\mathcal{E}}}
\newcommand\calS{{\mathcal{S}}}

\newcommand\calR{{\mathcal{R}}}
\newcommand\calL{{\mathcal{L}}}
\newcommand\calW{{\mathcal{W}}}
\newcommand\calX{{\mathcal{X}}}
\newcommand\calF{{\mathcal{F}}}
\newcommand\calY{{\mathcal{Y}}}

\newcommand\bigo{{\mathcal{O}}}

\newcommand\real{{\mathbb{R}}}

\newcommand\expect{{\mathbb{E}}}
\newcommand\probability{{\mathbb{P}}}

\newcommand\argmax{{\mathrm{argmax}}}
\newcommand\sharpness{{\mathrm{s}}}
\newcommand\calibration{{\mathrm{c}}}

\newcommand{\algcomment}[1]{%
  \item[]\textcolor{gray}{\# #1}%
}
\newcommand\noisevar{{\sigma^2}}

\newcommand{\inner}[2]{\langle #1, #2 \rangle}
\newcommand{\ignore}[1]{}

\definecolor{violetframe}{HTML}{7B3FD2}

\newcolumntype{Y}[1]{>{\hsize=#1\hsize\linewidth=\hsize\raggedright\arraybackslash}X}

\newcommand{\overviewref}[1]{{\color{violetframe}\bfseries\mbox{#1}}}

\newcommand{\overviewitem}[2]{%
  \noindent\overviewref{#2:}\enspace #1\par\vspace{-0.18ex}%
}

\newcommand{\overviewblock}[2]{%
  {\bfseries #1}\par
  \vspace{-0.05ex}
  {\color{violetframe}\rule{\linewidth}{0.42pt}}\par
  \vspace{0.02ex}
  \raggedright #2%
}

\usepackage[textsize=tiny]{todonotes}

\title{Online Sharp-Calibrated Bayesian Optimization}

\author{%
  Marshal Sinaga$^{1, 2}$ \thanks{Equal contribution.}
  \quad
  Julien Martinelli$^{1, 2}$ \footnotemark[1]
  \quad
  Teemu Turpeinen$^{2}$
  \quad  
  Samuel Kaski$^{1, 2, 3}$
  \\[0.5em]
$^{1}$ELLIS Institute Finland
\quad
$^{2}$Aalto University
\quad
$^{3}$University of Manchester
\\[0.5em]
\texttt{\{marshal.sinaga, julien.martinelli, teemu.e.turpeinen, samuel.kaski\}@aalto.fi}
    }
  
\begin{document}

\maketitle

\begin{abstract}
  Bayesian optimization (BO) is a widely used framework for optimizing expensive black-box functions, commonly based on Gaussian process (GP) surrogate models. Its effectiveness relies on uncertainty quantification that is both \emph{sharp} (informative) and \emph{well-calibrated} along the BO trajectory. In practice, GP kernel hyperparameters are unknown and are refit online from sequentially collected (non-i.i.d.) data, which can yield miscalibrated or overly conservative uncertainty and lies outside the fixed-kernel assumptions of standard BO regret theory. We propose \emph{Online Sharp-Calibrated Bayesian Optimization} (OSCBO),
a BO algorithm that adaptively balances GP sharpness and calibration by casting hyperparameter selection as a constrained online-learning problem. We also show that OSCBO preserves sublinear regret bounds by leveraging the theoretical guarantees of the underlying online learning algorithm. Empirically, OSCBO performs competitively across synthetic and real-world benchmarks, ranking among the strongest methods in final simple regret while maintaining robust cumulative-regret behavior.
\end{abstract}

\section{Introduction}\label{sec:introduction}

Many modern design problems reduce to optimizing an expensive black-box function under a tight evaluation budget.
Bayesian optimization (BO) addresses this setting by maintaining a probabilistic surrogate model, typically a Gaussian process (GP) with a chosen kernel, whose posterior mean and uncertainty are combined by an acquisition function to select new evaluations \citep{garnett2023bayesian}.
This uncertainty-driven sampling has enabled efficient optimization across diverse domains, including materials, biology, and robotics~\cite{chitturi2024targeted, stanton2022accelerating}.

The central theoretical lens in BO is regret, which measures how quickly the algorithm approaches the global optimum as the evaluation budget grows.
For GP-based BO, regret analyses combine high-probability GP uncertainty bounds with the exploration–exploitation behavior of policies like GP-UCB \citep{srinivas2009gaussian}.
Yet, these results almost always assume a fixed GP kernel throughout the BO loop, or equivalently, a fixed Reproducing Kernel Hilbert Space (RKHS) induced by kernel hyperparameter.
In practice, the hyperparameter is refit online (e.g., \emph{via} marginal likelihood maximization). This changes both the uncertainty estimates and the implied function class, creating a mismatch between standard BO practice and existing regret theory.

The widespread practice of refitting the GP hyperparameter at every BO iteration highlights that hyperparameter selection is a core part of the optimization loop. Since the hyperparameter shapes the posterior uncertainty, it directly determines acquisition values and thus the next evaluations.
Building on the view that calibration is central in sequential decision making~\citep{deshpande2024calibratedregressionadversaryregret}, we recast the classical exploration–exploitation principle as a tradeoff between \emph{calibration} and \emph{sharpness} of predictive uncertainty. Calibration asks whether GP confidence intervals attain their nominal coverage along the adaptive query sequence, quantified through a coverage gap, while sharpness measures confidence-interval width. Overly conservative uncertainty can remain calibrated but slow down exploitation; overly confident uncertainty is sharp yet miscalibrated and can trigger premature exploitation.
Consequently, explicitly balancing calibration and sharpness during online hyperparameter adaptation provides a principled way to avoid both miscalibration and excessive conservatism when experiments are costly.
%Consequently, explicitly balancing calibration and sharpness during online hyperparameter adaptation is essential, as both miscalibration and excessive conservatism can waste budget when experiments are costly \citep{stanton2023bayesian}.

A key complication is that BO data are non-i.i.d.: query points are chosen adaptively from past observations and the current model.
Under i.i.d. sampling, marginal-likelihood hyperparameter fitting can be interpreted as optimizing a proper scoring rule, aligning with the calibration–sharpness paradigm for probabilistic forecasts \citep{gneiting2007probabilistic}.
Under adaptive querying, this alignment can fail: the sampling distribution shifts with the algorithm, so neither standard hyperparameter fitting nor i.i.d.-based recalibration guarantees reliable uncertainty along the optimization trajectory \citep{deshpande2024calibratedregressionadversaryregret}. This motivates treating hyperparameter selection as part of the sequential decision problem, with explicit control of calibration and sharpness under adaptive data collection.

\textbf{Contributions.}
We propose Online Sharp-Calibrated Bayesian Optimization (OSCBO),
a method for adapting GP kernel hyperparameters online by trading off sharpness and calibration under adaptive, non-i.i.d. data collection.
Building on online learning with long-term constraints~\citep{castiglioni2022unifying}, OSCBO frames hyperparameter selection as a constrained online problem that promotes narrow confidence intervals while controlling cumulative coverage violations. When paired with the upper confidence bound (UCB) strategy, OSCBO retains a sublinear regret guarantee.
Our contributions are:
\begin{itemize}
    \item \textbf{Conceptual}: We reinterpret exploration--exploitation and the resulting regret guarantees through a sharpness--calibration lens.
    
    \item \textbf{Methodological}: We propose OSCBO, a principled online hyperparameter selection scheme. It actively promotes sharp uncertainty quantification while ensuring that long-run miscalibration does not exceed a specified threshold under non-i.i.d. sampling.
    
    \item \textbf{Theoretical}: We show that OSCBO achieves a sublinear regret bound by leveraging the regret guarantee of the underlying online algorithm.
    
    \item \textbf{Empirical}: We evaluate OSCBO on synthetic and real-world benchmarks, showing that it ranks among the best methods in final simple regret and remains robust in cumulative regret.
\end{itemize}
%
% \Cref{fig:gist} provides intuition for online hyperparameter refitting in BO and an overview of our theoretical contributions.

% --- colors for the roadmap box ---
\definecolor{violetframe}{RGB}{120,70,200}   % frame + rules
\definecolor{violetlight}{RGB}{155,120,220}  % right-column references

% optional: a tiny helper for right-column refs
\newcommand{\thref}[1]{\textcolor{violetlight}{#1}}

\begin{figure*}[h]
\centering
  \includegraphics[width=1\linewidth]{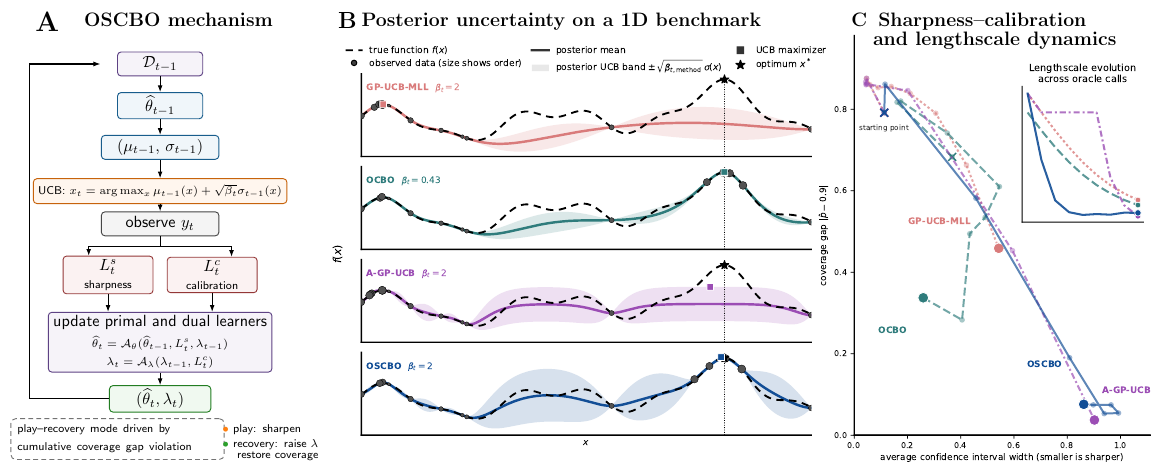}

\begin{tcolorbox}[
  enhanced,
  width=\linewidth,
  fontupper=\scriptsize,
  colback=white,
  colframe=violetframe,
  boxrule=0.8pt,
  arc=1pt,
  outer arc=1pt,
  left=0.45ex,
  right=0.45ex,
  top=0.18ex,
  bottom=0.01ex,
  boxsep=0.02ex,
]
\setlength{\tabcolsep}{0.45em}
\renewcommand{\arraystretch}{0.72}

\begin{tabularx}{\linewidth}{@{}
Y{0.95}
!{\color{violetframe}\vrule width 0.38pt}
Y{1.05}
!{\color{violetframe}\vrule width 0.38pt}
Y{1.00}
@{}}

\overviewblock{Sharpness--calibration lens on GP-UCB}{
  \overviewitem{Regret controlled }{Cor.~\ref{cor:sharpness}}
  by confidence width
  
  \overviewitem{Function coverage analogue}{Cor.~\ref{corr:calibration}}
}
&
\overviewblock{OSCBO method}{
  \overviewitem{Sharpness term}{Sec.~\ref{subsec:oscbo_losses}}
   and calibration constraint
   
  \overviewitem{Online primal--dual}{Sec.~\ref{subsec:oscbo_algo}}
   hyperparameters update
   
  \overviewitem{BO coupling and OSCBO procedure}{Alg.~\ref{alg:oscbo}}
}
&
\overviewblock{Theoretical guarantees}{
  \overviewitem{Uniform confidence}
  {Thm.~\ref{lemma:slater-condition}}
   and feasibility margin
  \overviewitem{Online-learning gap}{Lem.~\ref{lemma:convergence}}
   and violation bounds
   
  \overviewitem{OSCBO regret with GP-UCB}{Thm.~\ref{theorem:regret-bound}}
}

\end{tabularx}
\vspace{-1ex}
\end{tcolorbox}
\vspace{-1ex}
\captionsetup{skip=1pt}
% \caption{
% \textbf{Overview of OSCBO and main theoretical contributions.} (A) OSCBO updates GP hyperparameters online through a sharpness--calibration primal--dual loop coupled to GP-UCB. (B) On a 1D example, GP-UCB-MLL can be overconfident and OCBO conservative, while OSCBO maintains calibrated, sharper uncertainty for query selection. (C) OSCBO moves toward the sharpness--calibration frontier, balancing narrow intervals with empirical coverage.
% }
\caption{
\textbf{Overview of OSCBO and main theoretical contributions.}
\textbf{(A)} OSCBO updates GP hyperparameters online through a sharpness--calibration primal--dual loop coupled to GP-UCB.
\textbf{(B)} Posterior uncertainty and UCB query selection under different baselines, characterized by specific hyperparameter-adaptation rules or UCB bands.
\textbf{(C)} Different methods explore different regions of the sharpness--calibration tradeoff, measured by average confidence-interval width and empirical coverage gap, and induce distinct lengthscale evolution.
\textbf{In this example, OSCBO yields wider but better covered confidence bands than GP-UCB-MLL, is less post-hoc quantile-driven than OCBO, and avoids the highly schedule-driven lengthscale collapse of A-GP-UCB.}
}
\label{fig:gist}
\end{figure*}

\section{Related Work}

\textbf{BO with unknown hyperparameter.}
A growing line of work studies BO when GP hyperparameter are unknown and must be learned online, relaxing the fixed-kernel/RKHS assumption behind classical regret bounds.
A canonical example is the kernel \emph{lengthscale}, which governs correlation range and thus effective smoothness.
\cite{berkenkamp2019no} (A-GP-UCB) and \cite{ziomek2024bayesian} (LB-GP-UCB) propose principled lengthscale adaptation schemes with regret guarantees, designed to ensure that the induced RKHS is eventually rich enough to contain the objective.
Related results cover Matérn kernels with fixed smoothness \citep{liu2023adaptation}, and complementary approaches use meta-learning to initialize hyperparameter from prior tasks \citep{wang2018regret, feurer2015initializing}.
Our work also targets regret under iterative refitting, but via a different principle: selecting hyperparameter to manage the calibration--sharpness trade-off along the BO trajectory.

\textbf{Calibration and recalibration under adaptive BO.}
Recent work studies calibrated uncertainty for BO and, more broadly, for sequential prediction under shift and adversarially chosen data.
In BO, \cite{stanton2023bayesian} combine GP posteriors with conformal prediction to produce calibrated prediction sets under covariate shift, and conformal methods more broadly extend coverage beyond the i.i.d.\ setting to distribution shifts and adaptive data collection \citep{tibshirani2019conformal, gibbs2021adaptive, barber2023conformal, pmlr-v235-angelopoulos24a}.
Complementarily, \cite{deshpande2024online} propose an online \emph{post-hoc} recalibration rule based on a modified pinball loss, and \cite{deshpande2024calibratedregressionadversaryregret} (OCBO) provide calibration guarantees under arbitrary adversaries, including sequential data.
Sharp-calibrated Gaussian processes (SCGP) \citep{capone2023sharp} also learn GP uncertainty parameters with a sharpness--calibration objective, but are designed for calibrated regression rather than adaptive BO: they rely on a train/calibration split and target predictive-set calibration at a fixed confidence level, rather than optimizing a BO acquisition along a trajectory.
These approaches primarily target calibrated predictive uncertainty or quantile/CDF recalibration, and do not directly address the online calibration--sharpness trade-off for GP hyperparameter selection inside arbitrary BO acquisitions.
We instead adapt GP hyperparameters online to optimize sharpness while keeping long-run miscoverage under control along the BO trajectory.

\textbf{Online learning with constraints and dynamic environments.}
Hyperparameter adaptation in BO can be viewed as online learning with long-term constraints, where the learner optimizes an objective while controlling cumulative calibration violation.
The relevant guarantees are therefore optimality gaps and constraint-violation bounds, distinct from BO regret.
\cite{castiglioni2022unifying} provide best-of-both-worlds guarantees for stochastic and adversarial rewards/constraints with nonconvex black-box subroutines and a play--recovery mechanism.
Relatedly, \cite{bernasconi2024} study bandits with long-term constraints using optimistic feasibility estimates and time-varying feasible sets.

\vspace{-.2cm}
\section{Preliminaries}
\vspace{-.2cm}

Let $\calX \subset \real^d$ be a compact input space and $\calY \subset \real$ the output space. Let $f:\calX \to \calY$ be an expensive black-box function, with maximizer(s) $\x^\ast \in  \argmax_{\x \in \calX} f(\x)$.

\vspace{-.15cm}
\subsection{Bayesian Optimization with GP Surrogates}\label{sub:BO}
\vspace{-.15cm}

BO is a sequential strategy for black-box optimization that uses a probabilistic surrogate to select new evaluations \citep{garnett2023bayesian}.
At round $t \in \mathbb{N}$, we have observations
$\mathcal{D}_{t-1}=\{(\mathbf{x}_i,y_i)\}_{i=1}^{t-1}$ with
$y_i=f(\mathbf{x}_i)+\epsilon_i$. We assume that the noise process is
conditionally $R$-sub-Gaussian with respect to the history filtration
$(\mathcal{F}_t)_{t\ge0}$: $\mathbb{E}\!\left[\exp(\eta\epsilon_t)\mid\mathcal{F}_{t-1}\right]
\le \exp(\eta^2R^2/2), \forall t\ge 1,\ \forall \eta\in\mathbb{R}$.

% BO is a sequential strategy for black-box optimization that uses a probabilistic surrogate to select new evaluations \citep{garnett2023bayesian}.
%  At round $t \in \mathbb{N}$, we have observations $\dataset_{t - 1} = \{(\x_i, y_i)\}_{i=1}^{t - 1}$ with noisy evaluations $y_i = f(\x_i) + \epsilon_i$. We assume that the sequence $\{ \epsilon_i \}_{i = 1}^\infty$ is conditionally $R$-sub-Gaussian, i.e., $\expect[\exp(\eta \, \epsilon_i) \vert \calF_{t - 1}] \leq \exp(\eta^2 R^2 / 2), \, \forall i \geq 1, \forall \eta \in \real$, where $\calF_{t - 1}$ denotes a filtration generated by the random sequence $\{\x_i, \epsilon_i \}_{i = 1}^{t - 1}$.

Given $\dataset_{t - 1}$, a surrogate induces an acquisition function $\alpha:\calX \to \real$, and the next query is chosen as $\x_{t} = \argmax_{\x \in \calX}~ \alpha(\x; f_{t - 1})$. After observing $y_{t}$, we update $\dataset_{t}=\dataset_{t - 1}\cup\{(\x_{t},y_{t})\}$ and refit the surrogate.
The performance is frequently measured by the instantaneous regret $r_t = f(\x^\ast)  - f(\x_t)$ and the cumulative regret $R_T = \sum_{t=1}^T \big(f(\x^\ast)-f(\x_{t})\big)$, where $T$ denotes the total rounds.

The surrogate is chosen as a GP, i.e., a collection of random variables whose finite subsets are jointly Gaussian \citep{Rasmussen2006}. We place a zero-mean GP prior $f(\x) \sim \normal\!\big(0,\, k_{\hat{\btheta}}(\x,\x')\big)$, where the kernel $k_{\hat{\btheta}}\hspace{-.1cm}: \hspace{-.1cm}\calX \hspace{-.1cm} \times \hspace{-.1cm} \calX \hspace{-.1cm}\to \hspace{-.1cm}\real$ has hyperparameters $\hat{\btheta}\in\bTheta$. Here, $\btheta$ refers to kernel lengthscales, covering both isotropic (scalar) and Automatic Relevance Determination (ARD) cases. We assume $k_{\hat{\btheta}}(\x,\x') \le 1$ for each $\x,\x'\in\calX$ and $\hat{\btheta} \in \bTheta$. Given $\dataset_{t - 1}$ and a test point $\x$, GP predictive mean and variance are
\begin{align}\label{eq:predictive-dist}
&\mu_{t - 1}(\x;\hat{\btheta}) = \kk_{\hat{\btheta}}(\x)^\top (\K_{\hat{\btheta}}^{t - 1} + \sigma^2 \I)^{-1} \y_{t - 1}, \\
&\sigma_{t -1}^2(\x;\hat{\btheta}) = k_{\hat{\btheta}}(\x,\x) - \kk_{\hat{\btheta}}(\x)^\top (\K_{\hat{\btheta}}^{t - 1} + \sigma^2 \I)^{-1} \kk_{\hat{\btheta}}(\x),
\end{align}
where $\y_{t - 1}= [y_1, \dots y_{t - 1}]$, $\K^{t - 1}_{\hat{\btheta}}=[k_{\hat{\btheta}}(\x_i,\x_j)]_{i,j=1}^{t - 1}\in\real^{(t - 1) \times (t  - 1)}$, and $\kk_{\hat{\btheta}}(\x)=[k_{\hat{\btheta}}(\x_i,\x)]_{i=1}^{t - 1}\in\real^{t - 1}$, Here, the parameter $\sigma^2$ refers to the noise model.
In practice, $\hat{\btheta}$ is typically re-estimated online (e.g., by marginal likelihood maximization), so the posterior quantities $\mu_{t-1}(\cdot;\hat{\btheta})$ and $\sigma^2_{t-1}(\cdot;\hat{\btheta})$ depend both on the accumulated data and on the current hyperparameter estimate.
The acquisition function relies on these to query the next data point $\x_t$. A common acquisition strategy is the upper confidence bound (UCB), which selects $\x_{t}=\argmax_{\x \in \calX}\,\mu_{t - 1}(\x;\hat{\btheta})+\sqrt{\beta_t(\delta)}\,\sigma_{t - 1}(\x; \hat{\btheta})$, where $\beta_t(\delta)>0$ is an exploration parameter.

To state the later confidence and regret guarantees under online hyperparameter adaptation, we require a mild structural assumption on the admissible kernel family.
%In particular, we assume that the search set $\bTheta$ is compact and that the objective admits a uniform RKHS norm bound over this family.
%
\begin{assumption}\label{assumption:rkhs}
The hyperparameter space $\bTheta$ is compact, strictly bounded away from the origin, and has bounded diameter $D=\sup_{\btheta,\hat{\btheta}\in\bTheta}\|\btheta-\hat{\btheta}\|_\infty$.
Moreover, there exists $\btheta\in\bTheta$ such that the objective function $f$ lies in the RKHS $\hilbert_{k_{\btheta}}$, and the admissible family is chosen so that
$\sup_{\hat{\btheta}\in\bTheta}\|f\|_{k_{\hat{\btheta}}}\le B$ for some constant $B > 0$.
\end{assumption}

In order to guarantee the validity of $\Vert f \Vert_{k_{\hat{\btheta}}}$, for any $\hat{\btheta} \in \bTheta$, we consider the $\nu-$Matérn kernel family. It is known the RKHS under this family kernel is norm equivalent, i.e., given $\bTheta \subset [\btheta_{\min}, \btheta_{\max}]^d$, there exists a constant $C_\nu > 0$ depending on $\nu, \btheta_{\min}$, and $\btheta_{\max}$ such that for any $\hat{\btheta} \in \bTheta$, $\Vert f \Vert_{k_{\hat{\btheta}}} \leq C_\nu \Vert f \Vert_{k_{\btheta_{\min}}}$. \citep{kanagawa2018gaussian}. In addition, the assumption on $\bTheta$ is required to ensure the optimism principle of UCB and the regret bound of the underlying online learning scheme, both of which are essential for the regret analysis of OSCBO.

\subsection{Calibration and Sharpness under Adaptive Setting}\label{sub:prelimcalib}
\emph{Calibration} is measured through long-run coverage via conformal prediction sets. Since the inputs $\x_t$ may be selected adaptively from past observations, as in BO, the data are generally non-i.i.d. At each round $t$, the surrogate outputs a predictive distribution and constructs a prediction set using a score function $s_t:\calX\times\calY\to[0,L], C_t(\x_t)=\{y\in\calY:\; s_t(\x_t,y)\le \tau_t\}$
%where $\tau_t$ is a possibly time-varying threshold.
Define the long-run coverage
\vspace{-.1cm}
\begin{equation}
    \hat p_T := \frac{1}{T}\sum_{t=1}^T \mathbbm{1}\!\left[y_t\in C_t(\x_t)\right].
\end{equation}
The surrogate is said to be \emph{marginally calibrated} at miscoverage level $\delta\in[0,1]$ if $\hat p_T \approx 1-\delta$.

This metric alone does not guarantee informative uncertainty sets \citep{gneiting2007probabilistic}. \emph{Sharpness} is also considered, quantified by prediction-set volume. For GPs, this is determined by the posterior predictive variance, which we summarize by the cumulative log scaled predictive variance along the query sequence:
\vspace{-.1cm}
\begin{equation}
\sum_{t=1}^T \log (1 + \sigma^{-2} \sigma^2_{t-1}(\x_t; \hat{\btheta}_{t - 1})).
\end{equation}
\vspace{-.5cm}
\subsection{Online Optimization under Adaptive Setting}\label{subsec:constrained_online_learning}
Consider an online optimization problem with long-term constraints in an adaptive environment~\citep{castiglioni2022unifying}. At each round $t\in[T]$, a learner selects an action $\btheta_t\in\bTheta$, after which the environment reveals a loss $L_t^s:\bTheta\to\mathbb{R}$ and constraint violation $L_t^c:\bTheta\to\mathbb{R}$. 
Both $(L_t^s, L_t^c)$ depend on the interaction history, reflecting non-i.i.d. data collection. 

This problem assumes the existence of \emph{strategy mixtures} $\Xi$, a set of probability measures over $\bTheta$. Define the time-averaged loss $\bar L^s(\hat\btheta)=\frac{1}{T}\sum_{t=1}^T L_t^s(\hat\btheta)$, and constraints $\bar L^c(\hat\btheta)=\frac{1}{T}\sum_{t=1}^T L_t^c(\hat\btheta)$, we benchmark the performance against the constrained optimum
\begin{equation}\label{eq:opt}
\small
\mathrm{OPT}_{\bar L^s,\bar L^c}
=\inf_{\xi\in\Xi}\ \mathbb{E}_{\hat\btheta\sim\xi}\!\left[\bar L^s(\hat\btheta)\right]
\text{ s.t. }
\mathbb{E}_{\hat\btheta\sim\xi}\!\left[\bar L^c(\hat\btheta)\right]\le 0.
\end{equation}
A strict feasibility condition is assumed: there exists $\xi\in\Xi$ such that $\mathbb{E}_{\hat\btheta\sim\xi}[\bar L^c(\hat\btheta)]<0$, which ensures the constraint is satisfiable and the dual domain can be bounded~\citep{castiglioni2022unifying}.
Then, the associated Lagrangian relaxation for Equation~\ref{eq:opt} is $\mathcal{L}(\xi,\lambda)
= \mathbb{E}_{\hat\btheta\sim\xi}\![\bar L^s(\hat\btheta)] + \lambda \mathbb{E}_{\hat\btheta\sim\xi}\![\bar L^c(\hat\btheta)]$, where $\lambda\in\mathbb{R}_{+}$ is a multiplier.
This induces primal updates over $\Xi$ (or directly over $\bTheta$) and dual updates over $\lambda$.
In practice, our work is restricted to Dirac measures on $\bTheta$ and uses black-box primal and dual \emph{regret minimizers} (RMs), which guarantee a sublinear best-in-hindsight optimality gap~\citep{castiglioni2022unifying}.

% In our setting, we restrict to point-mass strategies, i.e., Dirac measures on $\bTheta$, and use black-box primal and dual \emph{regret minimizers} (RMs) to solve the resulting online problem. Their guarantees imply a sublinear best-in-hindsight optimality gap~\citep{castiglioni2022unifying}.

\section{A Sharpness--Calibration Lens on GP-UCB}
\vspace{-.2cm}
To motivate our approach, first consider the classical GP-UCB setting in which a single kernel $k_{\btheta}$ is fixed throughout the BO trajectory. In this setting, the GP-UCB confidence interval scales with $\sqrt{\beta_t(\delta)}\,\sigma_{t-1}(\x;\btheta)$. This quantity helps to bound the cumulative regret in terms of GPs' sharpness with a high probability. It also yields a long-run coverage along the queried trajectory with a high probability, which can be viewed as a probabilistic analogue of the calibration notion from \cref{sub:prelimcalib}.

\textbf{GP sharpness bounds the regret.}
A standard consequence of the GP-UCB analysis is:
\begin{restatable}{corollary}{corollarysharpness}\label{cor:sharpness}
Consider GP-UCB with fixed kernel $k_{\btheta}$. Fix $\delta\in(0,1)$ and let $\beta_t:(0,1)\to\real$ be non-decreasing. Then, with probability at least~$1-\delta$,
$R_T \le \sqrt{T\!\sum_{t=1}^T 4 \beta_t(\delta) \, \sigma^2 \, C \, \log(1 + \sigma^{-2} \, \sigma_{t-1}^2(\x_t;\btheta))}$, where $C = \sigma^{-2} / \log(1 + \sigma^{-2})$.
\end{restatable}
The proof follows from the confidence bound in \cite[Lemma~5.2 \& Theorem~6]{srinivas2009gaussian}, together with Cauchy--Schwarz (\cref{proof:sharpness}). The corollary tells that under a fixed kernel, the cumulative regret $R_T$ is controlled by the same posterior uncertainty that determines the confidence interval.

\textbf{Confidence interval bounds as a probabilistic long-run coverage.}
To connect with \cref{sub:prelimcalib}, define the score and prediction set at the function level, with the associated long-run coverage:
\vspace{-.1cm}
\[
s_t^f(\x,z)\!=\!\frac{|z-\mu_{t-1}(\x;\btheta)|}{\sqrt{\beta_t(\delta)}\,\sigma_{t-1}(\x;\btheta)},
\;
C_t^f(\x)\!=\!\{z\in\real:\ s_t^f(\x,z)\le 1\},
\;
\hat p_T^f\!=\!\frac{1}{T}\sum_{t=1}^T \indicator\!\left[f(\x_t)\in C_t^f(\x_t)\right].
\]
\vspace{-.05cm}
This is the function-level analogue ($f(\x_t)$) of the observation-level $(y_t)$ notion in \cref{sub:prelimcalib}.% it concerns the function value $f(\x_t)$ rather than the noisy observations $y_t$.
%Unlike \cref{sub:prelimcalib}, this is not observation-level calibration: it concerns the latent value $f(\x_t)$ rather than the noisy observation $y_t$.
%
\begin{restatable}{corollary}{corollarycalibration}\label{corr:calibration}
For any $\delta\in(0,1)$ and $\beta_t(\delta)$ as in \cref{cor:sharpness}, running GP-UCB with fixed kernel for $T$ rounds satisfies $\probability\!\left(\hat p_T^f \ge 1-\delta\right)\ge 1-\delta.$
\end{restatable}
%
% This follows directly from the uniform confidence event in GP-UCB~\cite[Lemma~5.2]{srinivas2009gaussian}. Hence, in the fixed-kernel setting, the same GP intervals both encode sharpness through their width and provide a conservative latent-function coverage guarantee along the queried trajectory.
This follows directly from the uniform confidence interval event in GP-UCB~\cite[Lemma~5.2]{srinivas2009gaussian}, which implies $f(\x_t)\in C_t^f(\x_t)$ for all queried points simultaneously with probability at least $1-\delta$. Of note, the same event also holds at $\x^\ast$, yielding the optimism step used in the GP-UCB regret analysis.

Taken together, \cref{cor:sharpness,corr:calibration} conveys that, under a fixed kernel, the GP-UCB confidence interval links probabilistic sharpness, calibration, and cumulative regret bound, motivating OSCBO. 
When hyperparameters are selected online, however, the BO trajectory is no longer governed by a single fixed-kernel confidence event, and this link is no longer automatic. The next section introduces explicit sharpness--calibration control during hyperparameter selection.

\vspace{-.2cm}
\section{Online Sharp-Calibrated BO}\label{sec:oscbo}
\vspace{-.2cm}
% We now present \emph{Online Sharp-Calibrated Bayesian Optimization} (OSCBO), which treats the kernel hyperparameter selection as a trade-off between the \emph{sharpness} and long-term \emph{calibration} violations. For each BO round, two pieces of information are extracted from the newly queried point: a prediction set that embodies sharpness and a miscoverage gap that contributes to the long-term calibration violation. \Cref{subsec:oscbo_losses} frames sharpness as a reward and miscoverage gap as a constraint on the GP kernel hyperparameter. \Cref{subsec:oscbo_algo} describes an online-learning framework used to update the hyperparameter sequentially, minimizing sharpness while maintaining the long-term calibration violation not exceeding the threshold. Lastly, \cref{subsec:oscbo_bo} integrates the learned hyperparameter into a BO loop, yielding OSCBO.

We present \emph{Online Sharp-Calibrated Bayesian Optimization} (OSCBO), which treats GP hyperparameter selection as a constrained online optimization problem, balancing sharpness against long-run calibration violations. \Cref{subsec:oscbo_losses} defines the per-round sharpness loss and calibration constraint extracted from each BO query. \Cref{subsec:oscbo_algo} describes the resulting primal--dual online update and the play--recovery mechanism. \Cref{subsec:oscbo_bo} couples this update with the BO loop, yielding OSCBO.

\vspace{-.2cm}
\subsection{Sharpness Loss and Calibration Constraint}\label{subsec:oscbo_losses}
\vspace{-.2cm}

At round $t-1$, OSCBO chooses $\hat\btheta_{t-1}\in\bTheta$. After selecting $\x_t$ and observing $y_t$, it extracts a sharpness loss $L_t^s(\hat\btheta)$ and a calibration constraint $L_t^c(\hat\btheta)$:
% At round $t - 1$, OSCBO selects GP hyperparameter $\hat\btheta_{t - 1} \in\bTheta$ before observing the next outcome. After choosing the next evaluation point $\x_{t}$, the observation $y_{t}$ becomes available. This query pair provides two pieces of information: a \emph{sharpness reward} $L_t^s(\hat\btheta)\in[0,1]$, and a \emph{calibration constraint} $L_t^c(\hat\btheta)\in[-1,1]$,
% defined as
%
\begin{align}
L_t^{s}(\hat\btheta_{t - 1}) 
&= 1 /  \log(1 + \sigma^{-2})   \log(1 + \sigma^{-2} \sigma^2_{t - 1}(\x_{t}; \hat{\btheta}_{t - 1})), \label{eq:sharpness-reward}
\\
L_t^{c}(\hat\btheta_{t - 1}) 
&= |y_{t}-\mu_{t - 1}\!(\x_{t};\hat\btheta_{t-1})|^p
     \, / \left( \beta^{1/2}_t(\delta)\, \sqrt{\sigma^2_{t - 1}\!(\x_{t};\hat\btheta_{t - 1}) + \sigma^{2} } \right)^p - 1\label{eq:calibration-constraint},
\end{align}
where $p \in \{1, 2\}$.
%denotes the polynomial degree.
We quantify the sharpness using the logarithm of scaled predictive variance to facilitate the regret analysis later. The constraint $L_t^{c}(\hat\btheta_{t - 1})\le 0$ is equivalent to the event that $y_{t}$ falls inside the GP confidence interval at $\x_{t}$ under $\hat\btheta_{t - 1}$, implying a calibrated GP model (\cref{sub:prelimcalib}).

\subsection{Online Learning Framework for Hyperparameter Selection}\label{subsec:oscbo_algo}

\textbf{Objective and performance criteria.}
Given the per-round loss and constraint feedback in \cref{subsec:oscbo_losses}, the goal is to select hyperparameters $\hat\btheta_1,\dots,\hat\btheta_T$ that achieve a minimum cumulative loss and simultaneously keeping the long-term constraint violation small. While \cite{castiglioni2022unifying} measures the performance against $\mathrm{OPT}_{\bar L^{s},\bar L^{c}}$ defined in \cref{eq:opt}, our interests lies in the optimality gap and long-term constraint violations w.r.t. the true hyperparameter $\btheta$:
\vspace{-.15cm}
\begin{equation}
\hat G_T
=
% \sum_{t=1}^T L_t^s(\hat\btheta_{t-1})
% -
% \sum_{t=1}^T L_t^s(\btheta),
\sum_{t=1}^T \left(L_t^s(\hat\btheta_{t-1})
- L_t^s(\btheta)\right),
\qquad
\hat V_T
=
\sum_{t=1}^T L_t^c(\hat\btheta_{t-1}).
\end{equation}
\vspace{-.05cm}
% By definition of $\hat G_T$, the cumulative reward achieved by OSCBO is directly related to that of the true kernel. This relation is the starting point of the regret analysis in Section~\ref{sec:theory}.
% $\hat{G}_T$ controls how much larger OSCBO's cumulative sharpness score can be than that of the true kernel, 
$\hat{G}_T$ bounds OSCBO's excess cumulative sharpness relative to the true kernel.
This relation is used in Section~\ref{subsec:theory_bo} to reduce the BO regret analysis to the maximum information gain of the true kernel.
% Assuming $G_T$ is bounded, we can express the cumulative reward $\sum_{t = 1}^T L_t^s(\hat{\btheta}_{t - 1})$ achieved by OSCBO in terms of the cumulative reward $\sum_{t = 1}^T L_t^s(\btheta)$ yielded by the true kernel. Furthermore, by the definition of $L_t^s$, we later show that $\sum_{t = 1}^T L_t^s(\btheta)$ can be bounded by the maximum information gain, providing a clear path to a sublinear regret bound for OSCBO.
% % Since $L_t^s$ is defined from the predictive variance, the cumulative reward under $\btheta$ can in turn be bounded by the maximum information gain, leading to the regret bound for OSCBO.

\textbf{Lagrangian game and play--recovery.}
Following the best-of-both-worlds constrained online-learning framework of \cite{castiglioni2022unifying}, OSCBO casts hyperparameter selection as an online game with per-round loss $L_t^s$ and constraint $L_t^c$. At each round, a \emph{primal} RM selects a hyperparameter, while a \emph{dual} RM selects a multiplier $\lambda_{t-1}\in\mathbb{R}_+$ that penalizes calibration violations through Lagrangian relaxation (\cref{subsec:constrained_online_learning}). To control long-run calibration violations, OSCBO alternates between play and recovery phases. In the \emph{play phase}, the primal RM optimizes sharpness while accounting for violations through $\lambda_{t-1}$. Once the cumulative calibration violation exceeds a threshold, OSCBO switches to the \emph{recovery phase}, prioritizing the miscoverage term $L_t^c$ to offset past violations.

\textbf{Feasibility margin and exploration scale.}
The play--recovery mechanism relies on a feasibility margin quantifying how strongly the calibration constraint can be satisfied on average:
\vspace{-.1cm}
\begin{equation}
\rho := \sup_{\hat{\btheta}\in\bTheta}\ \frac{1}{T}\sum_{t=1}^T -L_t^c(\hat\btheta).
\label{eq:margin}
\end{equation}
\vspace{-.1cm}
A positive $\rho$ is the strict-feasibility condition that permits a bounded dual domain in constrained online learning (Section~\ref{subsec:constrained_online_learning}). Since $\rho$ depends on the realized trajectory and is generally unknown, OSCBO is run with a supplied lower bound $\hat\rho\in(0,\rho]$. This lower bound determines the truncated dual domain and the play--recovery threshold in Algorithm~\ref{alg:oscbo} through $\tilde\rho=\max\{\hat\rho/2,\,T^{-1/4}\}$. The exploration scale $\beta_t$ appears both in the UCB acquisition and in the calibration constraint. Valid choices of $\hat\rho$ and $\beta_t$ are deferred to the theoretical analysis (Section~\ref{sec:theory}).
\begin{wrapfigure}[18]{r}{0.52\textwidth}
\vspace{-1.1\baselineskip}
\footnotesize
\begin{minipage}{0.52\textwidth}
\hrule
\vspace{2pt}
\refstepcounter{algorithm}
\noindent\textbf{Algorithm~\thealgorithm} \texttt{OSCBO}\label{alg:oscbo}
\vspace{2pt}
\hrule
\vspace{4pt}

\begin{algorithmic}[1]
\STATE {\bfseries Input:} $\dataset_0$, $T$, $\delta$, $\hat\rho$
\STATE $\tilde\rho \leftarrow \max\{\hat\rho/2,\, T^{-1/4}\}$,\quad $\eta \leftarrow \delta/3$,\quad $t \leftarrow 1$
\STATE $V_t \leftarrow 0$
\algcomment{--- Play phase initialization ---}
\STATE $\mathcal{R}^P \leftarrow \texttt{INIT}^P(\bTheta,\ [-1/\tilde\rho,\ 1+1/\tilde\rho],\ \eta)$
\STATE $\mathcal{R}^D \leftarrow \texttt{INIT}^D(\calS_{\tilde\rho},\ [-1/\tilde\rho,\ 1/\tilde\rho],\ 0)$
\STATE $v \leftarrow 1$
\WHILE{$t \le T$}
    \IF{$V_t > (T-t)\tilde\rho + M_{\tilde\rho}-1$}
        \algcomment{--- Switch to recovery phase ---}
        \STATE $\mathcal{R}^P \leftarrow \texttt{INIT}^P(\bTheta,\ [-1,1],\ \eta)$
        \STATE $\mathcal{R}^D \leftarrow \texttt{INIT}^D(\Delta,\ [-1,1],\ 0)$
        %\STATE $v \leftarrow 0$
    \ENDIF
    \STATE $\hat\btheta_{t-1} \leftarrow \mathcal{R}^P.\texttt{NEXTELEMENT}()$
    \STATE $\blambda_{t-1} \leftarrow \mathcal{R}^D.\texttt{NEXTELEMENT}()$
    \STATE $\x_t \leftarrow \argmax_{\x\in\calX}\alpha(\x;\hat\btheta_{t-1})$
    \STATE Query $y_t$ and set $\dataset_t \leftarrow \dataset_{t-1}\cup\{(\x_t,y_t)\}$
    \STATE $\texttt{RECALIBRATE}(\mathcal{R}^P,\mathcal{R}^D)$ %,v)$
    \STATE $V_{t+1} \leftarrow V_t + L_t^c(\hat\btheta_{t-1})$
    \STATE $t \leftarrow t + 1$
\ENDWHILE
\end{algorithmic}

\vspace{2pt}
\hrule
\end{minipage}
\vspace{-0.8\baselineskip}
\end{wrapfigure}

\textbf{Regret minimizers.}
OSCBO treats primal and dual learners as black-box RMs (\cref{subsec:constrained_online_learning}, \cref{app:rm_interface}), instantiated via \texttt{INIT} and alternating \texttt{NEXTELEMENT}/\texttt{OBSERVEUTILITY} updates.
Our implementation uses online mirror descent with a negative-entropy regularizer for dual RM; primal RM uses FTPL to handle non-convex $\bTheta$ (\cref{app:ftpl,app:omd}).

\subsection{BO Coupling and Feedback Loop}\label{subsec:oscbo_bo}
\cref{alg:oscbo} gives the full OSCBO procedure and \cref{alg:recalibrate} the corresponding Lagrangian-game update. At round $t-1$, the primal RM outputs $\hat\btheta_{t-1}$, which together with $\dataset_{t-1}$ defines the GP posterior $(\mu_{t-1}(\cdot;\hat\btheta_{t-1}),\sigma_{t-1}(\cdot;\hat\btheta_{t-1}))$ used by the acquisition to select $\x_t$. After observing $y_t$ and updating $\dataset_t$, OSCBO evaluates the realized loss and constraint feedback $L_t^s(\hat\btheta_{t-1})$ and $L_t^c(\hat\btheta_{t-1})$ and passes them to the primal and dual RMs via \texttt{RECALIBRATE}. 
The play–recovery threshold depends on the play-phase violation budget $M_{\tilde\rho}$, its precise definition is deferred to Appendix~\ref{app:rm_interface}.
Our theory specializes $\alpha$ to UCB, but OSCBO is compatible with any GP-based acquisition.

\vspace{-.2cm}
\section{Theoretical Analysis}\label{sec:theory}
\vspace{-.1cm}

This section establishes the guarantees of OSCBO in three steps. First, it justifies the two analytical inputs used by the algorithm: the exploration scale $\beta_t$ and the lower bound $\hat\rho$ on the feasibility margin $\rho$. Second, it bounds the online-learning optimality gap and cumulative calibration violation. Third, it combines these ingredients with the GP-UCB proof template to obtain a BO regret bound~\cite{srinivas2009gaussian}.

\subsection{Uniform Confidence and Feasibility Margin}\label{subsec:theory_confidence}

Let $d(\btheta,\hat{\btheta}) \hspace{-.05cm}:=\hspace{-.05cm}
\sup_{\x,\x'}\bigl|k_{\btheta}(\x,\x')-k_{\hat{\btheta}}(\x,\x')\bigr|$, and let $\mathcal N_\varepsilon(\bTheta,d)$ denote the $\varepsilon$-covering number of $\bTheta$ under $d$. For a fixed $\hat\btheta\in\bTheta$, let $\gamma_t(\hat\btheta)$ be the maximum information gain after $t$ evaluations under $k_{\hat\btheta}$:
\vspace{-.2cm}
\begin{equation}
\gamma_t(\hat\btheta)
:=
\max_{A\subset\calX:\, |A|=t}
\frac{1}{2}\log\det\!\left(\I+\sigma^{-2}\K_A^{\hat\btheta}\right),
\label{eq:gamma}
\end{equation}
and write $\Gamma_{t-1}\hspace{-.05cm}:=\hspace{-.05cm}\max_{\hat{\btheta}\in\bTheta}\gamma_{t-1}(\hat{\btheta})$.
The next theorem specifies $\beta_t$ so that the GP confidence bound holds uniformly over $\bTheta$ and justifies a valid lower bound $\hat\rho$ on the feasibility margin (Equation~\ref{eq:margin}).

\begin{restatable}{theorem}{lemmaslater}\label{lemma:slater-condition}
Assume \cref{assumption:rkhs} and let $\rho$ be the feasibility margin  (\cref{subsec:oscbo_algo}). Let the continuous domain be discretized by a grid of size $\varepsilon_T = T^{-4}$ and define
$\beta_t^{1/2}(\delta)
=
B
+
R\sqrt{2\!\left(\Gamma_{t-1}+1+\log\! \left(\frac{6\mathcal N_{\varepsilon_T}(\bTheta,d)}{\delta} \right) \right)}$.
Then, with a probability of at least $1 - \delta/2$, the following safe environment conditions hold simultaneously:
\vspace{-.1cm}
\begin{itemize}
\item The feasibility margin satisfies  $\rho \geq \hat{\rho} > 0$.
\item For $\forall t\in[T],\ \forall \x\in\calX,\ \forall \hat{\btheta}\in\bTheta,$ it holds that \[
|f(\x)-\mu_{t-1}(\x;\hat{\btheta})|
\le
\beta_t^{1/2}(\delta)\,\sigma_{t-1}(\x;\hat{\btheta})
+
\sqrt{L_{\sigma^2}\varepsilon_T} + L_{\mu_{t-1}}\varepsilon_T,
\]
\end{itemize}
\end{restatable}
\vspace{-.2cm}
\textit{Sketch of proof.}
Cover $\bTheta$ by an $\varepsilon$-net, apply the kernelized bandit confidence bound of \cite[Theorem~2]{chowdhury2017kernelized} on the net points, and extend it to all $\hat{\btheta}\in\bTheta$ using Lipschitz continuity of the posterior mean and variance under $d$. The lower bound $\hat\rho\le \rho$ then follows by applying this confidence event to the constraint process and controlling the resulting martingale fluctuations.  Details in \cref{app:slater-condition}.

\vspace{-.15cm}
\subsection{Online-Learning Guarantees for Sharpness and Calibration}\label{subsec:theory_online}
\vspace{-.1cm}

Assuming the high-probability event of \cref{lemma:slater-condition} holds, the constrained online-learning analysis yields sublinear bounds on the sharpness optimality gap and cumulative calibration violation.
\begin{restatable}{lemma}{lemmaconvergence}\label{lemma:convergence}
Let $\btheta$ and $\hat{V}_T$ be the hyperparameter of the true kernel $k_\btheta$ and cumulative calibration-constraint violation induced by the per-round functions $L_t^s$ and $L_t^c$ (\cref{subsec:oscbo_losses}). Let the primal and dual RMs be realized by the FTPL and OMD algorithm, respectively.
On the safe environment events from \cref{lemma:slater-condition}, \cref{alg:oscbo} guarantees that with probability at least $1-\tfrac{1}{2}\delta$,
\begin{align}
\small
G_T \le \tilde{\bigo}(\sqrt{T}) + R_T^P + R_T^D, & & \hat{V}_T \leq \rho\left(  \sum_{t = 1}^T L_t^s(\btheta)  + \tilde{\bigo}(\sqrt{T})  + R_T^P + R_T^D \right),
\end{align}
\vspace{-.1cm}
where $R_T^P = \tilde{\bigo}{\sqrt{T}}$ and $R_T^D = \tilde{\bigo}{\sqrt{T}}$ denote the regret bound of FTPL and OMD algorithm.
\end{restatable}
\textit{Proof sketch.} Our approach leverages the Azuma–Hoeffding concentration inequality to convert the expected regret bounds of the FTPL and OMD algorithms into high-probability guarantees.
% Next, we incorporate the resulting regret bounds into the Lagrangian formulation of our problem to derive bounds on both the cumulative reward and the long-term constraint violations.
% Finally, we apply the union bound to combine these results into a single high-probability event.
 Substituting these into the Lagrangian formulation of our problem yields bounds on the cumulative loss and long-term constraint violations, and a union bound combines the events.

Thus, the online-learning layer incurs only sublinear degradation relative to the true kernel. It remains to relate $\sum_{t=1}^T L_t^s(\btheta)$ to the maximum information gain, yielding a GP-UCB-style regret bound.

% Lemma~\ref{lemma:convergence} establishes bounds on the sharpness–reward optimality gap $\hat G_T$ and the cumulative calibration violation $\hat V_T$ in terms of the primal and dual regret guarantees. In contrast to \cite{castiglioni2022unifying}, our analysis does not require partitioning each phase into separate play and recovery stages, since our Bayesian optimization (BO) regret analysis relies solely on bounds for the optimality gap $G_T$. Nevertheless, we note that the OSCBO algorithm exhibits similar behavior to the meta-algorithm proposed in \cite{castiglioni2022unifying}. In the next subsection, we demonstrate that the cumulative reward $\sum_{t = 1}^T L_t^s(\btheta)$ can be bounded using the maximum information gain, which in turn enables us to derive a regret bound in the classical GP-UCB form.
\vspace{-.15cm}
\subsection{BO Regret Bounds for OSCBO with GP-UCB}\label{subsec:theory_bo}
\vspace{-.1cm}

We now translate the online-learning guarantees into a BO regret bound for OSCBO with UCB acquisition. The choice of $\beta_t$ in \cref{lemma:slater-condition} serves two purposes: it yields the uniform confidence event over $\bTheta$ used above, and it is also large enough to support the UCB optimism step in the regret proof below. Combined with \cref{lemma:convergence}, this reduces the BO analysis to controlling the cumulative sharpness loss under the reference kernel, which is done through the maximum information gain.

\begin{restatable}{theorem}{theoremregretbound}\label{theorem:regret-bound}
Let $\delta\in(0,1)$, and set \(\varepsilon_T:=T^{-4}\). Assume $f\in\mathcal{H}_{k_\btheta}$ with $\|f\|_{k_\btheta}\le B$, and that the noise is conditionally $R$-sub-Gaussian.
Run \texttt{OSCBO} for $T$ rounds with UCB acquisition and exploration scale $\beta_t$ defined as in \cref{lemma:slater-condition}.
Then, with probability at least $1-\delta$,
\begin{equation}\label{eq:regret_opt_version}
R_T \leq \tilde{\bigo} \left(  \left(B + \sqrt{\Gamma_{T - 1} +  \log(6 \normal_{\varepsilon_T}(\bTheta, d) / \delta)} \right) \left( T^{3/4} + \sqrt{T \gamma_{T}(\btheta)} \right) \right)
\end{equation}
where $\gamma_{T}(\boldsymbol{\theta})$ denotes the maximum information gain associated with the true kernel $k_{\boldsymbol{\theta}}$,
$\Gamma_{T - 1} := \max_{\hat{\boldsymbol{\theta}} \in \boldsymbol{\Theta}} \gamma_{T}(\hat{\boldsymbol{\theta}})$,
and $\mathcal{N}_{\varepsilon_T}(\boldsymbol{\Theta}, d)$ denotes the covering number of $\boldsymbol{\Theta}$ with respect to the metric $d$.
\end{restatable}

\textit{Proof sketch.} We start by applying \cref{lemma:slater-condition} for each round to bound $r_t = f(\x^\ast) - f(\x_t)$. The optimism principle of UCB acquisition allows us to bound $\sum_{t = 1}^T r_t$ in terms of the cumulative loss $\sum_{t = 1}^T L_t^s(\x_t; \hat{\btheta}_{t - 1})$. Then, we apply \cref{lemma:convergence} to obtain an upper-bound in terms of the cumulative predictive variance corresponding to the true kernel $k_\btheta$. By applying \cite[Lemma 5.4]{srinivas2009gaussian}, we can obtain the OSCBO regret bound in terms of the maximum information gain $\gamma_{T}(\btheta)$.

The bound is looser than classical GP-UCB with known kernel hyperparameters because online hyperparameter adaptation introduces primal/dual regret terms and a covering-number inflation of $\beta_t$. Since these costs remain sublinear, OSCBO still achieves sublinear regret. By contrast, generic online-learning approaches to BO may incur a per-round learning cost and hence linear regret~\citep{hoffman2011portfolio}.

% As expected, the regret bound yielded by OSCBO is less tight than GP-UCB with known kernel hyperparameters. We observe that the primal and dual RMs' regret bound, as well as the extended $\beta_t$ value due to the covering number, contribute significantly to the additional term in our regret bound. Fortunately, both primal and dual RMs' regret bounds grow sublinearly, resulting in a sublinear OSCBO regret bound. Notably, OSCBO differs from standard online learning approaches used in Bayesian optimization (BO), which typically bound the per-round learning cost and consequently yield linear regret \citep{hoffman2011portfolio}.

\vspace{-.25cm}
\section{Experiments} 
\vspace{-.25cm}

\begin{figure}[t]
    \centering
    \includegraphics[width=.95\linewidth]{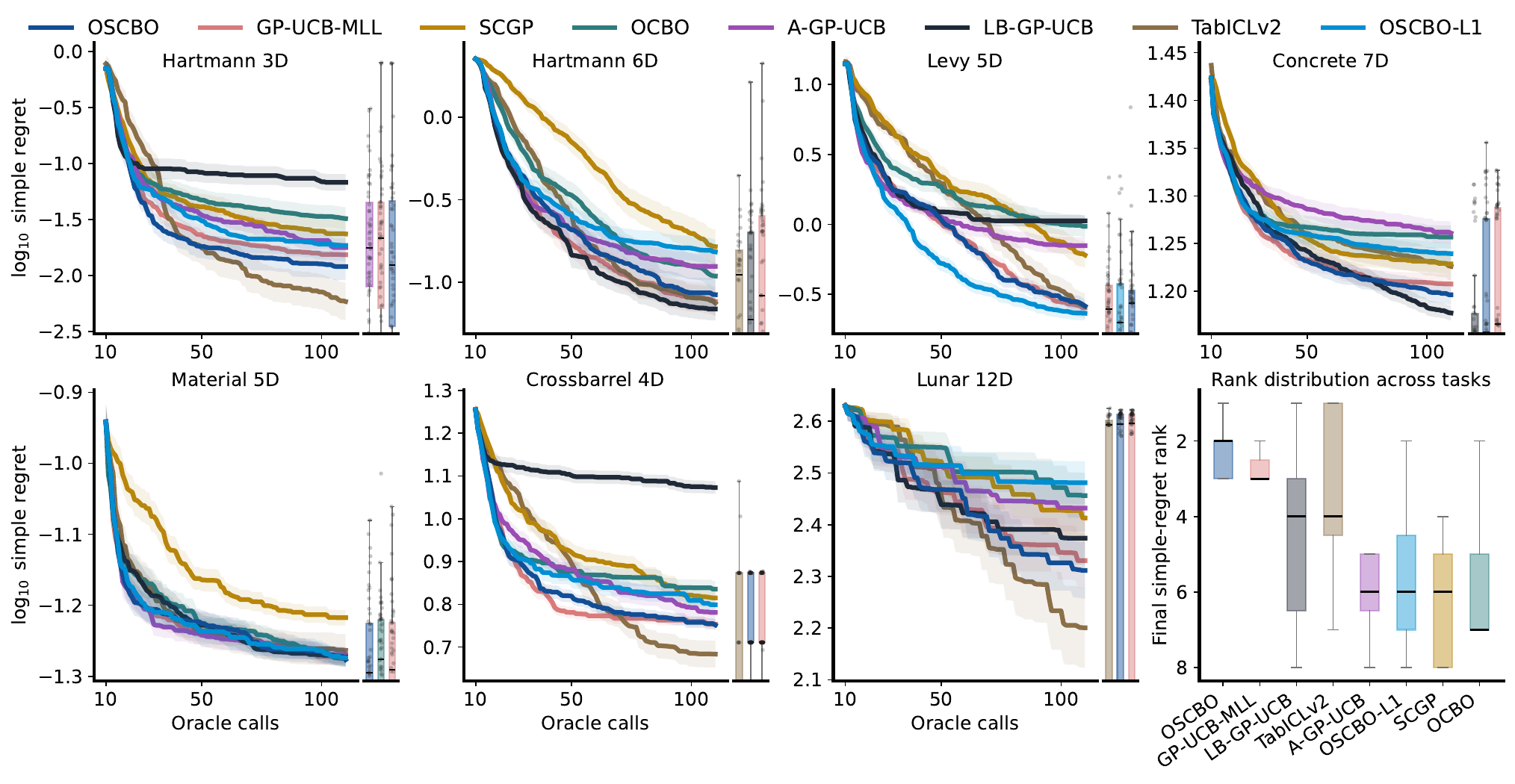}
    \vspace{-0.28cm}
\caption{\textbf{Synthetic and real-world benchmarks.}
Each plot displays simple regret across tasks and final simple-regret distribution for the top three methods,
with mean $\pm$ standard error over 20 seeds.
%Small right insets show the final simple-regret distribution for the top three methods.
Lower-right: final simple-regret rank distribution across tasks (lower is better).
\textbf{OSCBO consistently attains competitive final simple regret, ranking among the strongest methods overall.}}
 %highlighting stable exploration–exploitation behavior over the full budget}.}
    \label{fig:results}
\end{figure}

% \begin{figure*}
%     \centering
%     \includegraphics[width=1\linewidth]{Manuscript/fig/plot_regrets.pdf}
%     \caption{\textbf{Synthetic and real-world benchmarks:} simple regret (top) and cumulative regret (bottom), mean $\pm \tfrac{1}{2}$ std over 20 seeds.
%     Lower-right: final rank trade-off across tasks/seeds (lower is better; error bars: $\pm$1 std).
% \textbf{OSCBO is near the bottom-left trade-off frontier, matching or improving upon the best baseline in simple regret while remaining competitive in cumulative regret.}} %highlighting stable exploration–exploitation behavior over the full budget}.}
%     \label{fig:results}
% \end{figure*}

\begin{comment}
\begin{figure}
    \centering
    \includegraphics[width=1\linewidth]{Manuscript/fig/overall_normalized_infinitewidth_and_ablations.pdf}
\caption{
\textbf{Ablation study.}
We report normalized simple regret averaged over 4 tasks: Hartmann3D, Hartmann6D, Styblinskitang2D, and Levy5D. Mean $\pm \tfrac{1}{2}$ std reported across 20 seeds.
(\textbf{a}) Comparing OSCBO to OCBO when optimizing the Infinite-Width BNN kernel.
(\textbf{b}) Controlling the sharpness ($v$), calibration ($\lambda$), and play/recovery scheduling in OSCBO.
}
\label{fig:ablation}
\vspace{-.65cm}
\end{figure}
\end{comment}

\textbf{Baselines.}
We report two variants of our method: OSCBO ($p=2$ in
calibration penalty Equation~\ref{eq:calibration-constraint}), and OSCBO-L1, which uses $p=1$, GP-UCB-MLL~\cite{srinivas2009gaussian}, Online Calibration BO (OCBO)~\cite{deshpande2024calibratedregressionadversaryregret}, Adaptive GP-UCB (A-GP-UCB)~\cite{berkenkamp2019no}, Lengthscale-Balancing GP-UCB (LB-GP-UCB)~\cite{ziomek2024bayesian}, Sharp Calibrated GP (SCGP)~\cite{capone2023sharp}, and TabICLv2 for BO~\cite{muller2023pfns4bo,qu2026tabiclv2,yu2026gitbo}. Details are available in Appendix~\ref{sec:baseline-details}.
% These baselines cover MLL-driven hyperparameter refitting, online calibration of confidence widths, adaptive lengthscale schedules, sharp calibrated uncertainty, and tabular foundation-model surrogates.
% Performance is reported via simple regret and cumulative regret.

\textbf{Implementation details.}
OSCBO and the GP-based baselines are implemented in BoTorch~\citep{balandat2020botorch}; TabICLv2 is used through a custom BO wrapper. Unless specified otherwise, we use an isotropic Matérn kernel with $\nu=2.5$, fixed output scale, and fixed observation noise, estimating only the lengthscale. We evaluate on synthetic benchmarks and real-world tasks, including Lunar 12D~\cite{eriksson2019scalable}, Material 5D~\cite{mekki2021two,ziomek2024bayesian}, Concrete 7D~\cite{yuan2026unleashing}, and Crossbarrel 4D~\cite{gongora2020bayesian}.
 For each task, we report mean $\pm$ standard error over 20 seeds. Further details on tasks and hyperparameters are given in~\cref{sec:expdetails}.

\vspace{-.2cm}
\subsection{Results}
\label{sec:results}
\vspace{-.2cm}

\cref{fig:results} reports simple regret across synthetic and real-world benchmarks. 
OSCBO is consistently among the strongest methods: it matches or improves upon GP-UCB-MLL on several tasks, while avoiding the weaker behavior of calibration-only or schedule-based alternatives. 
OSCBO-L1 is slightly less favorable than the default OSCBO variant.
Together, these results suggest that the
sharpness--calibration update provides an effective way to adapt the lengthscale,
with the choice of calibration penalty controlling the trade-off between final
simple regret and trajectory-level performance.
TabICLv2's competitive but non-dominant performance suggests that bypassing GP lengthscale optimization does not remove the value of explicit lengthscale adaptation.
%Interestingly, TabICLv2 ranks
%third overall, highlighting the promise of tabular foundation-model surrogates
%that bypass GP lengthscale optimization altogether.

The cumulative-regret analogue in \cref{fig:supresults} shows a complementary
picture. OSCBO-L1 achieves the best aggregate cumulative-regret rank, whereas
the default OSCBO lies in the main GP-baseline cluster. Together with the
simple-regret results, this points to a trade-off controlled by the calibration
exponent: $p=2$ favors final-solution quality, while $p=1$ improves
budget-efficient performance along the trajectory. We return to this analysis in the next section.

\begin{figure}[t]
    \centering
    \includegraphics[width=1\linewidth]{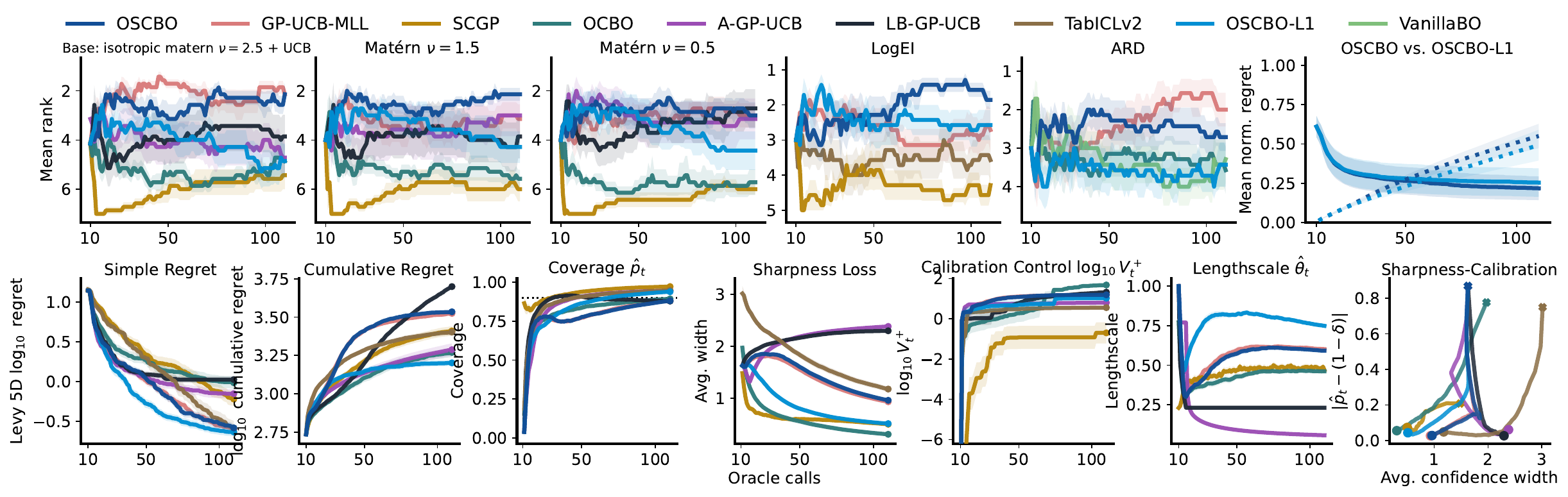}
        \vspace{-0.7cm}
\caption{\textbf{Sensitivity analysis and diagnostics.}
\textbf{Top:} Mean rank $\pm$ standard error across tasks. Regret is first averaged over 20 seeds for each task and baseline; baselines are then ranked within each task, and ranks are averaged across tasks.
For the last panel, mean normalized regrets are shown. \textbf{OSCBO demonstrates consistent performance across these different regimes, highlighting its robustness. Bottom:} Levy 5D diagnostics, \textbf{showing that OSCBO
aligns closely with GP-UCB-MLL, while the $p=1$ OSCBO-L1 penalty induces
distinct regret and uncertainty dynamics.}}
    \label{fig:ablation}
\end{figure}
\vspace{-.3cm}
% \begin{figure}[t]
%     \centering
%     \includegraphics[width=1\linewidth]{NeurIPS_2026/figures_neurips/levy_sharpness_calibration.pdf}
%     \caption{\textbf{Diagnostics on Levy 5D}. Curves show mean performance and uncertainty-quality diagnostics over 20 seeds $\pm$ standard error.
% }
%     \label{fig:levy_sharp}
% \end{figure}

\subsection{Sensitivity and Ablation Analysis}

We perform ablatios on the BO loop using the configuration from our main experiments: an isotropic Mat\'ern-$5/2$ GP surrogate, UCB acquisition, and the default $p=2$ primal calibration update.  Each panel includes only baselines for which the corresponding modification is meaningful and available.

\textbf{Robustness to kernel and acquisition choices.}
Figure~\ref{fig:ablation} shows that OSCBO is not tied to a single kernel-acquisition pair. (task-averaged rank are shown; see Figure~\ref{fig:ablation_grid} for raw regret trajectories).
%provides the corresponding raw regret trajectories before rank aggregation. 
Varying the Matérn smoothness consistently keeps OSCBO among the top-performing methods across trials. Furthermore, replacing UCB with LogEI maintains OSCBO's competitiveness, suggesting that the sharpness--calibration update is effective beyond UCB.
%Of note, TabICLv2 struggles under LogEI, despite strong performance with UCB, suggesting that its performance is more sensitive to the choice of acquisition function.
When using ARD lengthscales, OSCBO performs slightly worse than GP-UCB-MLL, primarily due to the Lunar task (\cref{fig:ablation_grid}).
Since ARD
introduces one lengthscale per dimension, we also include VanillaBO~\citep{pmlr-v235-hvarfner24a}; however, its behavior is not consistently stronger, possibly because the benchmarks remain moderate in dimension. % dimensional.
% possibly because the benchmarks remain only moderately dimensional.
% Finally, let us observe that with ARD lengthscales, OSCBO incurs a small loss versus GP-UCB-MLL, driven by poor performance on Lunar (Figure~\ref{fig:ablation_grid}).

\textbf{Effect of the calibration penalty and diagnostics.}
Figure~\ref{fig:ablation} (top-right panel) summarizes the simple--cumulative regret trade-off between OSCBO and OSCBO-L1, while the diagnostic panels (Figure~\ref{fig:ablation}, bottom row; Figure~\ref{fig:diagnostics_grid}) illustrate the underlying dynamics on Levy 5D. The two variants differ in the exponent $p$ of the calibration constraint (Equation~\ref{eq:calibration-constraint}), which enters both the FTPL and OMD objectives (Equations~\ref{eq:ftpl_objective} and~\ref{eq:omd-update}). With $p=2$, the quadratic calibration penalty resembles the data-fit geometry of the GP marginal likelihood, and OSCBO closely tracks GP-UCB-MLL in regret, sharpness--calibration, and lengthscale trajectories. With $p=1$, large residuals are penalized only linearly, reducing the influence of rare large prediction errors and leading to fewer lengthscale corrections. This shifts the sharpness--calibration trade-off: OSCBO-L1 improves cumulative regret, but at the cost of slightly weaker final simple regret across tasks.

\vspace{-.3cm}
\section{Discussion}
\vspace{-.3cm}

We introduced OSCBO after a core observation: under adaptive querying, the
exploration--exploitation trade-off can be viewed as a trade-off between
\emph{sharp} and \emph{calibrated} uncertainty. Treating GP hyperparameter
refitting as an online decision problem lets us formalize and control this
tension along the BO trajectory. In particular, the default quadratic OSCBO
update mirrors the log-determinant and quadratic data-fit structure of GP
marginal likelihood, which explains its close empirical alignment with
GP-UCB-MLL.
However, unlike black-box MLL refitting, OSCBO places this update
inside a constrained online-learning procedure, yielding a sublinear regret
analysis for the resulting GP-UCB rule.  More broadly, our results show that the
refitting objective itself is a design choice: changing the primal calibration
penalty alters the lengthscale dynamics, with OSCBO-L1 providing a more
budget-efficient trajectory at a small cost in final simple regret.
A particularly promising avenue is to transfer OSCBO-style sharpness/calibration control to surrogates with more challenging uncertainty, like Bayesian neural networks~\cite{li2024a} or scalable approximate GPs. Lastly, our theory covers only UCB; extending it to strategies like Expected Improvement or Thompson Sampling is left for future work.
%Finally, while we focus on UCB, the framework extends to other acquisitions (e.g., EI, TS).

\bibliographystyle{plain}
\bibliography{references}

%%%%%%%%%%%%%%%%%%%%%%%%%%%%%%%%%%%%%%%%%%%%%%%%%%%%%%%%%%%%

%%%%%%%%%%%%%%%%%%%%%%%%%%%%%%%%%%%%%%%%%%%%%%%%%%%%%%%%%%%%%%%%%%%%%%%%%%%%%%%
%%%%%%%%%%%%%%%%%%%%%%%%%%%%%%%%%%%%%%%%%%%%%%%%%%%%%%%%%%%%%%%%%%%%%%%%%%%%%%%
% APPENDIX
%%%%%%%%%%%%%%%%%%%%%%%%%%%%%%%%%%%%%%%%%%%%%%%%%%%%%%%%%%%%%%%%%%%%%%%%%%%%%%%
%%%%%%%%%%%%%%%%%%%%%%%%%%%%%%%%%%%%%%%%%%%%%%%%%%%%%%%%%%%%%%%%%%%%%%%%%%%%%%%
\newpage
\appendix

\setcounter{figure}{0}
\setcounter{table}{0}
\setcounter{equation}{0} % optional
\renewcommand{\thefigure}{S\arabic{figure}}
\renewcommand{\thetable}{S\arabic{table}}
\renewcommand{\theequation}{S\arabic{equation}}

\onecolumn

\begin{center}
\bf{\Large{Supplementary Materials}}
\end{center}

The appendix is organized as follows:
\begin{itemize}
    \item \cref{sec:addres} reports complementary experimental results referenced in the main text, including:
    \begin{itemize}
        \item the cumulative-regret analogue of the main benchmark results (Figure~\ref{fig:supresults});
        \item per-task raw regret trajectories for the robustness and ablation study (Figure~\ref{fig:ablation_grid});
        \item per-task diagnostic trajectories for regret, coverage, confidence width, calibration violation, sharpness--calibration dynamics, and learned lengthscales (Figure~\ref{fig:diagnostics_grid}).
        \item computation time (Figure~\ref{fig:time})
    \end{itemize}
    
    \item \cref{sec:detailsalgo} provides implementation-level details of OSCBO, including:
    \begin{itemize}
        \item the \texttt{RECALIBRATE} primal--dual update;
        \item the regret-minimizer interface (\cref{app:rm_interface});
        \item the primal regret minimizer based on Follow-the-Perturbed-Leader (FTPL; \cref{app:ftpl});
        \item the dual regret minimizer based on Online Mirror Descent (OMD; \cref{app:omd}).
    \end{itemize}
    
    \item \cref{app:proofs} contains the full proofs of the theoretical results stated in the main text, including the sharpness and calibration corollaries, the feasibility argument, the online-learning guarantees, and the final GP-UCB-style regret bound.
    
    \item \cref{sec:expdetails} describes the experimental setup, including:
    \begin{itemize}
        \item the BO baselines and their implementation details (\cref{sec:baseline-details});
        \item the hyperparameter settings used in all experiments (\cref{sec:hyperparameter-details});
        \item the GP kernels used in the benchmark and robustness experiments (\cref{app:ibnn});
        \item the compute resources used for the experiments (\cref{sec:hardware-details}).
    \end{itemize}

    \item \cref{sec:testfunc} details the benchmark functions and oracle construction, including:
    \begin{itemize}
        \item input/output normalization and preprocessing;
        \item the synthetic test functions (\cref{sec:sb});
        \item the Lunar Lander benchmark (\cref{sec:rwb});
        \item the tabular real-world benchmarks (\cref{sec:rwtb}).
    \end{itemize}
\end{itemize}

\section{Additional results}\label{sec:addres}

\begin{figure}[H]
    \centering
    \includegraphics[width=1\linewidth]{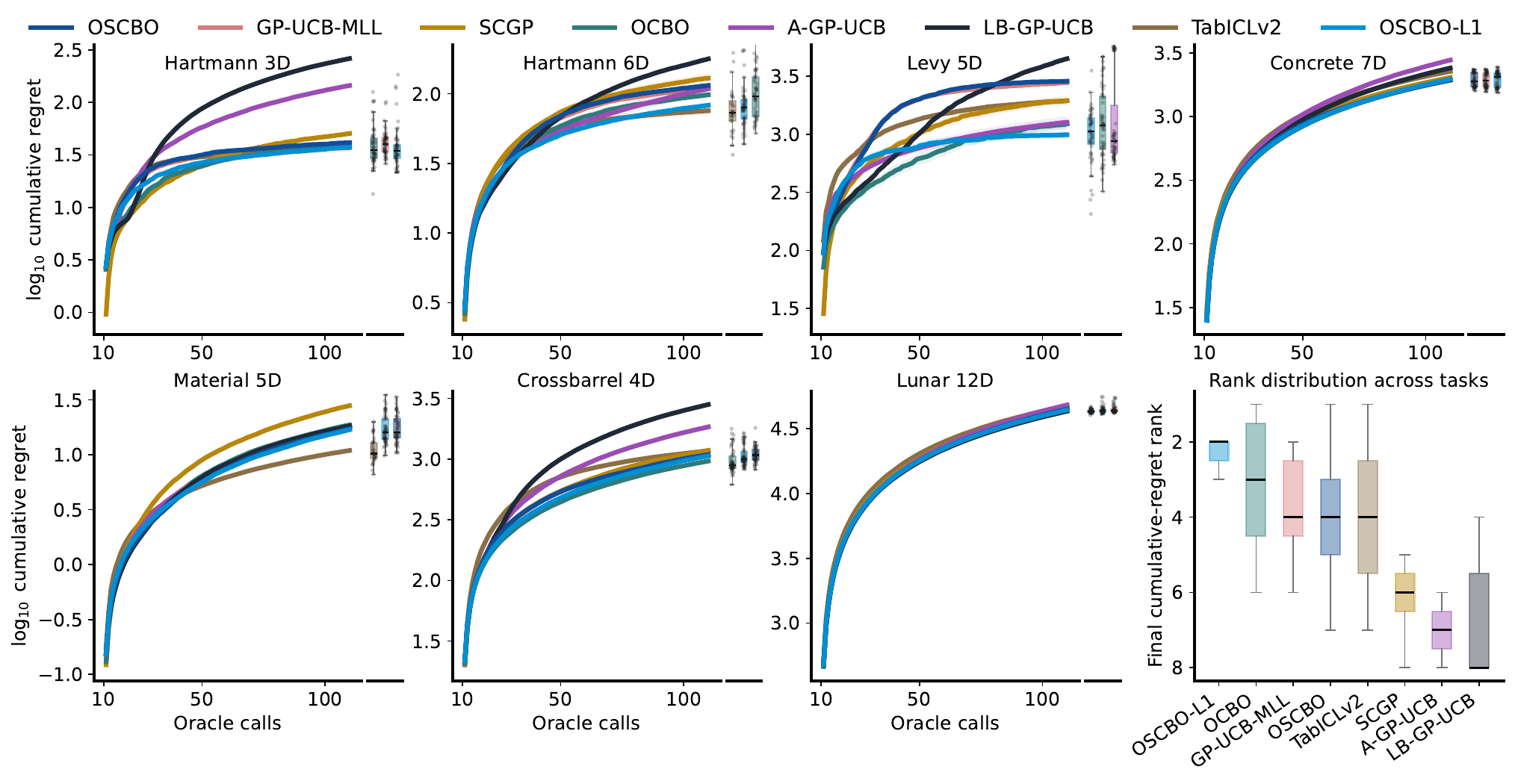}
\caption{\textbf{Cumulative regret analogue of Figure~\ref{fig:results}, synthetic and real-world benchmarks.}
Each plot displays cumulative regret across tasks and final cumulative-regret distribution for the top three methods,
with mean $\pm$ standard error over 20 seeds.
%Small right insets show the final simple-regret distribution for the top three methods.
Lower-right: final cumulative-regret rank distribution across tasks (lower is better).
\textbf{OSCBO-L1 achieves the best aggregate cumulative-regret rank, while OSCBO lies in the main GP-baseline cluster, suggesting that the $L_1$ ($p=1$) improves accumulated regret by reducing sensitivity to large residuals.}}
 %highlighting stable exploration–exploitation behavior over the full budget}.}
    \label{fig:supresults}
\end{figure}

\begin{figure}[H]
    \centering
    \includegraphics[width=1\linewidth]{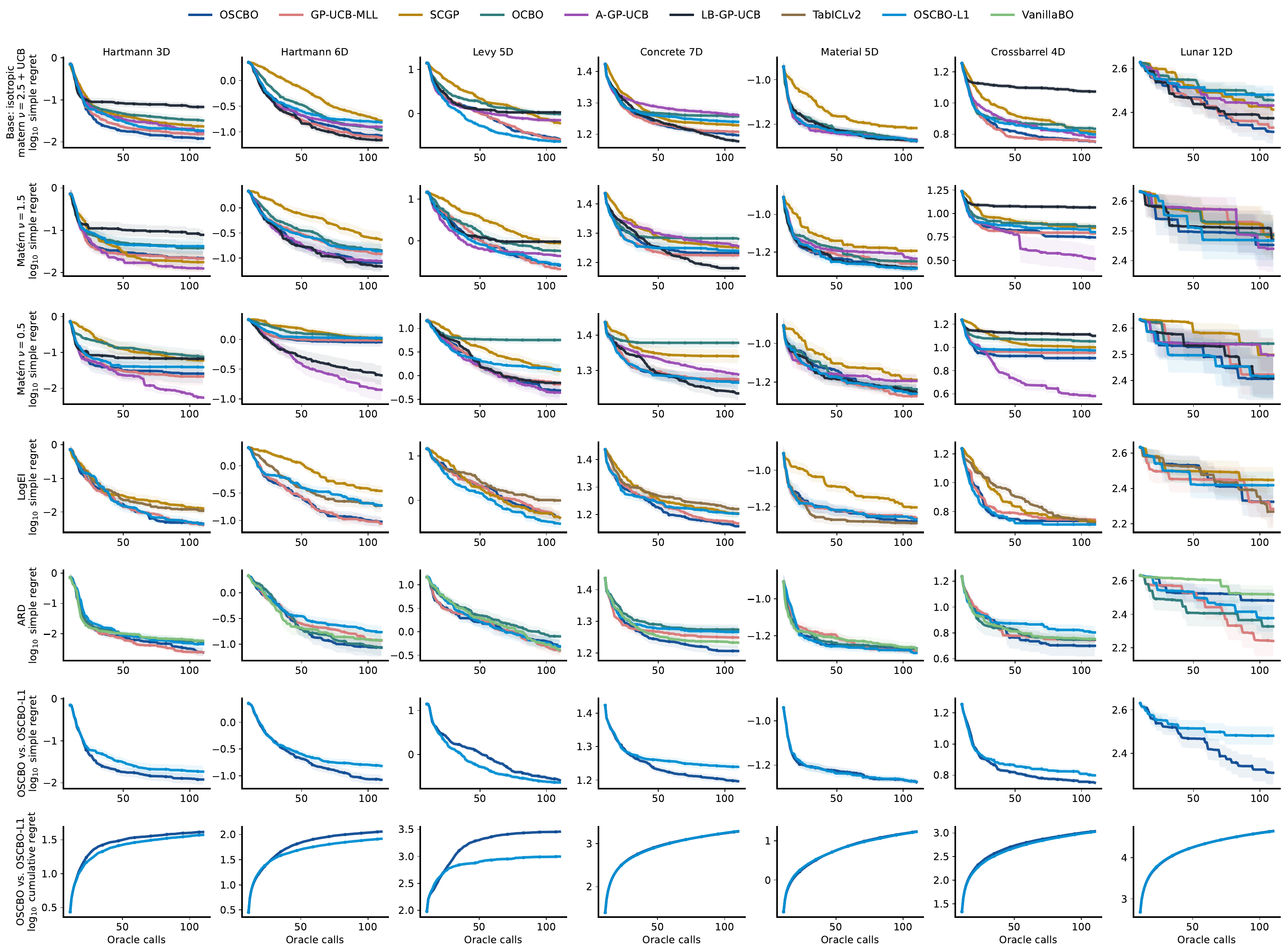}
    \caption{\textbf{Per-task ablation results}. Raw regret trajectories corresponding to the aggregated rank results in Figure~\ref{fig:ablation}. All panels report $\log_{10}$ simple regret, except the last row, which reports $\log_{10}$ cumulative regret for OSCBO vs. OSCBO-L1. Mean $\pm$ standard error over 20 seeds. Lower is better.}
    \label{fig:ablation_grid}
\end{figure}

\begin{figure}[H]
    \centering
    \includegraphics[width=1\linewidth]{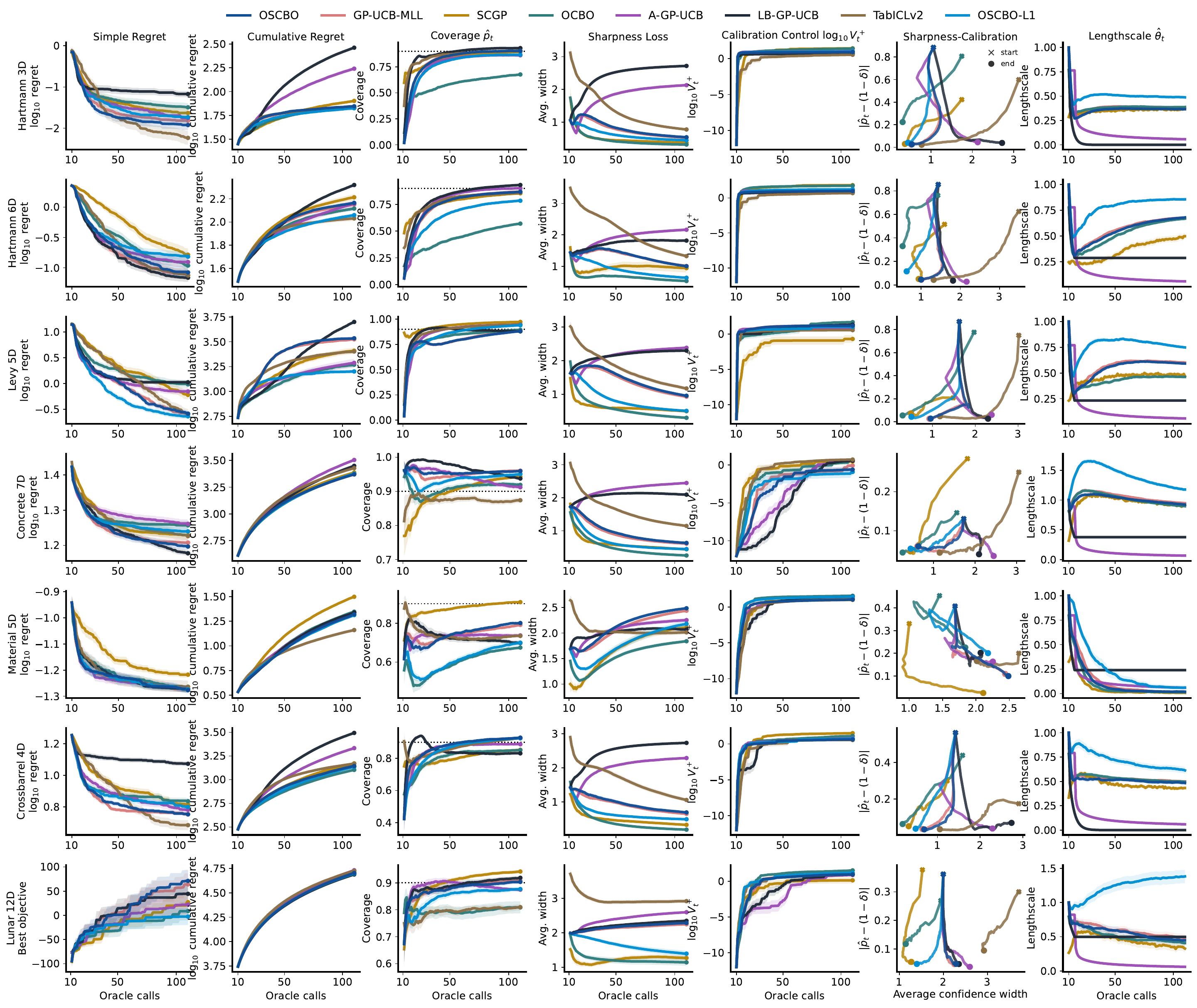}
\caption{
\textbf{Per-task diagnostics.}
Rows correspond to benchmark tasks and columns report simple regret, cumulative
regret, empirical coverage $\hat p_t$, average confidence width, cumulative
positive calibration violation $\log_{10} V_t^+$, the sharpness--calibration
trajectory, and the learned lengthscale $\hat\theta_t$. Mean
$\pm$ standard error over 20 seeds. \textbf{OSCBO often
tracks GP-UCB-MLL closely, while OSCBO-L1 induces distinct lengthscale and
uncertainty dynamics, consistent with its improved cumulative-regret behavior.
}}
\label{fig:diagnostics_grid}
\end{figure}

\begin{figure}[H]
    \centering
    \includegraphics[width=1\linewidth]{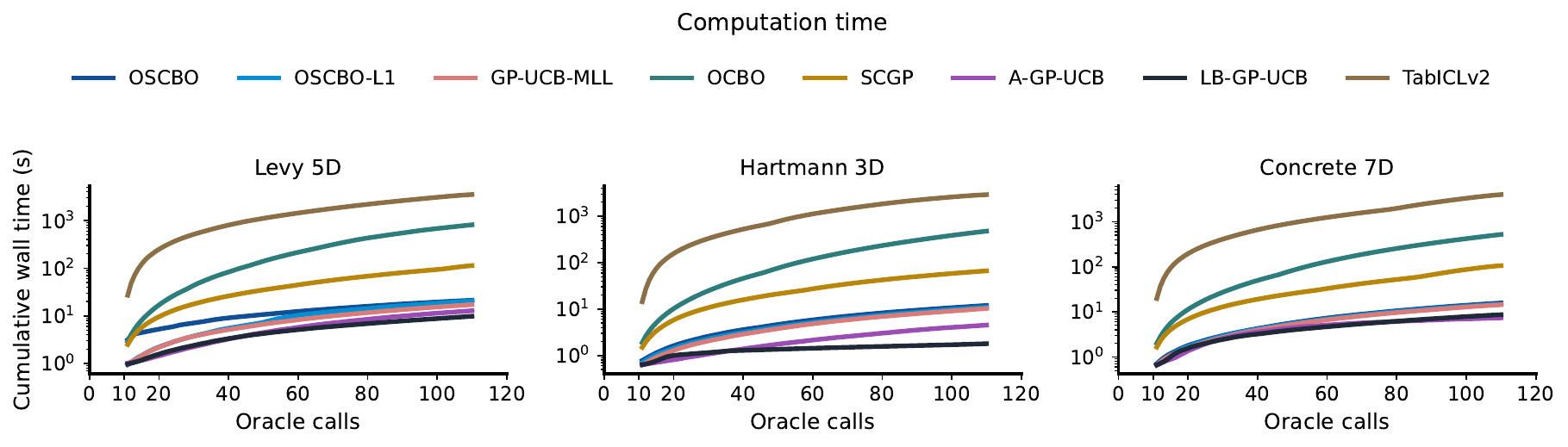}
    \caption{\textbf{Computation time.} Cumulative wall-clock time as a function of oracle calls on Levy 5D, Hartmann 3D, and Concrete 7D. Mean $\pm$ standard error over 10 seeds.}
    \label{fig:time}
\end{figure}

\section{Details of Algorithm}\label{sec:detailsalgo}

\subsection{Primal-Dual Update Algorithm}

In this section, we provide the \texttt{RECALIBRATE} subroutine, which utilizes the primal and dual RMs.

\begin{algorithm}[H]
   \caption{\texttt{RECALIBRATE} \citep{castiglioni2022unifying}}
   \label{alg:recalibrate}
\begin{algorithmic}[1]
   \STATE {\bfseries Input:} $\mathcal{R}^P$, $\mathcal{R}^D$ 
    \STATE Play $\hat{\btheta}_{t - 1}$ and get $L^{\sharpness}_t$ and $L^\calibration_t$

    \algcomment{Primal RM update}
    \STATE $u_t^P(\hat{\btheta}_{t - 1}) \leftarrow \, L^\sharpness_t(\hat{\btheta}_{t - 1}) + \lambda_{t - 1} L^\calibration_t(\hat{\btheta}_{t - 1})$
    
    \STATE $\mathcal{R}^P.\texttt{OBSERVEUTILITY}(u_t^P(\hat{\btheta}_{t - 1}))$
    \algcomment{Dual RM update}
    \STATE $u_t^D(\lambda) \mapsto - L_t^s(\hat{\btheta}_{t - 1})  - \lambda L_t^\calibration(\hat{\btheta}_{t - 1})$
    \STATE $\mathcal{R}^D.\texttt{OBSERVEUTILITY}(u_t^D)$
\end{algorithmic}
\end{algorithm}

\subsection{Regret-minimizer interface}\label{app:rm_interface}
A \emph{regret minimizer} (RM) is an online decision-making routine over a domain $\calW$ interacting with a black-box environment. 
At each round $t\in[T]$, the RM (i) outputs a decision $\w_t\in\calW$ via \texttt{NEXTELEMENT}, then (ii) receives a bounded utility function $u_t:\calW\to[c,d]$ and updates its internal state via \texttt{OBSERVEUTILITY}. 
The utility $u_t$ may be history-dependent (e.g., depend on $\w_1,\dots,\w_{t-1}$), matching adaptive data collection. In this paper, we define the primal utility at round $t$ as
\begin{equation}
    u_t^P(\hat{\btheta}_{t - 1}) = L_t^s(\hat{\btheta}_{t - 1}) + \lambda_{t - 1} L_t^c(\hat{\btheta}_{t - 1}),
\end{equation}
while the dual utility as
\begin{equation}
u_t^D(\lambda) = - L_t^s(\hat{\btheta}_{t - 1}) - \lambda L_t^c(\hat{\btheta}_{t - 1}).
\end{equation}
We write $\texttt{INIT}(\calW,[c,d],\eta)$ for an RM initialization routine that returns an instance tailored to $\calW$ space and utility range $[c,d]$, with a failure probability at most $\eta$. OSCBO considers the follow-the-perturbed-leader algorithm to act as the primal RM and the online mirror descent algorithm as the dual RM. Given a current round $t$ and the OSCBO parameter $\eta \in (0, 1)$, the involved RMs are controlled by the quantity $M_{\tilde{\rho}}$:
\begin{equation}
    M_{\tilde{\rho}} = \frac{2}{\tilde{\rho}} \sqrt{T} + \left( 2 + \frac{3}{\tilde{\rho}} \right) \calE_{t, \eta} + \left(1 + \frac{2}{\tilde{\rho}}\right) R_T^P + \frac{1}{\tilde{\rho}} R_T^D,  
\end{equation}
where $R_T^P$ and $R_T^D$ are the primal and dual RM cumulative regret bounds, respectively, and $\calE_{t, \eta} = \sqrt{8t \log(18 t^2 / \eta)}$ is the value bounding differences between expectations and empirical means of constraint functions \cite{castiglioni2022unifying}.

\subsection{Primal RM: Follow the perturbed leader (FTPL)}\label{app:ftpl}

At round $t$, \texttt{OSCBO} invokes the \texttt{RECALIBRATE} subroutine, which yields a primal
utility feedback $u_t^P$. Depending on the phase, either $\calR_I^P$ or $\calR_{II}^P$ is called; both invoke
\texttt{OBSERVEUTILITY} and update
\begin{equation}\label{eq:ftpl-update}
\hat{\btheta}_{t + 1}
=
\arg\min_{\hat{\btheta} \in \bTheta} \calL_t^P(\hat{\btheta}) =
\sum_{j = 1}^{t} u^P_j(\hat{\btheta}) - \inner{\hat{\btheta}}{\bseta},
\end{equation}
where the perturbations $\bseta$ are drawn (componentwise) from a probability distribution. For the theoretical analysis, we assume an exponential distribution with parameter $\tau$. However, in our experiments we instead use an isotropic multivariate normal distribution. As noted by \cite{suggala2020online}, the regret bounds of the FTPL algorithm extend beyond the exponential distribution and hold for a broader class of perturbation distributions. We provide the practical formulation of the objective above through the following proposition:
\begin{restatable}{proposition}{propositionutilityprimal}\label{prop:utility-primal}
Let $\K_\btheta^{t}$ be the covariance matrix of inputs $\X_{t} = [\x_1, \dots, \x_{t}]^\top \in \real^{t \times d}$ and the observation evaluations $\y_{t} = [y_1, \dots, y_{t}]^\top \in \real^{t}$. For any $\btheta \in \bTheta$, define the primal utility function at round $1 \le j \le t$ as $u_j^P(\btheta) = L_j^s(\btheta) + \lambda_{j - 1} L_j^c(\btheta)$, where $L_j^s$ and $L_j^c$ denote the sharpness loss and the calibration constraint, respectively. Then, we define the FTPL objective  after $t$ rounds as
\begin{equation}
\calL_t^P(\btheta) = \frac{\left( \log \det(\K_\btheta^t + \sigma^2 \I) - t \log(\sigma^2) \right) }{\log(1 + \sigma^{-2})} + \boldsymbol{\lambda}_t^\top \vert \mathrm{diag}(\vv_t) \, (\LL_\btheta^t)^{-1} \y_t  \vert^p - \sum_{j=1}^t \lambda_{j - 1} - \inner{\bseta}{\btheta}
\label{eq:ftpl_objective}
\end{equation}
where $\LL_\btheta^t$ is a Cholesky decomposition, satisfying $\K_\btheta^t + \sigma^2 \I = \LL_\btheta^t (\LL_\btheta^t)^\top$, $\bseta$ is a noise drawn per action, $\vert . \vert^p$ denotes the element-wise absolute value raised to the power of $p$, $\vv_t = [1 / \beta_1^{1/2}(\delta), \dots, 1 / \beta_t^{1/2}(\delta)]^\top$, and $\boldsymbol{\lambda}_t = [\lambda_1, \dots, \lambda_t]$.
\end{restatable}
The full proof can be found in \cref{app:prop-utility-primal}. The FTPL algorithm requires two assumptions to establish a sublinear regret guarantee $\mathcal{O}(\sqrt{T})$ \citep{suggala2020online}. First, it assumes that $\bTheta$ is bounded with $\ell_\infty$-diameter $
D=\sup_{\hat{\btheta},\,\tilde{\btheta}\in\bTheta}
\|\hat{\btheta}-\tilde{\btheta}\|_\infty$. Next, FTPL algorithm assumes that each $u_j^P, 1\le j \le t$ is $L$-Lipschitz w.r.t.\ the $\ell_1$ norm:
$
\forall \hat{\btheta},\tilde{\btheta}\in\bTheta, |u_j^P(\hat{\btheta})-u_j^P(\tilde{\btheta})|
\le
L\|\hat{\btheta}-\tilde{\btheta}\|_1$.

We show that $u_j^P$ satisfies this condition via \cref{lemma:utility-lispchitz-property}. Following \cite{suggala2020online}, we assume access to an approximate optimization oracle $\mathbf{O}_{g,h}$
for solving~\eqref{eq:ftpl-update}.
The oracle takes as input a function $u:\bTheta\to\real$ and a vector $\bseta$, and returns an approximate minimizer of $ \sum_{j = 1}^t u_j^P(\hat{\btheta})-\inner{\bseta}{\hat{\btheta}}$. An oracle is called a $(g,h)$-approximate optimization oracle if it returns $\btheta^\ast\in\bTheta$ such that
\begin{equation}
\sum_{j = 1}^t u_j^P(\btheta^\ast)-\inner{\bseta}{\btheta^\ast}
\le
\inf_{\btheta\in\bTheta}\left( \sum_{j=1}^t u_j^P(\btheta)-\inner{\bseta}{\btheta} \right)
+
\big(g+h\|\bseta\|_1\big).
\end{equation}
In practice, we use Adam \citep{Kingma2015}
as our optimization oracle.

\subsection{Dual RM: Online mirror descent (OMD)}\label{app:omd}

At round $t$, \texttt{OSCBO} invokes \texttt{RECALIBRATE}, which yields the dual utility feedback
$u_t^D(\lambda_t)$.
Depending on the phase, either $\calR_I^D$ or $\calR_{II}^D$ is called; both invoke
\texttt{OBSERVEUTILITY} and update the scalar multiplier $\lambda_t$ via online mirror descent:
\begin{equation}\label{eq:omd-update}
\lambda_{t}
=
\arg\min_{\lambda\in\calW}
\Big(\nabla u_t^D(\lambda_{t - 1})\Big)\,\lambda
+\frac{1}{\omega_t}\,D_\psi(\lambda,\lambda_{t - 1}),
\end{equation}
where $D_\psi:\real_{\geq 0} \times\real_{> 0} \to\real$ is the \emph{Bregman} divergence induced by a
strictly convex $\psi:\real_{\geq 0}\to\real$, and $\omega_t>0$ is the step size. By the definition of $u_t^D$, we find that $\nabla \,  u_t^D(\lambda) = - L_t^c(\hat{\btheta}_{t - 1})$ During the play phase we set $\calW=\calS_{\tilde\rho} := \{\lambda\in\mathbb{R}_+:|\lambda|\le 1/\tilde\rho\}$, while during the recovery phase we set
$\calW=\Delta:=[0,1]$.

For any $\lambda, \lambda'$, the Bregman divergence is $D_\psi(\lambda,\lambda')
=
\psi(\lambda)-\psi(\lambda')
-\nabla\psi(\lambda')\,(\lambda-\lambda')$. It satisfies $D_\psi(\lambda,\lambda')\ge 0$ and is generally not symmetric.
We use the negative entropy regularizer $\psi(\lambda)=\lambda\log\lambda - \lambda$, which yields
\begin{equation}\label{eq:negative-entropy}
D_\psi(\lambda,\lambda')
=
\lambda\log\frac{\lambda}{\lambda'}-\lambda+\lambda' \leq \bigo\left( \frac{1}{\tilde{\rho}} \log\left(\frac{1}{\tilde{\rho}}\right) \right).
\end{equation}
Under this choice of $\psi$, the OMD update in~\eqref{eq:omd-update} admits closed-form solutions
on $\Delta$ and on $\calS_{\tilde\rho}$.
On $\Delta$, the update is
\begin{equation}\label{eq:omd-simplex}
\lambda_{t}
=
\min(\lambda_{t - 1} \exp( - \omega_t \nabla u_t^D(\lambda_{t - 1} )), \, 1),
\end{equation}
and on $\calS_{\tilde{\rho}}$, the update becomes
\begin{equation}\label{eq:omd-capped}
\lambda_{t}
=
\min \! \Big(\lambda_{t - 1} \exp\big(- \omega_t \nabla u_t^D(\lambda_{t - 1})\big),\, 1/\tilde{\rho}\Big).
\end{equation}

Since $D_\psi$ is defined for $\lambda,\lambda'>0$, we initialize $\lambda_1>0$ and the multiplicative
form of the update preserves positivity; see \cite{orabona2019modern} for general conditions. OMD can also be instantiated with alternative regularizers (e.g., Euclidean regularizer $\psi(\lambda)=\tfrac12\lambda^2$).

\section{Proofs}\label{app:proofs}

\subsection{\cref{prop:utility-primal}} \label{app:prop-utility-primal}

\propositionutilityprimal*

\textit{Proof:}

Given a hyperparameter $\btheta \in \bTheta$, the cumulative primal utility is defined as
\begin{equation}
   \sum_{j = 1}^t u_j^P(\btheta)  =\sum_{j = 1}^t L_j^s(\btheta) + \lambda_{j - 1} L_j^c(\btheta),
\end{equation}
where $L_j^s$ and $L_j^c$ denote the sharpness loss and constraint calibration at round $j$, defined in \cref{eq:sharpness-reward} and \cref{eq:calibration-constraint}, respectively. Our goal is to derive the closed form of each term. Based on \cite[Lemma 5.3]{srinivas2009gaussian} and the definition of information gain under multivariate normal distribution, the term $\sum_{j = 1}^t L_j^s(\btheta)$ can be expressed in terms of the log determinant of the covariance matrix $\K_\btheta^{t}$:
\begin{align}
    \sum_{j = 1}^t L_j^s(\btheta)  &= \frac{1}{\log(1 + \sigma^{-2})} \log \det(\I + \sigma^{-2} \K_\btheta^{t}) \nonumber \\
    &= \frac{1}{\log(1 + \sigma^{-2})} \left( \log \det(\K_\btheta^t + \sigma^2 \I) - t \log(\sigma^2) \right) 
\end{align}
Let $\LL_\btheta^{-1}$ be the Cholesky decomposition of $\K_\btheta^{t} + \sigma^2 \I$, i.e., $\K_\btheta^{t} + \sigma^2 \I = \LL_\btheta^{t} (\LL_\btheta^{t})^\top$. We note that each element of a unique solution $\z_{t}$ of the linear system $\LL_\btheta^{t} \z_{t} = \y_t$ follows
\begin{equation}
z_j = \frac{y_j - \mu_{j - 1}(\x_j; \btheta)}{\sqrt{\sigma^2 + \sigma^2_{j - 1}(\x_j; \btheta)}} \quad 1 \le j \le t. 
\end{equation}
Then, the term $\sum_{j = 1}^t \lambda_{j - 1} L_j^c(\btheta)$ can be vectorized as an inner product. Let $\vert . \vert^p$ denotes the element-wise absolute value raised to the power of $p$. We can express the sum as
\begin{align}
    \sum_{j = 1}^t \lambda_{j - 1} \, L_j^c(\btheta)  = \boldsymbol{\lambda}_t^\top \vert \mathrm{diag}(\vv_t) \, \z_t  \vert^p - \sum_{j=1}^t \lambda_{j - 1} \nonumber = \boldsymbol{\lambda}_t^\top \vert \mathrm{diag}(\vv_t) \, (\LL_\btheta^t)^{-1} \y_t  \vert^p - \sum_{j=1}^t \lambda_{j - 1},
\end{align}
where $\mathrm{diag}(\vv_t)$ is a diagonal matrix associated with a vector $\vv_t = [1 /\beta_1^{1/2}(\delta), \dots, 1 / \beta_t^{1/2}(\delta)]^\top, \, \delta \in (0, 1)$, and $\boldsymbol{\lambda}_t = [\lambda_1, \dots, \lambda_t]$. Therefore, we can define the objective $\sum_{j = 1}^t u_j^P(\btheta) - \inner{\bseta}{\btheta}$ as
\begin{equation}
\calL_t^P(\btheta) = \frac{\left( \log \det(\K_\btheta^t + \sigma^2 \I) - t \log(\sigma^2) \right) }{\log(1 + \sigma^{-2})} + \boldsymbol{\lambda}_t^\top \vert \mathrm{diag}(\vv_t) \, (\LL_\btheta^t)^{-1} \y_t  \vert^p - \sum_{j=1}^t \lambda_{j - 1} - \inner{\bseta}{\btheta}
\end{equation}
to complete our proof.

\subsection{Proof of \cref{cor:sharpness}}\label{proof:sharpness}

\corollarysharpness*

\textit{Proof.}
On the high-probability confidence event of~\cite[Lemma~5.2]{srinivas2009gaussian} (which holds under Assumption~\ref{assumption:rkhs} and the choice of $\beta_t(\delta)$ from~\cite[Theorem~6]{srinivas2009gaussian}), each instantaneous regret satisfies
\[
r_t \le 2\sqrt{\beta_t(\delta)}\,\sigma_{t-1}(\x_t;\btheta).
\]
Furthermore, from ~\cite[Lemma~5.4]{srinivas2009gaussian}, it holds that
\[
r_t^2 \le 4 \beta_t(\delta) \, \sigma^2 \, C \, \log(1 + \sigma^{-2} \, \sigma_{t-1}(\x_t;\btheta)),
\]
where $C = \sigma^{-2} / \log(1 + \sigma^{-2})$. Therefore, by the Cauchy--Schwarz inequality, it follows that
\[
R_T=\sum_{t=1}^T r_t
\le
\sqrt{T \sum_{t=1}^T r_t^2}
= \sqrt{\,T\sum_{t=1}^T \beta_T(\delta)\, \sigma^2 \, C \log (1 + \sigma^{-2} \, \sigma_{t-1}^2(\x_t;\btheta))\,},
\]
completing the proof.

% Note that the inequality $R_T \leq T \sum_{t = 1}^T r_t^2$ holds by the Cauchy-Schwarz inequality. Applying [Lemma 5.2] of \cite{srinivas2009gaussian}, we have that $R_T^2 \leq T \sum_{t = 1}^T 4 \beta_t \sigma^2_{t}(\x_{t + 1})$. Note that [Lemma 5.2] is valid on \cref{assumption:rkhs} through [Theorem 6] of \cite{srinivas2009gaussian}. Taking the square root of both sides completes the proof.

\subsection{Proof of \cref{corr:calibration}}\label{proof:calibration}

\corollarycalibration*

\textit{Proof.}
Fix $\delta\in(0,1)$ and let $\beta_t(\delta)$ be as in Corollary~\ref{cor:sharpness}.
For each round $t$, define the latent-function score and prediction set
\[
s_t^f(\x,z)
:=
\frac{|z-\mu_{t-1}(\x;\btheta)|}
{\sqrt{\beta_t(\delta)}\,\sigma_{t-1}(\x;\btheta)},
\qquad
C_t^f(\x)
:=
\{z\in\real:\ s_t^f(\x,z)\le 1\}.
\]
Let
\[
A_t^f=\indicator\!\big[f(\x_t)\in C_t^f(\x_t)\big],
\qquad
\hat p_T^f=\frac1T\sum_{t=1}^T A_t^f.
\]
We rely on the uniform confidence bound of GP-UCB \citep[Theorem~6]{srinivas2009gaussian}, which states that
\begin{equation}\label{eq:uniform-ci}
\probability\!\left(
\bigcap_{t\ge 1}\ \bigcap_{\x\in\calX}
\{f(\x)\in C_t^f(\x)\}
\right)
\ge 1-\delta.
\end{equation}
On the event in~\eqref{eq:uniform-ci}, in particular for each $t$ we have
$f(\x_t)\in C_t^f(\x_t)$, so $A_t^f=1$ for all $t$, hence
$\hat p_T^f=1\ge 1-\delta$.
Therefore,
\[
\probability\!\left(\hat p_T^f\ge 1-\delta\right)
\ge
\probability\!\left(
\bigcap_{t\ge 1}\ \bigcap_{\x\in\calX}
\{f(\x)\in C_t^f(\x)\}
\right)
\ge 1-\delta,
\]
which proves the claim.

\subsection{Proof of \cref{corr:extended-strong-duality}}

\begin{corollary}\label{corr:extended-strong-duality}
Let $L^\sharpness \in [0, 1]$ and $L^\calibration \in [-1, 1]$. Assume there exists a strictly feasible safe kernel hyperparameter $\btheta^\circ$ such that $L^c(\btheta^\circ) \leq - \rho$, for some $\rho > 0$. Then, it holds that
\begin{equation}
    \underset{\xi \in \Xi}{\min} \, \underset{\lambda \geq 0}{\max} \, \calL(\xi, \lambda) = \underset{\lambda \in \calS_\rho}{\max} \, \underset{\xi \in \Xi}{\min} \, \calL(\xi, \lambda) = \mathrm{OPT}
\end{equation}
where $\calS_{\rho} = [0, 1/\rho]$ is a restricted dual domain, and $\lambda) = \mathcal{L}(\xi, \lambda) = \expect_{\hat{\btheta} \sim \xi}[L^s(\hat{\btheta})] + \lambda \expect_{\hat{\btheta} \sim \xi}[L^c(\hat{\btheta})]$.
\end{corollary}

\textit{Proof:}

The proof is analogous to that of \cite[Theorem 3.3]{castiglioni2022unifying}. Because $\mathcal{L}(\xi, \lambda)$ is linear in both $\xi$ and $\lambda$, it is convex-concave. Furthermore, $\Xi$ and $\mathcal{S}_\rho$ are convex domains. We first show that restricting the dual domain does not change the max-min value, i.e., $\underset{\lambda \geq 0}{\max} \, \, \underset{\xi \in \Xi}{\min} \, \calL(\hat{\btheta}, \lambda) = \underset{\lambda \in \calS_\rho}{\max} \, \, \underset{\xi \in \Xi}{\min} \, \calL(\xi, \lambda)$. We note that for any $\lambda > 1 / \rho$, the primal RM can play the strictly safe hyperparameter $\xi^\circ$ such that
\begin{equation}
    \underset{\xi \in \Xi}{\min} \calL(\xi, \lambda) \leq \calL(\xi^\circ, \lambda) = L^s(\xi^\circ) + \lambda L^c(\xi^\circ)
\end{equation}
By the assumption, it holds that $L^s(\xi^\circ) \leq 1$ and $L^c(\xi^\circ) \leq - \rho$. Therefore, we find that
\begin{equation}
    \calL(\xi^\circ, \lambda) \leq 1 - \lambda \rho.
\end{equation}
Since we assume that $\lambda > 1 / \rho$, it follows that $1 - \lambda \rho < 0$. Therefore, it holds that
\begin{equation}
    \underset{\lambda \notin \calS_\rho}{\sup} \; \, \underset{\xi \in \Xi}{\min} \calL(\xi, \lambda) < 0.
\end{equation}
Furthermore, if the dual RM plays $\lambda = 0 \in \calS_\rho$, we guarantee that
\begin{equation}
\underset{\lambda \in \calS_\rho}{\max} \; \, \underset{\xi \in \Xi}{\min} \, \calL(\xi, \lambda) \geq \underset{\xi \in \Xi}{\min} \, \calL(\xi, 0) = \underset{\xi \in \Xi}{\min} \,\expect_{\hat{\btheta} \in \xi} [L^s(\hat{\btheta})] \geq 0.
\end{equation}
Since playing any $\lambda > 1 / \rho$ yields a value worse than $\lambda = 0$, we conclude that the supremum is achieved within the restricted domain, i.e.,
\begin{equation}\label{eq:feasible-dual-domain}
    \underset{\lambda \geq 0}{\max} \; \,  \underset{\xi \in \Xi}{\min} \, \calL(\xi, \lambda) =  \underset{\lambda \in \calS_\rho }{\max} \; \,  \underset{\xi \in \Xi}{\min} \, \calL(\xi, \lambda).
\end{equation}
By applying minimax theorem over the compact and convex sets $\Xi$ and $\mathcal{S}_\rho$, strong duality holds, and we obtain:
\begin{align}
    \mathrm{OPT} &= \underset{\xi \in \Xi}{\min} \, \underset{\lambda \geq 0}{\max} \, \calL(\xi, \lambda), \\
    & \geq \underset{\lambda \geq 0}{\max} \, \underset{\xi \in \Xi}{\min} \, \calL(\xi, \lambda), \\
    & =  \underset{\lambda \in \calS_\rho}{\max} \, \underset{\xi \in \Xi}{\min} \calL(\xi, \lambda), \\
    & =  \underset{\xi \in \Xi}{\min} \,\underset{\lambda \in \calS_\rho}{\max} \,  \calL(\xi, \lambda) = \mathrm{OPT}.
\end{align}

\subsection{Proof of \cref{lemma:lipschitz-property}}
\begin{lemma} \label{lemma:lipschitz-property}
    Given the hyperparameter $\btheta \in \bTheta$, we denote $\mu_{t - 1}(\x_t; \btheta)$ and $\sigma_{t - 1}(\x_t; \btheta)$ as the GP predictive mean and std on point $\x_t$ conditioned on $\dataset_{t - 1}$. For any $\hat{\btheta}, \btheta \in \bTheta$, define a metric $d(\hat{\btheta}, \btheta) = \sup_{\x, \hat{\x} \in \calX} \vert k_{\hat{\btheta}}(\x, \hat{\x}) - k_\btheta(\x, \hat{\x}) \vert$. For any $t \in [T]$, $\hat{\btheta}, \btheta \in \bTheta$, it holds that 
    \begin{itemize}
        \item $\vert \mu_{t - 1}(\x_t; \hat{\btheta}) - \mu_{t - 1}(\x_t; \btheta) \vert \leq L_{\mu_{t - 1}} d(\hat{\btheta}, \btheta)$, for some constant $L_{\mu_{t - 1}} > 0$,
        \item $\vert \beta_t^{1/2}(\delta) \, \sigma_{t - 1}(\x_t; \hat{\btheta}) - \beta_t^{1/2}(\delta) \, \sigma_{t - 1}(\x_t; \btheta) \vert \leq L_{\sigma_{t - 1}} d(\hat{\btheta}, \btheta)$, for some constant $L_{\sigma_{t - 1}} > 0$.
    \end{itemize}
\end{lemma}

\textit{Proof:}
Define $\tilde{\K}_{\hat{\btheta}}^{t - 1} = \K_{\hat{\btheta}}^{t - 1} + \sigma^{2} \I$, for any $\hat{\btheta} \in \bTheta$. By the definition of the metric $d$, we can generalize the bound to vector and matrix $2$-norms. Specifically, it holds that
\begin{itemize}
    \item $\vert k_{\hat{\btheta}}(\x_t, \x_t) - k_{\btheta}(\x_t, \x_t) \vert \leq d(\hat{\btheta}, \btheta)$,
    \item $\Vert \kk_{\btheta}(\x_t) \Vert_2 \leq \sqrt{t - 1}$, since for any $\x, \hat{\x} \in \calX$, and $\btheta \in \bTheta$, it holds that $k_\btheta(\x, \hat{\x}) \leq 1$,
    \item $\Vert \kk_{\hat{\btheta}}(\x_t) - \kk_\btheta(\x_t) \Vert_2 \leq \sqrt{t - 1} d(\hat{\btheta}, \btheta)$,
    \item $\Vert \K_{\hat{\btheta}}^{t - 1} - \K_\btheta^{t - 1} \Vert_2 \leq (t - 1) d(\hat{\btheta}, \btheta)$,
\end{itemize}
for any $\hat{\btheta}, \btheta \in \bTheta$. For the rest of the proof, we will also utilize the bound of the spectral norm of the inverse covariance matrix, i.e., $\Vert (\tilde{\K}_{\btheta}^{t - 1})^{-1} \Vert_2 \leq 1 / \sigma^2$.

\paragraph{Lipschitz continuity of the predictive mean}
By the definition of the GP predictive mean, we have
\begin{align}
    &\vert \mu_{t - 1}(\x_t; \hat{\btheta}) - \mu_{t - 1}(\x_t; \btheta) \vert = \vert \kk_{\hat{\btheta}}(\x_t)^\top (\tilde{\K}_{\hat{\btheta}}^{t - 1})^{-1} \y_{t - 1} - \kk_\btheta(\x_t)^\top (\tilde{\K}_{\btheta}^{t - 1})^{-1} \y_{t - 1}  \vert, \\
    &= \vert (\kk_{\hat{\btheta}}(\x_t) - \kk_{\btheta}(\x_t))^\top (\tilde{\K}_{\hat{\btheta}}^{t - 1})^{-1} \y_{t - 1} + \kk^\top_\btheta(\x_t) ((\tilde{\K}_{\hat{\btheta}}^{t - 1})^{-1} - (\tilde{\K}_{\btheta}^{t - 1})^{-1}) \y_{t - 1} \vert.
\end{align}
By applying the triangle inequality and the Cauchy-Schwarz inequality, we obtain
\begin{align}
    &\vert \mu_{t - 1}(\x_t; \hat{\btheta}) - \mu_{t - 1}(\x_t; \btheta) \vert \nonumber \\
    &\leq \Vert \kk_{\hat{\btheta}}(\x_t) - \kk_{\btheta}(\x_t) \Vert_2 \Vert \tilde{\K}_{\hat{\btheta}}^{t - 1} \Vert_2 \Vert \y_{t - 1} \Vert_2 + \Vert \kk_{\btheta} \Vert_2 \Vert (\tilde{\K}_{\hat{\btheta}}^{t - 1})^{-1} - (\tilde{\K}_{\btheta}^{t - 1})^{-1} \Vert_2 \Vert \y_{t - 1} \Vert_2.
\end{align}
Next, we apply the matrix inverse difference identity $A^{-1} - B^{-1} = A^{-1} (B - A) B^{-1}$, and the Cauchy-Schwarz inequality on $(\tilde{\K}_{\hat{\btheta}}^{t - 1})^{-1} - (\tilde{\K}_{\btheta}^{t - 1})^{-1}$ to obtain
\begin{equation}
\Vert (\tilde{\K}_{\hat{\btheta}}^{t - 1})^{-1} - (\tilde{\K}_{\btheta}^{t - 1})^{-1} \Vert_2 \leq \Vert (\tilde{\K}_{\hat{\btheta}}^{t - 1})^{-1} \Vert_2 \Vert \tilde{\K}_\btheta^{t - 1} - \tilde{\K}_{\hat{\btheta}}^{t - 1} \Vert_2 \Vert (\tilde{\K}_{\btheta}^{t - 1})^{-1} \Vert_2. 
\end{equation}
By applying the obtained scalar, vector, and matrix norms bounds, it holds that
\begin{align}
    \vert \mu_{t - 1}(\x_t; \hat{\btheta}) - \mu_{t - 1}(\x_t; \btheta) \vert \leq \, \Vert \y_{t - 1} \Vert_2 \left(\frac{\sqrt{t - 1} }{\sigma^2} \, + \frac{ \sqrt{t - 1} \, (t - 1) }{\sigma^4} \right) d(\hat{\btheta}, \btheta).
\end{align}
Choosing $L_{\mu_{t - 1}} = \Vert \y_{t - 1} \Vert_2 \left(\frac{\sqrt{t - 1} }{\sigma^2} \, + \frac{ \sqrt{t - 1} \, (t - 1) }{\sigma^4} \right)$ completes the first part of the proof.

\paragraph{Lipschitz continuity of the predictive standard-deviation}
We start by showing that the predictive GP variance is Lipschitz continuous w.r.t. the metric $d$. Similar to the predictive mean, we apply the triangle inequality to obtain
\begin{align}
&\vert \sigma^2_{t - 1}(\x_t; \hat{\btheta}) - \sigma^2_{t - 1}(\x_t; \btheta) \vert \nonumber \\ 
&\leq \vert k_{\hat{\btheta}}(\x_t, \x_t) -  k_\btheta(\x_t, \x_t) \vert + \vert (\kk_{\hat{\btheta}}(\x_t) - \kk_\btheta(\x_t))^\top (\tilde{\K}_{\hat{\btheta}})^{-1} \kk_{\hat{\btheta}}(\x_t) \vert  \nonumber \\
& \quad  +  \vert \kk_{\hat{\btheta}}^\top(\x_t) ( (\tilde{\K}_{\hat{\btheta}}^{t - 1})^{-1} - (\tilde{\K}_\btheta^{t - 1})^{-1} ) \kk_\btheta(\x_t) \vert + \vert \kk_{\btheta}^\top(\x_t) (\tilde{\K}_{\btheta}^{t - 1})^{-1} (\kk_{\hat{\btheta}}(\x_t) - \kk_\btheta(\x_t) ) \vert .
\end{align}
Then, by applying the Cauchy-Schwarz inequality, the matrix inverse difference identity, and the norm bounds, it holds that
\begin{align}
\vert \sigma_{t - 1}^2(\x_t; \hat{\btheta}) - \sigma_{t - 1}^2(\x_t; \btheta) \vert &\leq d(\hat{\btheta}, \btheta) + \frac{2 (t - 1) \, d(\hat{\btheta}, \btheta)}{\sigma^2} + \frac{(t - 1)^2  d(\hat{\btheta}, \btheta)}{\sigma^4}, \\
&= L_{\sigma^2_{t - 1}} d(\hat{\btheta}, \btheta),
\end{align}
where $L_{\sigma_{t - 1}^2} = \left( 1  + \frac{t - 1}{\sigma^2} \right)^2$. Since it holds that $\vert \sqrt{C} - \sqrt{D} \vert \leq \sqrt{\vert C - D \vert}$ for $C, D \geq 0$, we find that.
\begin{align}
    \vert \sigma_{t - 1}(\x_t; \hat{\btheta}) - \sigma_{t - 1}(\x_t; \btheta) \vert &\le \sqrt{\vert \sigma^2_{t - 1}(\x_t; \hat{\btheta}) - \sigma^2_{t - 1}(\x_t; \btheta) \vert}, \\
    &\le \sqrt{\left( 1  + \frac{t - 1}{\sigma^2} \right)^2 d(\hat{\btheta}, \btheta)}
\end{align}

Thus, we set $L_{\sigma_{t - 1}} = \sqrt{L_{\sigma^2_{t - 1}}}$ to complete our proof.

\subsection{Proof of \cref{lemma:utility-lispchitz-property}}
\begin{lemma}\label{lemma:utility-lispchitz-property}
Let $L_s^t(\btheta)$ and $L_t^c(\btheta)$ be the sharpness loss and calibration constraint functions defined as in \cref{eq:sharpness-reward} and \cref{eq:calibration-constraint}, respectively. Define $u_t^P(\btheta) = L_t^s(\btheta) + \lambda_{t - 1} L_t^c(\btheta)$ as the primal utility function. Furthermore, assume that the kernel $k_\btheta$ is Lipschitz continuous, e.g., Matérn kernel with $\nu = \frac{5}{2}$. On the high-probability observation bound event $E_\mathrm{Y}$ (established in \cref{lemma:slater-condition}), $u_t^P$ is a Lipschitz continuous function w.r.t $\ell_1$ norm, i.e., $\vert u_t^P(\hat{\btheta}) - u_t^P(\btheta)\vert \leq K \Vert \hat{\btheta} - \btheta \Vert_1, \; K > 0$, for any $\hat{\btheta}, \btheta \in \bTheta$.
\end{lemma}

\textit{Proof:}
The strategy is to show that both $L_t^s$ and $L_t^c$ are Lipschitz continuous functions w.r.t. $\ell_1$ norm. Then, since Lipschitz continuity is preserved under linear combination, $u_t^P$ will remain Lipschitz continuous. Assuming that the kernel is Lipschitz continuous and applying the results from \cref{lemma:lipschitz-property}, we have that
\begin{align}
    &\vert \mu_{t - 1}(\x_t; \hat{\btheta}) - \mu_{t - 1}(\x_t; \btheta) \vert \leq L_{\mu_{t - 1}} d(\hat{\btheta}, \btheta) \leq L_\mu \Vert \hat{\btheta} - \btheta \Vert_1, \\
    & \vert \sigma^2_{t -  1}(\x_t; \hat{\btheta}) - \sigma^2_{t - 1}(\x_t; \btheta) \vert \leq L_{\sigma^2_{t - 1}} d(\hat{\btheta}, \btheta) \leq L_{\sigma^2} \Vert \hat{\btheta} - \btheta \Vert_1,
\end{align}
for any $\hat{\btheta}, \btheta \in \bTheta$.

\paragraph{Lipschitz continuity of $L_s^t(\btheta)$} We first show that the scalar function $h(z) = \frac{1}{\log(1 +  \sigma^{-2})} \log(1 + \sigma^{-2} z), z \geq 0$ is Lipschitz continuous. By taking the first derivative of $h$, we find that
\begin{equation}
    h^\prime(z) = \frac{1}{\log(1 + \sigma^{-2})} \frac{\sigma^{-2}}{1 + \sigma^{-2} z}.
\end{equation}
Since $\sigma_{t - 1}(\x_t; \hat{\btheta}) \geq 0$ for any $\hat{\btheta} \in \bTheta$ and $t \in [T]$, it follows that $\max_{z \geq 0} \vert h'(z) \vert = \frac{\sigma^{-2}}{\log(1 + \sigma^{-2})}$. By the mean value theorem, we obtain
\begin{align}
\vert L_t^s(\hat{\btheta}) - L_t^s(\btheta) \vert
&\le
\frac{\sigma^{-2}}{\log(1 + \sigma^{-2})}
\vert \sigma^2_{t - 1}(\x_t; \hat{\btheta}) - \sigma^2_{t - 1}(\x_t; \btheta) \vert, \\
&\le
\frac{L_{\sigma^2} \sigma^{-2}}{\log(1 + \sigma^{-2})}
\Vert \hat{\btheta} - \btheta \Vert_1.
\end{align}
\textbf{Lipschitz continuity of $L_t^c(\btheta)$} \\

\textbf{Case 1: $p=1$.} Define $A(\btheta) = \vert y_t - \mu_{t - 1}(\x_t; \hat{\btheta}) \vert$, $B(\hat{\btheta}) = \frac{1}{\sqrt{\beta_t(\delta)}\sqrt{\sigma^{2} + \sigma^2_{t - 1}(\x_t; \hat{\btheta})}}$, for any $\hat{\btheta} \in \bTheta$. Since the prediction $y_t$ is fixed at round $t$, $A(\hat{\btheta})$ is bounded by $E_t = \sup_{\bar{\btheta} \in \bTheta} \vert y_t - \mu_{t - 1}(\x_t; \bar{\btheta}) \vert$. By the definition of $L_t^c$ with $p=1$, and applying the triangle inequality, it follows that
\begin{equation}
\vert A(\hat{\btheta}) B(\hat{\btheta}) - A(\btheta) B(\btheta) \vert 
\leq B(\hat{\btheta}) \vert A(\hat{\btheta}) - A(\btheta) \vert + A(\btheta) \vert B(\hat{\btheta}) - B(\btheta) \vert.
\end{equation}
The goal is to bound each term in the inequality above. We notice that $B(\btheta) \leq \frac{1}{ \sqrt{\beta_t(\delta)} \sigma}$ for all $\btheta \in \bTheta$. Let $g(z) = \frac{1}{\sqrt{\beta_t(\delta)} \sqrt{\sigma^2 + z}}$, with the corresponding first derivative $g^\prime(z) = -\frac{1}{2 \sqrt{\beta_t(\delta)}} (\sigma^2 + z)^{-3/2}$. Furthermore, we define the maximum absolute first derivative $\max_{z \geq 0} \vert g^\prime(z) \vert = \frac{1}{2 \sqrt{\beta_t(\delta)} \sigma^3}$. By the mean value theorem, we obtain
\begin{align}
    \vert B(\hat{\btheta}) - B(\btheta) \vert
    &\le \frac{1}{2 \sqrt{\beta_t(\delta)} \sigma^3}
    \vert \sigma^2_{t - 1}(\x_t; \hat{\btheta}) - \sigma^2_{t - 1}(\x_t; \btheta) \vert, \\
    &\le \frac{L_{\sigma^2}}{2 \sigma^3 \sqrt{\beta_t(\delta)}}
    \Vert \hat{\btheta} - \btheta \Vert_1,
\end{align}
for any $\hat{\btheta}, \btheta \in \btheta$. Next, we apply the reverse triangle inequality to obtain
\begin{align}
\vert A(\hat{\btheta}) - A(\btheta) \vert &= \vert \vert y_t - \mu_{t - 1}(\x_t; \hat{\btheta}) \vert - \vert y_t - \mu_{t - 1}(\x_t; \btheta) \vert \vert, \\
& \leq \vert (y_t - \mu_{t - 1}(\x_t; \hat{\btheta})) - (y_t - \mu_{t - 1}(\x_t; \btheta)) \vert, \\ 
&= \vert \mu_{t - 1}(\x_t; \hat{\btheta}) - \mu_{t - 1}(\x_t; \btheta) \vert, \\
& \leq L_\mu \Vert \hat{\btheta} - \btheta \Vert_1.
\end{align}
Combining all inequalities, we find that
\begin{align}
    \vert A(\hat{\btheta}) B(\hat{\btheta}) - A(\btheta) B(\btheta) \vert   
     &\leq  B(\hat{\btheta}) \vert A(\hat{\btheta}) - A(\btheta) \vert + A(\btheta) \vert B(\hat{\btheta}) - B(\btheta) \vert,  \\
    &\leq \left( \frac{L_\mu}{\sqrt{\beta_t(\delta)} \sigma}  + \frac{E_t L_{\sigma^2}}{ 2 \sigma^3 \sqrt{\beta_t(\delta)}} \right) \Vert \hat{\btheta} - \btheta \Vert_1. \label{eq:Lipschitz-Lct-1}
\end{align}

\textbf{Case 2: $p=2$.} For $L_t^c$ with $p=2$, we need to show that the squared product $(A(\hat{\btheta}) \, B(\hat{\btheta}))^2$ is Lipschitz continuous:
\begin{align}
    \vert (A(\hat{\btheta}) \, B(\hat{\btheta}))^2 - (A(\btheta) \, B(\btheta))^2 \vert &= \vert A(\hat{\btheta}) B(\hat{\btheta}) +  A(\btheta) B(\btheta)  \vert \, \vert  A(\hat{\btheta}) B(\hat{\btheta}) -  A(\btheta) B(\btheta) \vert, \\
    &\leq  2 \max_{\bar{\btheta} \in \bTheta} \vert  A(\bar{\btheta}) B(\bar{\btheta}) \vert \, C_{AB} \Vert \hat{\btheta} - \btheta \Vert_1, \\
    & \leq \frac{2 E_t \, C_{AB}}{\sqrt{\beta_t(\delta)} \sigma} \Vert \hat{\btheta} - \btheta \Vert_1, 
\end{align}
where $C_{AB} = \left( \frac{L_\mu}{\sqrt{\beta_t(\delta)} \sigma}  + \frac{E_t L_{\sigma^2}}{ 2 \sigma^3 \sqrt{\beta_t(\delta)}} \right)$. The first inequality follows from the triangle inequality and \cref{eq:Lipschitz-Lct-1}, while the last one stems from the fact that $A$ and $B$ functions are bounded.

\paragraph{Lipschitz continuity of $u_t^P$} Finally, we apply the triangle inequality and apply the Lipschitz continuity results of $L_t^s$ and $L_t^c$:
\begin{align}
    \vert u_t^P(\hat{\btheta}) - u_t^P(\btheta) \vert &\leq \vert L_t^s(\hat{\btheta}) - L_t^s(\btheta) \vert + \lambda_{t - 1} \vert L_t^c(\hat{\btheta}) - L_t^c(\btheta) \vert \\
    &\leq K \Vert \hat{\btheta} - \btheta \Vert_1,
\end{align}
where $K =  \frac{L_{\sigma^2} \sigma^{-2}}{\log(1 + \sigma^{-2})} + \lambda_{t - 1} C_{AB}$ for $L_t^c$ with $p=1$, and $K= \frac{L_{\sigma^2} \sigma^{-2}}{\log(1 + \sigma^{-2})} + \lambda_{t - 1} \frac{2 E_t \, C_{AB}}{ \sqrt{\beta_t(\delta)} \sigma }$ for $p=2$, completing our proof.

\subsection{Proof of \cref{lemma:slater-condition}}\label{app:slater-condition}

\lemmaslater*

\textit{Proof:}

\textbf{Uniform Confidence Interval Bound}

We first formally state the confidence interval bound result from \cite{chowdhury2017kernelized}:

\begin{theorem}\label{theorem:chowdury-CI} \cite{chowdhury2017kernelized}
Let $\btheta$ be the hyperparameter of the true kernel $k$. Define $\beta^{1/2}_t(\delta) = B + R \sqrt{2 (\gamma_{t - 1}(\btheta) + 1 + \log(6 / \delta))}$, for some $R \geq 0$, and $\gamma_{t - 1}(\btheta)$ denote the maximum information gain of a kernel function $k_\btheta$ after $t - 1$ rounds. Furthermore, we set $\sigma^{2} = 1 + 2/T$, where $T > 1$ is the total number of rounds. Under \cref{assumption:rkhs}, with a probability of at least $1 -  \delta / 6, \, \delta \in (0, 1)$, the event $\vert \mu_{t - 1}(\x_t; \btheta) - f(\x) \vert \leq \beta_t^{1/2}(\delta) \sigma_{t - 1}(\x_t; \btheta)$ holds for all $\x \in \calX$ and $t \geq 1$.
\end{theorem}

For any $\varepsilon > 0$, we construct a covering number $\normal_{\varepsilon}(\bTheta, d)$ of the $\varepsilon-$covering set $(\btheta_i^\varepsilon)_{i=1}^{\normal_{\varepsilon}(\bTheta, d)}$. Here, we define $d(\btheta, \hat{\btheta}) = \sup_{\x, \x^\prime} \vert k_{\btheta}(\x, \x^\prime) - k_{\hat{\btheta}}(\x, \x^\prime) \vert, \forall \btheta, \hat{\btheta} \in \bTheta$. Pick $\hat{\delta} =  \delta\normal_{\varepsilon}(\bTheta, d) / 6$ and set $\beta^{1/2}_t(\delta) = B + R \sqrt{2  \Gamma_{t - 1} + 1 + \log(6 \normal_{\varepsilon}(\bTheta, d) / \delta))}$. By applying the union bound, we have, with a probability of at least $1 - \delta / 6, \delta \in (0, 1)$, $\forall \x \in \calX$, $\forall \btheta_i^\varepsilon$
\begin{align}\label{eq:ci-covering-number}
\vert f(\x) - \mu_{t - 1}(\x; \btheta_i^\varepsilon) \vert \leq \beta_t^{1/2}(\delta)  \, \sigma_{t - 1}(\x; \btheta_i^\varepsilon).
\end{align}
By the definition of $\varepsilon-$covering, there exists $i \in [\normal_{\varepsilon}(\bTheta, d)]$ such that
$d(\hat{\btheta}_{t - 1}, \btheta_i^\varepsilon) \leq \varepsilon$. Therefore, for any $t \in [T]$, $\x \in \calX$, and $\hat{\btheta} \in \bTheta$, with a probability of at least $1 -  \delta / 6$, it holds that
\begin{align}
\vert f(\x) - \mu_{t - 1}(\x; \hat{\btheta}) \vert &\leq \vert f(\x) - \mu_{t - 1}(\x; \btheta_i^\varepsilon) \vert + \vert \mu_{t - 1}(\x; \btheta_i^\varepsilon) - \mu_{t - 1}(\x; \hat{\btheta}) \vert, \\
& \leq \beta_t^{1/2}(\delta) \sigma_{t - 1}(\x; \btheta_i^\varepsilon) + L_{\mu_{t - 1}} d(\hat{\btheta}, \btheta_i^\varepsilon), \\
& \leq \beta_t^{1/2}(\delta) \left( \sigma_{t - 1}(\x; \hat{\btheta}) + \vert \sigma_{t - 1}(\x; \btheta_i^\varepsilon) - \sigma_{t - 1}(\x; \hat{\btheta}) \vert \right) + L_{\mu_{t - 1}} \varepsilon, \\
& \leq \beta_t^{1/2}(\delta) \sigma_{t - 1}(\x; \hat{\btheta}) + \beta_t^{1/2}(\delta) \sqrt{L_{\sigma_{t - 1}^2} \varepsilon} + L_{\mu_{t - 1}} \varepsilon.
\end{align}

The inequalities follow from the triangle inequality and the union bound, together with the Lipschitz continuity of the predictive posterior mean $\mu_{t - 1}(\x; \hat{\btheta})$ and the confidence width $\sigma_{t - 1}(\x; \hat{\btheta})$ (\cref{lemma:lipschitz-property}).

\textbf{Observation Vector Norm Bound}

The term $L_{\mu_{t - 1}}$ scales with the vector norm $\|\mathbf{y}_{t-1}\|_2$. Since $y_i = f(x_i) + \epsilon_i$, and the noise $\epsilon_i$ is an unbounded $R$-sub-Gaussian random variable, $C_L = \beta_t^{1/2}(\delta) \sqrt{L_{\sigma_{t - 1}^2} \varepsilon_T} + L_{\mu_{t - 1}} \varepsilon_T$ is also an unbounded random variable. Define $E_\mathrm{Y}$ as the event where the observation noise is uniformly bounded. By the sub-Gaussian tail bound and a union bound over $T$ steps, it follows that
\begin{equation*}
    |\epsilon_t| \le R \sqrt{2 \log(12T/\delta)},
\end{equation*}
for all $t \in [T]$ with probability at least $1 - \delta/6$. Assuming $E_\mathrm{Y}$ and using the fact that $|f(x)| \le B$, the observation norm is deterministically bounded by 
\begin{equation*}
\|\mathbf{y}_{t-1}\|_2 \le \sqrt{T} (B + R \sqrt{2 \log(12T/\delta)}).
\end{equation*}
Since $L_{\mu_{t - 1}}$ depends on the vector norm $\|\mathbf{y}_{t-1}\|_2$, the random variable $L_{\mu_{t - 1}}$ is deterministically bounded by a constant $\bar{L}_{\mu} = \mathcal{O}(T^2 \sqrt{\log T})$, assuming the event $E_Y$ holds. Therefore, we conclude that $C_L$ can be bounded by  $\bar{C}_L = \beta_t^{1/2}(\delta) \sqrt{L_{\sigma_{t - 1}^2} \varepsilon_T} + \bar{L}_\mu \varepsilon_T$.

\textbf{Feasibility Condition}

Let $(\Omega, \calF, \probability)$ be the underlying probability space. Let the filtration $\{\calF_t\}_{t = 0}^\infty$ be a $\sigma-$algebra generated by the history. We update the filtration through $\calF_t = \sigma(\x_1, \epsilon_1, \dots, \x_t, \epsilon_t)$. Furthermore, define the covering error $\varepsilon_T = T^{-4}$. The goal is to show that for any sequentially selected $\x_t$, the conditional expectation of $L^c_t$ is strictly negative. Let $\expect_{t}[.] = \expect[. \vert \calF_{t - 1}]$ denote the conditional expectation. Recall that $y_t = f(\x_t) + \epsilon_t$, and let $\Delta_t = f(\x_t) - \mu_{t - 1}(\x_t; \hat{\btheta})$, for any $\hat{\btheta} \in \bTheta$.

\textbf{Case 1: $p=1$.}
When considering $L_t^c$ with $p=1$, we have that
\begin{equation}
    \expect_t[L_t^c(\hat{\btheta})] = \frac{\expect_t[\vert \Delta_t + \epsilon_t \vert]}{\beta_t^{1/2}(\delta) \sqrt{\sigma^2_{t - 1}(\x_t; \hat{\btheta}) + \sigma^2 }} - 1.
\end{equation}

By applying Cauchy-Schwarz inequality on $\vert \Delta_t + \epsilon_t \vert$, we find that:
\begin{align}
    \expect_t[\vert \Delta_t + \epsilon_t \vert] &\leq \sqrt{\expect_t[(\Delta_t + \epsilon_t)^2]} = \sqrt{\Delta_t^2 + 2 \Delta_t , \expect_t[\epsilon_t] + \expect_t[\epsilon^2_t]} = \sqrt{ \Delta_t^2 + R^2},
\end{align}
where the last equality holds due to $\expect_t[\epsilon_t] = 0$, and $\expect_t[\epsilon_t^2] = R^2$. Substituting the above results back into the expectation provides us
\begin{equation} \label{eq:expected_Ltc}
\expect_t[L_t^c(\hat{\btheta})] \leq \frac{\sqrt{\Delta_t^2 + R^2}}{\beta_t^{1/2}(\delta) \sqrt{\sigma^2_{t - 1}(\x_t; \hat{\btheta}) + \sigma^2}} - 1.
\end{equation}
Let $E_\mathrm{CI}$ be the event where the uniform confidence interval bound holds for all $t \in [T]$, and $\hat{\btheta} \in \bTheta$. Assume that the joint event $E_\mathrm{CI} \cap E_\mathrm{Y}$ holds. On this joint event, it deterministically holds that $\vert \Delta_t \vert \leq \beta_t^{1/2}(\delta) \, \sigma_{t - 1}(\x_t; \hat{\btheta}) + \bar{C}_L$, where $\bar{C}_L = \beta_t^{1/2}(\delta) \sqrt{L_{\sigma_{t - 1}^2} \varepsilon_T} + \bar{L}_\mu \varepsilon_T$. Let $\tau$ be the stopping time at which the confidence interval bound fails for the first time, i.e., $ \vert \Delta_\tau \vert > \beta_\tau^{1/2}(\delta) \, \sigma_{\tau - 1}(\x_\tau; \hat{\btheta})$. If the confidence interval bound never fails, we set $\tau = T + 1$. Therefore, for any $t < \tau$, it follows that $\Delta_t^2 \leq \beta_t(\delta) \, \sigma^2_{t - 1}(\x_t; \hat{\btheta}) + 2 \beta_t^{1/2}(\delta) \, \sigma_{t - 1}(\x_t; \hat{\btheta}) \bar{C}_L  + \bar{C}^2_L$. By choosing $\varepsilon_T = T^{-4}$, the error bound decays as $\mathcal{O}(T^{-1} \log T)$. For a sufficiently large $T$, $2 \beta_t^{1/2}(\delta) \, \sigma_{t - 1}(\x_t; \hat{\btheta}) \bar{C}_L \varepsilon_T + \bar{C}^2_L \varepsilon_T^2 \leq R^2 / \beta_t(\delta)$, we find that $\Delta_t^2 + R^2 \leq \beta_t(\delta) \, \sigma_{t - 1}^2(\x_t; \hat{\btheta}) + R^2(1 + 1/ \beta_t(\delta))$. Recall that $R^2 > 0$ be the fixed sub-Gaussian variance proxy of the true environment. In our algorithm, we specify the GP surrogate noise parameter to match this true noise upper bound, setting $\sigma^2 = R^2$. Substituting this into the bound yields: Assuming that $\beta_1 \geq 2$, it follows that
\begin{align}
\expect_t[L_t^c(\hat{\btheta})] &\leq \sqrt{\frac{\beta_t(\delta) 1 + R^2(1 + 1/\beta_t(\delta))}{\beta_t(\delta) (1 + R^2)}} - 1, \\
&\leq \sqrt{\frac{2 + (1 + 1/2)}{4}} - 1 \leq \sqrt{3.5/4} - 1 \approx -0.065.
\end{align}
Let $\btheta^\ast = \min_{\bar{\btheta} \in \bTheta} \expect_t[L_t^c(\bar{\btheta})]$. For the stop process, it holds that $\expect_t[L_t^c(\btheta^\ast)] \leq - \rho$, where $\rho = 0.065 > 0$.

Next, we aim to obtain the proxy variance of the sub-Gaussian martingale difference sequence $M_t = L_t^c(\hat{\btheta}) - \expect_t[L_t^c(\hat{\btheta}]$. We construct the stopped sequence $\tilde{M}_t = M_t \, \mathbf{1}\{ t < \tau \}$. Since $\tau$ is fully determined by $\calF_{t - 1}$ and $\x_t$, $\tilde{M}_t$ remains a valid martingale difference sequence. Next, we reformulate $L_t^c$ as a function of the noise $\epsilon_t$ conditioned on the filtration $\calF_{t - 1}$:
\begin{equation}
    h_t(\epsilon_t; \calF_{t - 1}) = \frac{\vert \Delta_t + \epsilon_t \vert}{\beta^{1/2}_t(\delta) \sqrt{\sigma_{t - 1}^2(\x_t; \hat{\btheta}) + R^2}} -1.
\end{equation}
Note that $\probability(\vert h_t(\epsilon_t) \vert > m) \leq \probability(\vert h_t(\epsilon_t) - h_t(0) \vert > m -  \vert h_t(0) \vert)$. By the Lipschitz continuity property, we find that $\vert h_t(\epsilon_t) - h_t(0) \vert \leq L_t \vert \epsilon_t \vert$, where $L_t  = \frac{1}{\beta^{1/2}_t(\delta) \sqrt{\sigma_{t - 1}^2(\x_t; \hat{\btheta}) + R^2}}$. Since $\tilde{M}_t$ clamps the sequence to zero if $\Delta_t$ explodes, the active sequence strictly inherits the sub-Gaussian probability tail bound \citep{vershynin2020high}:
\begin{equation}
    \probability(\vert h_t(\epsilon_t) \vert > m) \leq \probability(L_t \vert \epsilon_t \vert > m - \vert h_t(0) \vert ) \leq  2 \exp(-c m^2 / L_t^2 R^2),
\end{equation}
where $c > 0$ is a constant. Here, $\tilde{M}_t$ has a decay rate of $\Vert h_t(\epsilon_t) \Vert^2_{\psi_2} = L_t^2 R^2$. Since $\expect_t[\tilde{M}_t] = 0$, the moment generating function satisfies
\begin{align}
\expect[\exp( \eta \, \tilde{M}_t )] \leq \exp(C \eta^2 L_t^2 R^2),
\end{align}
for some constant $C > 0$. We note that $L_t \leq 1 / R$ for all $t \in [T]$. Therefore, we conclude $\tilde{M}_t$ has a proxy variance $v^2 = C$. Next, we apply the Azuma-Hoeffding concentration inequality to obtain
\begin{equation}
\probability\left(\sum_{t = 1}^T \tilde{M}_t \geq s \right) \leq \exp\left(- \frac{s^2}{2 T v^2} \right) = \exp\left( - \frac{s^2}{2TC} \right).
\end{equation}
Setting the failure probability $\exp\left( -\frac{s^2}{2TC} \right) = \frac{\delta}{6}$, we have that $s = \sqrt{2 T C \log(6/\delta)}$. Furthermore Let $\bar{M} = \hat{M} / T$, where $\hat{M} = \sum_{t = 1}^T \tilde{M}_t$. Then $\probability(\bar{M} \geq \bar{s}) = \probability(\hat{M} \geq T \bar{s})$. We then set $\bar{s} = \sqrt{2 C \log(6/\delta) / T}$. Therefore, with a probability of at least $1 -  \delta / 6$ the following bound holds:
\begin{equation}
    \frac{1}{T} \sum_{t = 1}^T \tilde{M}_t \leq \sqrt{2 C \log (6/\delta) / T}.
\end{equation}
Let $E_\mathrm{M}$ denote the event where the average stopped martingale concentrates. If the event $E_\mathrm{CI}$ holds, the stopping time never triggered, i.e., $\tau = T + 1$. Consequently, we have that $M_t = \tilde{M}_t$, and $\expect_t[L_t^c(\btheta^\ast)] \leq -0.065$ holds universally. When $E_\mathrm{CI}$ and $E_\mathrm{M}$ hold, it follows that
\begin{align}
\frac{1}{T} \sum_{t = 1}^T L_t^c(\btheta^\ast) = \frac{1}{T} \sum_{t = 1}^T \expect_t[L_t^c(\btheta^\ast)] + \frac{1}{T} \sum_{t = 1}^T M_t \
&\leq -0.065 + \sqrt{\frac{2 C \log (6/\delta)}{T}}.
\end{align}
We require $0.065 - \sqrt{2 C \log (6/\delta) / T} \geq T^{-1/4}$ to satisfy the Slater's condition. Since $\rho = 0.065 > 0$, $T^{-1/4}$ and $\sqrt{2 \log (6/\delta) / T}$ vanish as $T \rightarrow \infty$, the above inequality holds.

\textbf{Case 2: $p=2$.} Define the constraint $L_t^c$ as
\begin{equation}
    L_t^c(\hat{\btheta}) = \frac{(y_t - \mu_{t - 1}(\x_t; \hat{\btheta}))^2}{\beta_t(\delta) (\sigma^2_{t - 1}(\x_t; \hat{\btheta}) + \sigma^2)} - 1.
\end{equation}
We expand the numerator and obtain $(y_ t - \mu_{t - 1}(\x_t; \hat{\btheta}_t))^2 = (\Delta_t + \epsilon_t)^2 = \Delta_t^2 + 2 \Delta_t \epsilon_t + \epsilon_t^2$. Note that $\expect_t[\epsilon_t] = 0$, and $\expect_t[\epsilon_t^2] = R^2$. Assuming $\sigma^2 = R^2$, it follows that
\begin{equation}
\expect_t[L_t^c(\hat{\btheta}_{t - 1})] = \frac{\Delta_t^2 + \sigma^2}{\beta_t(\delta) (\sigma^2_{t - 1}(\x_t; \hat{\btheta}_{t - 1}) + \sigma^2)} - 1.
\end{equation}
Let $E_\mathrm{CI}$ be the event where the uniform confidence interval bound holds for all $t \in [T]$, and $\hat{\btheta} \in \bTheta$. We assume the joint event $E_\mathrm{CI} \cap E_\mathrm{Y}$ holds, giving the bound $\bar{C}_L = \beta_t^{1/2}(\delta) \sqrt{L_{\sigma_{t - 1}^2} \varepsilon_T} + \bar{L}_\mu \varepsilon_T$. On this joint event, it deterministically holds that $\vert \Delta_t \vert \leq \beta_t^{1/2}(\delta) \, \sigma_{t - 1}(\x_t; \hat{\btheta}) + \bar{C}_L$. Let $\tau$ be the stopping time at which the confidence interval bound fails for the first time. If the confidence interval bound never fails, we set $\tau = T + 1$. Therefore, for any $t < \tau$, it follows that $\Delta_t^2 \leq \beta_t(\delta) \, \sigma^2_{t - 1}(\x_t; \hat{\btheta}) + 2 \beta_t^{1/2}(\delta) \, \sigma_{t - 1}(\x_t; \hat{\btheta}) \bar{C}_L + \bar{C}_L^2$. Substituting this bound into the conditional expectation for $t \leq \tau$ provides
\begin{align}
\expect_t[L_t^c(\hat{\btheta})] &\le \frac{\beta_t(\delta) \sigma_{t - 1}^2(\x_t; \hat{\btheta}) + \sigma^2 + 2 \beta_t^{1/2}(\delta) \sigma_{t - 1}(\x_t, \hat{\btheta}) \bar{C}_L + \bar{C}_L^2}{\beta_t(\delta) \left( \sigma^2{t - 1}(\x_t; \hat{\btheta}) + \sigma^2 \right)} - 1, \\
&= \frac{\sigma^2(1 - \beta_t(\delta)) + 2 \beta_t^{1/2}(\delta) \sigma{t - 1}(\x_t; \hat{\btheta}) \bar{C}_L + \bar{C}_L^2}{\beta_t(\delta) \left( \sigma^2{t - 1}(\x_t; \hat{\btheta}) + \sigma^2 \right)}.
\end{align}
Assuming that $\beta_t(\delta) \geq 2$, it holds that $\sigma^2(1 - \beta_t(\delta)) \leq -\sigma^2$. By choosing $\varepsilon_T = T^{-4}$, the error bound decays as $\mathcal{O}(T^{-1} \log T)$. For a sufficiently large $T$, it holds that $2 \beta_t^{1/2}(\delta) \sigma_{t - 1}(\x_t; \hat{\btheta}) \bar{C}_L + \bar{C}_L^2 \leq \sigma^2 / 2$ holds deterministically. Therefore, the worst-case numerator is bounded by $-\sigma^2/2$. Let $\btheta^\ast = \arg\min_{\bar{\btheta} \in \bTheta} \expect_t[L_t^c(\bar{\btheta})]$. For the stopped process, it holds that 
\begin{equation}
\expect_t[L_t^c(\btheta^\ast)] \le \frac{-\sigma^2 / 2}{\beta_t(\delta) (1 + \sigma^2)} = - \rho,
\end{equation}
Since $\beta_t(\delta) \ge 2$ and $\sigma^2 > 0$, $\rho > 0$ hold.

Next, we aim to obtain the proxy variance of the sub-Gaussian martingale difference sequence $M_t = L_t^c(\hat{\btheta}) - \expect_t[L_t^c(\hat{\btheta}]$. We construct the stopped sequence $\tilde{M}_t = M_t \, \mathbf{1}\{ t < \tau \}$. Since $\tau$ is fully determined by $\calF_{t - 1}$ and $\x_t$, $\tilde{M}_t$ remains a valid martingale difference sequence. Let $V_t = \beta_t(\delta)(\sigma^2_{t - 1}(\x_t; \hat{\btheta}) + \sigma^2)$. Then, we express $\tilde{M}_t$ as:
\begin{equation}
\tilde{M}_t = \frac{\epsilon_t^2 - \sigma^2 + 2\Delta_t\epsilon_t}{V_t} \mathbf{1}\{t \leq \tau\}.
\end{equation}
Since $\epsilon_t$ is sub-Gaussian, the term $\epsilon_t^2 - \sigma^2$ is a centered sub-exponential random variable \citep{vershynin2020high}. Additionally, since the stopping time ensures that $\vert \Delta_t \vert$ is bounded, the term $2 \Delta_t \epsilon_t$ has a bounded sub-Gaussian norm, thus, is sub-exponential \citep{vershynin2020high}. Since the addition of sub-exponential results in a sub-exponential random variable, we conclude that $\tilde{M}_t$ exhibits sub-exponential tails. By the property of sub-exponential distribution \citep{vershynin2020high}, there exists $v, c >0$ such that the moment generating function satisfies
\begin{equation}
    \expect[\exp(\eta \tilde{M}_t)] \leq \exp(\eta^2 v^2 / 2), 
\end{equation}
for all $\vert \eta \vert < 1/c$. We then apply Bernstein's inequality for martingales to obtain
\begin{equation}
\probability\left(\sum_{t = 1}^T \tilde{M}_t \geq s \right) \leq \exp\left(- \frac{s^2}{2 (T v^2 + cs)} \right).
\end{equation}
Setting the failure probability $\exp\left( -\frac{s^2}{2(T^2 v + cs)} \right) = \frac{\delta}{6}$, we have that $s = c \log(6/\delta) + \sqrt{c^2 \log^2(6/\delta) + 2 \nu^2 T \log(6/\delta)}$. By applying the inequality $\sqrt{C + D} \leq \sqrt{C} + \sqrt{D}$, it follows that $s \le \sqrt{2 \nu^2 T \log(6/\delta)} + 2 c \log(6/\delta)$. Furthermore, let $\bar{M} = \hat{M} / T$, where $\hat{M} = \sum_{t = 1}^T \tilde{M}_t$. Then $\probability(\bar{M} \geq \bar{s}) = \probability(\hat{M} \geq T \bar{s})$. We then set $\bar{s} = \sqrt{\frac{2 \nu^2 \log(6/\delta)}{T}} + \frac{2 c \log(6/\delta)}{T}$. Therefore, with a probability of at least $1 - \delta / 6$ the following bound holds:
\begin{equation}
    \frac{1}{T} \sum_{t = 1}^T \tilde{M}_t \leq \sqrt{\frac{2 \nu^2 \log(6/\delta)}{T}} + \frac{2 c \log(6/\delta)}{T}.
\end{equation}
Let $E_\mathrm{M}$ denote the event where the average stopped martingale concentrates. If the event $E_\mathrm{CI}$ holds, the stopping time is never triggered, i.e., $\tau = T + 1$. Consequently, we have that $M_t = \tilde{M}_t$, and $\expect_t[L_t^c(\btheta^\ast)] \leq -0.065$ holds universally. When $E_\mathrm{CI}$ and $E_\mathrm{M}$ hold, it follows that
\begin{align}
    \frac{1}{T} \sum_{t = 1}^T L_t^c(\btheta^\ast) &= \frac{1}{T} \sum_{t = 1}^T \expect_t[L_t^c(\btheta^\ast)] + \frac{1}{T} \sum_{t = 1}^T M_t, \\
    &\leq - \rho + \sqrt{\frac{2 \nu^2 \log(6/\delta)}{T}} + \frac{2 c \log(6/\delta)}{T}.
\end{align}
Note that for both cases ($p=1$ and $p=2$), the event $E_\mathrm{M}$ fails with a probability of at most $\delta/6$. Previously, we defined the noise event $E_\mathrm{Y}$ and the confidence bound event $E_\mathrm{CI}$ to also fail with at most $\delta/6$ respectively. By applying the union bound, the events $E_\mathrm{Y}$, $E_\mathrm{CI}$, and $E_\mathrm{M}$ hold simultaneously with a probability of at least $1 - (\delta/6 + \delta/6 + \delta/6) = 1 - \delta/2$. We define this joint high-probability event as $E_\mathrm{safe}$. On the event $E_\mathrm{safe}$, the strict feasibility condition $\frac{1}{T} \sum_{t = 1}^T L_t^c(\btheta^\ast) \leq - \rho$ holds universally.

\subsection{Proof of \cref{lemma:convergence}}

\lemmaconvergence*

\textit{Proof:}

\textbf{Regret bound of FTPL algorithm.} Recall that our FTPL algorithm aims to select $\hat{\btheta}_{t - 1}$ that minimizes the primal utility function
\begin{equation}
    u_t^P(\hat{\btheta}) = L_t^s(\hat{\btheta}) + \lambda_{t - 1} L_t^c(\hat{\btheta}), 
\end{equation}
for each round $t$. Assuming the global safe event $E_\mathrm{Safe}$ holds, the utility is bounded by a constant $M_P = 1 + \bar{M}_P/\hat{\rho} = \bigo(1)$, for some constant $\bar{M}_P > 0$. As established in \cref{lemma:utility-lispchitz-property}, the GP predictive mean causes the Lipschitz constant of the utility function to grow as $K = \bigo(T^2 \sqrt{\log T})$. To maintain sublinear regret, we restrict the FTPL algorithm to select the hyperparameter from the same discrete $\varepsilon_T$-cover used in \cref{lemma:slater-condition}, denoted as $\bTheta_{\varepsilon_T}$. This set contains $N = (D / \varepsilon_T)^d = \bigo(T^{4d})$ points, where $D=\sup_{\btheta,\hat{\btheta}\in\bTheta}\|\btheta-\hat{\btheta}\|_\infty$ is the diameter of $\bTheta$. Let $[\btheta^\ast]_{\varepsilon_T} \in \bTheta_{\varepsilon_T}$ be the closest grid point to the true continuous optimum $\btheta^\ast = \arg \min_{\hat{\btheta} \in \bTheta} \sum_{t = 1}^T u_t^P(\hat{\btheta})$. We decompose the continuous expected regret into the discrete FTPL regret and the continuous discretization error:
\begin{align}
\expect[R_T^P] = \sum_{t = 1}^T \expect_t [u_t^P(\hat{\btheta}_{ t - 1})]  - \sum_{t = 1}^T u_t^P(\btheta^\ast) &\leq \left( \sum_{t = 1}^T \expect_t [u_t^P(\hat{\btheta}_{ t - 1})]  - \sum_{t = 1}^T u_t^P([\btheta^\ast]_{\varepsilon_T}) \right) \nonumber \\
& \quad + \sum_{t=1}^T \vert u_t^P([\btheta^\ast]_{\varepsilon_T}) - u_t^P(\btheta^\ast) \vert. \label{eq:ftpl-decomposition}
\end{align}

Following \cite{suggala2020online}, the first term is bounded by $\bigo(M_P \sqrt{T \log N}) = \tilde{\bigo}(\sqrt{T})$. The continuous discretization error is bounded by the deterministic Lipschitz property as $\sum_{t=1}^T K \Vert [\btheta^\ast]_{\varepsilon_T} - \btheta^\ast \Vert_1 \leq T K \varepsilon_T$. By utilizing the grid resolution $\varepsilon_T = T^{-4}$, this error evaluates to $\bigo(T^3 \sqrt{\log T} \times T^{-4}) = \tilde{\bigo}(T^{-1})$. Therefore, the conditional expected regret satisfies:
\begin{equation}
\expect[R_T^P] = \sum_{t = 1}^T \expect_t[u_t^P(\hat{\btheta}_{ t - 1})] - \sum_{t = 1}^T u_t^P(\btheta^\ast) \leq \tilde{\bigo}(\sqrt{T}) + \tilde{\bigo}(T^{-1}) = \tilde{\bigo}(\sqrt{T}). \label{eq:expected-primal-regret}
\end{equation}
Define $X_t = u_t^P(\hat{\btheta}_{t - 1}) - \expect_t[u_t^P(\hat{\btheta}_{t - 1})]$ at round $t$. Since $L_t^c$ relies on unbounded observation noise, $X_t$ is not almost-surely bounded. We then construct the stopped martingale sequence $\tilde{X}_t = X_t \mathbf{1}\{t < \tau\}$, where $\tau$ is the stopping time defined in \cref{lemma:slater-condition}. We note that the stochastic tail of the stopped sequence $\tilde{X}_t$ is identically governed by the sub-Gaussian ($p=1$) or sub-exponential ($p=2$) properties. 

\textbf{Case 1: $p=1$.} The moment generating function for the stopped increment is bounded by a variance proxy $v^2 = \bigo(1)$, yielding:
\begin{equation}
 \expect[\exp(\eta \tilde{X}_t) \vert \calF_{t - 1}] \leq \exp\left(\frac{\eta^2 v^2}{2} \right).    
\end{equation}
Applying the corresponding martingale concentration inequality for sub-Gaussian provides:
\begin{equation}
\probability \left(\sum_{t = 1}^T \tilde{X}_t \geq s \right) \leq \exp\left(- \frac{s^2}{2 v^2 T}\right). \label{eq:azuma-hoeffding-primal-regret}
\end{equation}
We set $\exp(- s^2 / (2 v^2 T)) =  \delta / 6$ to obtain $s = \sqrt{2 v^2 T \log(6 /  \delta)}$. Assuming the safe event $E_\mathrm{safe}$ holds, the stopping time is never triggered ($\tau = T+1$), meaning $X_t = \tilde{X}_t$ universally. Let $E_\mathrm{FTPL}$ denote this concentration event. Combining \cref{eq:expected-primal-regret} and \cref{eq:azuma-hoeffding-primal-regret}, with a probability of at least $1 -  \delta / 6$, it holds that
\begin{align}
R_T^P &= \sum_{t = 1}^T X_t + \left( \sum_{t = 1}^T \expect_t[u_t^P(\hat{\btheta}_{t - 1}) ] - \sum_{t = 1}^T u_t^P(\btheta^\ast) \right) \\
&\leq \sqrt{ 2 v^2 T \log(6 / \delta)} + \tilde{\bigo}(\sqrt{T}) \\
&\leq \bigo(\sqrt{T \log(1/ \delta)}) + \tilde{\bigo}(\sqrt{T}) \\
&= \tilde{\bigo}(\sqrt{T}).
\end{align}

\textbf{Case 2: $p=2$.} Since $u_t^P = L_t^s + \lambda_{t-1}L_t^c$, and $L_t^s \in [0, 1]$, the stochastic tail of the stopped sequence $\tilde{X}_t$ is identically governed by the sub-exponential properties of the constraint noise. Under the sub-exponential case, the stopped increment admits a variance proxy $v^2 = \mathcal{O}(1)$ and a scale parameter $c = \mathcal{O}(1)$. For any $|\eta| < 1/ c$, the conditional moment generating function is bounded by:$$ \mathbb{E}_t[\exp(\eta \tilde{X}_t)] \leq \exp\left(\frac{\eta^2 v^2}{2} \right). $$Applying the Bernstein concentration inequality for sub-exponential martingale difference sequences provides:$$ \mathbb{P} \left(\sum_{t = 1}^T \tilde{X}_t \geq s \right) \leq \exp\left(- \frac{1}{2} \min\left(\frac{s^2}{v^2 T}, \frac{s}{c}\right)\right). $$We set the right-hand side to $\delta / 6$ to obtain $s = \max\left( \sqrt{2 v^2 T \log(6 / \delta)}, 2 c \log(6 / \delta) \right)$. For sufficiently large $T$, the $\mathcal{O}(\sqrt{T})$ Gaussian regime strictly dominates the logarithmic exponential regime. Assuming the safe event $E_\mathrm{safe}$ holds, the stopping time is never triggered ($\tau = T+1$), meaning $X_t = \tilde{X}_t$ universally. Let $E_\mathrm{FTPL}$ denote this concentration event. Combining this high-probability martingale bound with the conditional expected regret, it holds with a probability of at least $1 - \delta / 6$ that:$$ R_T^P \leq \max\left( \sqrt{2 v^2 T \log(6 / \delta)}, 2 c \log(6 / \delta) \right) + \tilde{\mathcal{O}}(\sqrt{T}) = \tilde{\mathcal{O}}(\sqrt{T}). $$

\textbf{Regret bound of OMD algorithm} Next, the dual RM OMD aims to select $\hat{\lambda}_{t - 1}$ that minimizes the dual utility function
\begin{equation}
    u_t^D(\hat{\lambda}) = - L_t^s(\hat{\btheta}_{t - 1}) - \hat{\lambda} L_t^c(\hat{\btheta}_{t - 1}),
\end{equation}
for each round $t$. By the Slater's condition (\cref{corr:extended-strong-duality}), the domain is restricted to $\Lambda = [0, 1/\hat{\rho}]$. Furthermore, one can confirm that the utility function $u_t^D$ is linear w.r.t. $\hat{\lambda}$, thus implying convexity. On the safe event $E_\mathrm{safe}$, the conditional expected utility $\vert \expect_t[u_t^D] \vert$ is strictly bounded by $M_D = 1 + \bar{M}_D /\hat{\rho} = \bigo(1)$, for some constant $\bar{M}_D > 0$. 

When considering the negative entropy regularization, the function $\psi$ of is $\hat{\rho}$-strongly convex w.r.t. the absolute value norm $\vert . \vert$, i.e., $(\nabla \psi(\lambda) - \nabla \psi(\hat{\lambda})) (\lambda - \hat{\lambda}) \geq  \hat{\rho} \vert \lambda - \hat{\lambda} \vert^2, \forall \lambda, \hat{\lambda} \in \Lambda$. Let the primal norm in the one-dimensional space be the absolute norm $\Vert x \Vert = \vert x \vert$, and let the corresponding dual norm defined by $\Vert g \Vert_\ast = \sup_{\vert x \vert \leq 1} \, x g$. Then, for the scalar gradient $\nabla u_t^D(\lambda_{t - 1})$, it follows that
\begin{equation}\label{eq:primal-dual-norm}
    \Vert \nabla u_t^D(\lambda_{t - 1})  \Vert_\ast =  \vert \nabla u_t^D(\lambda_{t - 1}) \vert \rightarrow \Vert \nabla u_t^D(\lambda_{t - 1})  \Vert^2_\ast =  \vert \nabla u_t^D(\lambda_{t - 1}) \vert^2.
\end{equation}
Assuming that $\omega_1 = \dots = \omega_T = \omega = \sqrt{\frac{2 \hat{\rho} }{T}}$, and applying [Theorem 6.11] of \cite{orabona2019modern}, we have that
\begin{align}
    \sum_{t = 1}^T \expect_t[u_t^D(\lambda_{ t - 1})]  - \min_{\hat{\lambda} \in \Lambda} \sum_{t = 1}^T \expect_t[u_t^D(\hat{\lambda})] &\leq  \frac{D_\psi(\lambda, \hat{\lambda}_1)}{\omega}  + \frac{\omega}{2 \hat{\rho}} \sum_{t = 1}^T \Vert \expect_t[\nabla u_t^D(\lambda_{t - 1})] \Vert^2_\ast, \\
    & \leq \frac{1}{\omega} \, \bigo\left(\frac{1}{\hat{\rho}} \log \left(\frac{1}{\hat{\rho}}\right) \right) + \frac{\omega}{2 \hat{\rho}} \sum_{t = 1}^T \bar{M}_D^2, \\
    & \leq  \frac{1}{\omega} \, \bigo\left(\frac{1}{\hat{\rho}} \log \left(\frac{1}{\hat{\rho}}\right) \right) + \frac{\omega} {2 \hat{\rho}} \, T \bar{M}_D^2, \\
    &=  \tilde{\bigo}(\sqrt{T}). \label{eq:expected-dual-regret}
\end{align}

The second inequality is due to the definition of Bergman divergence and \cref{eq:primal-dual-norm}, the third inequality stems from the fact that $L_t^c$ is bounded within $[-1, 1]$, and the last inequality holds since $\hat{\rho}$ is a constant. Define $Y_t = u_t^D(\lambda_{t - 1}) - \expect[u_t^D(\lambda_{t - 1})]$. Since $\expect[Y_t \vert \calF_{t - 1}] = 0$, then $Y_t$ is a martingale difference sequence. Since $u_t^D$ depends on the unbounded random variable $L_t^c$, we require the stopped martingale sequence $\tilde{Y}_t = Y_t \mathbf{1}\{t < \tau\}$. 

\textbf{Case 1: $p=1$.} The stopped sequence admits a variance proxy $v^2 = \bigo(1)$. For any $\eta \in \real$, we find that  By applying Hoeffding's lemma, for any $\eta \in \real$, we find that
\begin{equation}
 \expect_t[\exp(\eta \tilde{Y}_t)] \leq \exp\left(\frac{\eta^2 v^2}{2} \right).    
\end{equation}
Therefore, the Azuma-Hoeffding concentration inequality for sub-Gaussian increments provides
\begin{equation}
\probability\left( \sum_{t = 1}^T \tilde{Y}_t \geq s \right) \leq \exp\left( - \frac{s^2}{2 v^2 T} \right).
\end{equation}
Setting $\exp(- s^2 / (2 v^2 T)) = \delta / 6$, we obtain $s = \sqrt{2 v^2 \, T \log(6/ \delta )}$. On the event $E_\mathrm{safe}$, the stopping time is never triggered ($\tau = T+1$), so $Y_t = \tilde{Y}_t$ universally. Let $E_\mathrm{OMD}$ denote this concentration event. Combining \cref{eq:expected-dual-regret} with the martingale concentration, we find that with a probability of at least $1 - \delta / 6$:
\begin{align}
    R_T^D &= \sum_{t = 1}^T Y_t + \left( \sum_{t = 1}^T \expect_t[u_t^D(\lambda_{ t - 1})]  - \min_{\hat{\lambda} \in \Lambda} \sum_{t = 1}^T \expect_t[u_t^D(\hat{\lambda})] \right) \\
    &\leq \sqrt{2 v^2 \, T \, \log(6 / \delta)} + \tilde{\bigo}(\sqrt{T}) \\
    &= \tilde{\bigo}(\sqrt{T}).
\end{align}

\textbf{Case 2: $p=2$.} The stopped sequence admits a variance proxy $v^2 = \mathcal{O}(1)$ and a scale parameter $c = \mathcal{O}(1)$. For any $|\eta| < 1/ c$, we find that:$$ \mathbb{E}_t[\exp(\eta \tilde{Y}_t)] \leq \exp\left(\frac{\eta^2 v^2}{2} \right). $$Therefore, the Bernstein concentration inequality for sub-exponential martingales yields:$$ \mathbb{P}\left( \sum_{t = 1}^T \tilde{Y}_t \geq s \right) \leq \exp\left( - \frac{1}{2} \min\left(\frac{s^2}{v^2 T}, \frac{s}{c}\right) \right). $$Setting the right-hand side equal to $\delta / 6$, we obtain $s = \max\left( \sqrt{2 v^2 T \log(6 / \delta)}, 2 c \log(6 / \delta) \right)$. Assuming $E_\mathrm{safe}$ holds, the stopping time is never triggered ($\tau = T+1$), so $Y_t = \tilde{Y}_t$ universally. Let $E_\mathrm{OMD}$ denote this concentration event. Combining the conditional expected dual regret with the martingale concentration, we find that for sufficiently large $T$, the Gaussian tail strictly dominates. Thus, with a probability of at least $1 - \delta / 6$:$$ R_T^D \leq \max\left( \sqrt{2 v^2 T \log(6 / \delta)}, 2 c \log(6 / \delta) \right) + \tilde{\mathcal{O}}(\sqrt{T}) = \tilde{\mathcal{O}}(\sqrt{T}). $$

\textbf{Bounding the cumulative loss and constraint calibration} For any hyperparameter $\hat{\btheta} \in \bTheta$ and any dual variable $\hat{\lambda} \in \Lambda$, the following master inequality holds:
\begin{equation}\label{eq:master-inequality}
\sum_{t = 1}^T L_t^s(\hat{\btheta}_{t - 1}) + \hat{\lambda} L_t^c(\hat{\btheta}_{t - 1}) \leq \sum_{t = 1}^T L_t^s(\hat{\btheta}) + \lambda_{t - 1} L_t^c(\hat{\btheta}) + R_T^P + R_T^D
\end{equation}
We first set $\hat{\btheta} = \btheta$ and $\hat{\lambda} = 0$ to evaluate the cumulative regret against the true hyperparameter $\btheta$:
\begin{equation}\label{eq:cumulative-reward-bound}
    \sum_{t = 1}^T L_t^s(\hat{\btheta}_{t - 1}) + (0) L_t^c(\hat{\btheta}_{t - 1}) \leq \sum_{t = 1}^T L_t^s(\btheta) + \sum_{t = 1}^T \lambda_{t - 1} L_t^c(\btheta) + R_T^P + R_T^D.
\end{equation}
We aim to bound the term $\sum_{t = 1}^T \lambda_{t - 1} L_t^c(\btheta)$. Since $\lambda_{t - 1}$ is deterministic conditioned on $\calF_{t - 1}$, we can utilize the bound of $\expect_t[L_t^c(\hat{\btheta}_{t - 1})]$ from \cref{lemma:slater-condition} to obtain
\begin{equation}\label{eq:bounded-lambda-lct}
\expect_t[\lambda_{t - 1} L_t^c(\hat{\btheta}_{t - 1})] \leq - \lambda_{t - 1} \hat{\rho}.
\end{equation}
Furthermore, we rewrite $\sum_{t = 1}^T \lambda_{t - 1} L_t^c(\btheta)$ as follows:
\begin{align}\label{eq:sum-lambda-lct}
\sum_{t = 1}^T \lambda_{t - 1} L_t^c(\btheta) \leq \sum_{t = 1}^T \lambda_{t - 1} \expect_t[L_t^c(\btheta)] + \sum_{t = 1}^T \lambda_{t - 1} (L_t^c(\btheta) - \expect_t[L_t^c(\btheta)]).
\end{align}
Recall that $M_t = L_t^c(\btheta) - \expect_t[L_t^c(\btheta)]$ and let $Z_t = \lambda_{t - 1} M_t$. From \cref{corr:extended-strong-duality}, the dual variable is bounded such that $0 \leq \lambda_{t - 1} \leq \lambda_{\max} = 1/\hat{\rho}$. We evaluate the concentration of $Z_t$ based on the norm $p$ used in the calibration constraint. 

\textbf{Case 1: $p=1$.} As established in the feasibility proof, for $p=1$, the stopped sequence $\tilde{M}_t$ follows a sub-Gaussian tails with a proxy variance $v^2 = C$. Because $Z_t$ is scaled by $\lambda_{t-1} \in [0, 1/ \hat{\rho}]$, the scaled sequence $\tilde{Z}_t = \lambda_{t - 1} \tilde{M}_t$ remains sub-Gaussian with a maximum proxy variance $v^2 = C / \hat{\rho}^2$. By applying the Azuma-Hoeffding concentration inequality, we obtain:
\begin{equation}
\probability\left( \sum_{t = 1}^T \tilde{Z}_t \leq s \right) \leq \exp\left( - s^2 \hat{\rho}^2 / (2 T C ) \right) 
\end{equation}
By resetting $\exp(- s^2 \hat{\rho}^2 / (2TC)) =  \delta / 6$, we have $s = 1 / \hat{\rho} \sqrt{2 T C \log(6 /  \delta)}$. 

\textbf{Case 2: $p=2$.} As derived previously, there exists $v, c > 0$ such that the moment generating function satisfies $\expect[\exp(\eta \tilde{M}_t)] \leq \exp(\eta^2 v^2 / 2)$ for all $\vert \eta \vert < 1/c$. Since $\tilde{Z}_t = \lambda_{t - 1} \tilde{M}_t$ and $\lambda_{t-1} \le 1/ \hat{\rho}$, the scaled sequence $\tilde{Z}_t$ is also sub-exponential. Its moment generating function scales as:
\begin{equation}
    \expect[\exp(\eta \tilde{Z}_t)] = \expect[\exp(\eta \lambda_{t-1} \tilde{M}_t)] \leq \exp\left(\frac{\eta^2 (v/ \hat{\rho})^2}{2}\right),
\end{equation}
which holds for all $\vert \eta \lambda_{t-1} \vert < 1/c \rightarrow \vert \eta \vert < \hat{\rho} / c$. Thus, $\tilde{Z}_t$ has sub-exponential parameters $v = v/ \hat{\rho}$ and $c = c/ \hat{\rho}$. We apply Bernstein's inequality for martingales to obtain:
\begin{equation}
\probability\left(\sum_{t = 1}^T \tilde{Z}_t \geq s \right) \leq \exp\left(- \frac{s^2}{2 (T v^2 + c s)} \right).
\end{equation}
Setting the failure probability to $\delta/6$ and solving for $s$ yields $s = c \log(6/\delta) + \sqrt{c^2 \log^2(6/\delta) + 2 v^2 T \log(6/\delta)}$. By applying the sub-additive property of the square root ($\sqrt{C + D} \leq \sqrt{C} + \sqrt{D}$), we obtain:
\begin{equation}
    \sum_{t = 1}^T \tilde{Z}_t \leq \frac{1}{\hat{\rho}} \sqrt{2 v^2 T \log(6/\delta)} + \frac{2 c}{\hat{\rho}} \log(6/\delta),
\end{equation}
with probability at least $1 - \delta / 6$. For both $p=1$ and $p=2$, the martingale concentration yields a bound of the form $\tilde{\bigo}(\sqrt{T})$. Let $E_\mathrm{Z}$ denote the event where this concentration of the stopped sequence $\tilde{Z}_t$ holds. Recall that $\tilde{Z}_t$ is the sequence based on the stopping time $\tau$. Let $E_\mathrm{CI}$ be the global event where the uniform confidence interval bound holds for all $t \in [T]$ (established in \cref{lemma:slater-condition}). On the event $E_\mathrm{CI}$, the stopping time is never triggered ($\tau = T + 1$), which deterministically ensures $M_t = \tilde{M}_t$ and $Z_t = \tilde{Z}_t$ for all $t \in [T]$. Assuming $E_\mathrm{CI}$ and $E_\mathrm{Z}$ hold,  substituting the result back to \cref{eq:sum-lambda-lct} and applying \cref{eq:bounded-lambda-lct}, we obtain
\begin{equation}\label{eq:final-sum-lambda-ltc}
\sum_{t = 1}^T \lambda_{t - 1} L_t^c(\btheta) \leq - \hat{\rho} \sum_{t = 1}^T \lambda_{t - 1} + \tilde{\bigo}{\sqrt{T}}. 
\end{equation}
Furthermore, substituting the above result back to \cref{eq:cumulative-reward-bound} provides
\begin{align}
    \sum_{t = 1}^T L_t^s(\hat{\btheta}_{t - 1}) & \leq \sum_{t = 1}^T L_t^s(\btheta) - \hat{\rho} \sum_{t = 1}^T \lambda_{t - 1} + \tilde{\bigo}(\sqrt{T}) + R_T^P + R_T^D, \\
    & \leq \sum_{t = 1}^T L_t^s(\btheta) + \tilde{\bigo}(\sqrt{T}).
\end{align}
The last inequality stems from the fact that $\lambda_{t - 1} \in [0, 1/ \hat{\rho}], \hat{\rho} > 0$, for all $t \in [T]$. Rearranging terms provides the desired results. To obtain the cumulative constraint violation, we take \cref{eq:master-inequality}, and set $\hat{\btheta} = \btheta$ and $\lambda = 1 / \hat{\rho}$:
\begin{equation}
\sum_{t = 1}^T L_t^s(\hat{\btheta}_{t - 1}) + 1/ \hat{\rho}  \, L_t^c(\hat{\btheta}_{t - 1}) \leq  \sum_{t = 1}^T L_t^s(\btheta)  + \sum_{t = 1}^T \lambda_{t - 1} L_t^c(\btheta)  + R_T^P + R_T^D 
\end{equation}
By rearranging the inequality and applying \cref{eq:final-sum-lambda-ltc}, it holds that
\begin{align}
\hat{V}_T &\leq \hat{\rho} \left( \sum_{t = 1}^T L_t^s(\btheta)  - \sum_{t = 1}^T L_t^s(\hat{\btheta}_{t - 1}) - \hat{\rho} \sum_{t = 1}^T \lambda_{t - 1} + \tilde{\bigo}(\sqrt{T})  + R_T^P + R_T^D \right) \\ & \leq \hat{\rho} \, \left(  \sum_{t = 1}^T L_t^s(\btheta)  + \tilde{\bigo}(\sqrt{T})  + R_T^P + R_T^D \right)   
\end{align}
Let $E_\mathrm{FTPL}$ and $E_\mathrm{OMD}$ be the events where the FTPL and OMD regrets $R_T^P$ and $R_T^D$ concentrate, respectively. Each of the events $E_\mathrm{FTPL}$, $E_\mathrm{OMD}$, and $E_\mathrm{Z}$ fails with a probability of at most $\delta/6$. By applying the union bound exclusively over these three native martingale events, they hold simultaneously with a probability of at least $1 - \delta / 2$. Therefore, on the safe environment event $E_\mathrm{safe}$, the bounds are strictly established, completing the proof.

\subsection{Proof of \cref{theorem:regret-bound}}

\theoremregretbound*

\textit{Proof:}
We aim to upper bound
\[
R_T=\sum_{t=1}^T\big(f(\x^\ast)-f(\x_{t})\big).
\]
We first consider the following events:
\begin{itemize}
    \item $E_\mathrm{safe}$: The global safe environment event in \cref{lemma:slater-condition} hold.
    \item $E_\mathrm{regret}$: The algorithm simultaneously achieves bounded cumulative loss $G_T$ and long-term calibration violation $\hat{V}_T$ (\cref{lemma:convergence}). 
\end{itemize}

From our previous derivations, $\probability(E_\mathrm{safe}) \geq 1 - \delta / 2$ and $\probability(E_\mathrm{regret}) \geq 1 - \delta / 2$. By applying a union bound, the successful event $E_\mathrm{safe} \cap E_\mathrm{regret}$ occurs with probability at least $1 - \delta$.
By \cref{lemma:slater-condition} and on the event $E_\mathrm{safe}$, the uniform confidence bound over the continuous domain holds:
\begin{equation}
|f(\x)-\mu_{t-1}(\x;\hat{\btheta}_{t-1})| \le \beta_t^{1/2}(\delta)\,\sigma_{t-1}(\x;\hat{\btheta}_{t-1}) + \bar{C}_L, \qquad \forall \x\in\calX,
\end{equation}
where $\bar{C}_L = \beta_t^{1/2}(\delta) \sqrt{L_{\sigma_{t - 1}^2} \varepsilon_T} + \bar{L}_\mu \varepsilon_T$. Applying this at $\x^\ast$ gives
\begin{equation}
f(\x^\ast) \le \mu_{t-1}(\x^\ast;\hat{\btheta}_{t-1}) + \beta_t^{1/2}(\delta)\,\sigma_{t-1}(\x^\ast;\hat{\btheta}_{t-1}) + \bar{C}_L.
\end{equation}
By the optimism principle of the UCB acquisition function evaluated over the grid, we have:
\begin{equation}
\mu_{t-1}(\x^\ast;\hat{\btheta}_{t-1}) + \beta_t^{1/2}(\delta)\,\sigma_{t-1}(\x^\ast;\hat{\btheta}_{t-1}) \le \mu_{t-1}(\x_t;\hat{\btheta}_{t-1}) + \beta_t^{1/2}(\delta)\,\sigma_{t-1}(\x_t;\hat{\btheta}_{t-1}).
\end{equation}
Applying the confidence bound again at $\x_t$ yields
\begin{equation}
\mu_{t-1}(\x_t;\hat{\btheta}_{t-1}) + \beta_t^{1/2}(\delta)\,\sigma_{t-1}(\x_t;\hat{\btheta}_{t-1}) \le f(\x_t) + 2\beta_t^{1/2}(\delta)\,\sigma_{t-1}(\x_t;\hat{\btheta}_{t-1}) + \bar{C}_L.
\end{equation}
Therefore, the instantaneous regret $r_t = f(\x^\ast)-f(\x_t)$ is bounded by:
\begin{equation}\label{eq:rt_bound_fixed_eps}
r_t \le 2\beta_t^{1/2}(\delta)\,\sigma_{t-1}(\x_t;\hat{\btheta}_{t-1}) + 2\bar{C}_L.
\end{equation}
By applying the inequality $(a+b)^2 \le 2(a^2+b^2)$ for any $a,b>0$, we obtain
\begin{equation}
r_t^2 \le 8\beta_t(\delta)\,\sigma_{t-1}^2(\x_t;\hat{\btheta}_{t-1}) + 8\bar{C}_L^2.
\end{equation}

Let $C_2 := \sigma^{-2} / \log(1+\sigma^{-2})$ and $u_t := \sigma^{-2}\sigma_{t-1}^2(\x_t;\hat{\btheta}_{t-1})$. Since $0 \le u_t \le \sigma^{-2}k_{\hat{\btheta}_{t-1}}(\x_t,\x_t) \le \sigma^{-2}$, and the function $u/\log(1+u)$ is monotonically increasing for $u>0$, we have $u_t \le C_2\log(1+u_t)$. It directly follows that:
\begin{equation}
r_t^2 \le 8\beta_T \sigma^2 C_2 \log\!\bigl(1+\sigma^{-2}\sigma_{t-1}^2(\x_t;\hat{\btheta}_{t-1})\bigr) + 8\bar{C}_L^2.
\end{equation}
Summing over $t=1,\dots,T$ gives
\begin{align}
\sum_{t=1}^T r_t^2 &\le 8\beta_T \sigma^2 C_2 \sum_{t=1}^T \log\!\bigl(1+\sigma^{-2}\sigma_{t-1}^2(\x_t;\hat{\btheta}_{t-1})\bigr) + 8T\bar{C}_L^2, \\
&= 8\beta_T \sigma^2 C_2 \log(1+\sigma^{-2}) \sum_{t=1}^T L_t^s(\hat{\btheta}_{t-1}) + 8T\bar{C}_L^2.
\end{align}
Since $\bar{C}_L = \mathcal{O}(T^{-1} \sqrt{\log T})$,  we have that $8T\bar{C}_L^2 = \mathcal{O}(T^{-1} \log T) = \tilde{\bigo}(T^{-1})$, to be vanishing asymptotically. By the sharpness-transfer inequality derived in \cref{lemma:convergence} and on the event $E_\mathrm{regret}$, it holds that
\begin{equation}
\sum_{t=1}^T L_t^s(\hat{\btheta}_{t-1}) \le \sum_{t=1}^T L_t^s(\btheta) + \tilde{\bigo}(\sqrt T).
\end{equation}
Substituting this into the previous display yields
\begin{equation}
\sum_{t=1}^T r_t^2 \le 8\beta_T \sigma^2 C_2 \log(1+\sigma^{-2}) \left( \sum_{t=1}^T L_t^s(\btheta) + \tilde{\bigo}(\sqrt{T}) \right) + \tilde{\bigo}(T^{-1}).
\end{equation}
Crucially, \cite[Lemma~5.4]{srinivas2009gaussian} is applied only after \cref{lemma:convergence}
has transferred the cumulative sharpness term from the online sequence
$\hat{\btheta}_{t-1}$ to the fixed true kernel $k_{\btheta}$.
By applying \cite[Lemma~5.4]{srinivas2009gaussian}, we obtain
\begin{equation}
\sum_{t=1}^T r_t^2 \le C_1 \beta_T \left( C_3 \gamma_{T}(\btheta) + \tilde{\bigo}(\sqrt{T}) \right).
\end{equation}
where $C_1 = 8\sigma^2 C_2 \log\!\left(1+\frac{1}{\sigma^2}\right), C_3 = 2/\log(1+\sigma^{-2})$. Finally, by applying the Cauchy--Schwarz inequality, $R_T^2 \le T \sum_{t=1}^T r_t^2$. Therefore, with a probability of at least $1-\delta$, the cumulative regret satisfies:
\begin{align}
R_T &\le \sqrt{T \left( C_1 \beta_T \bigl( C_3 \gamma_{T}(\btheta) + \tilde{\bigo}(\sqrt{T}) \bigr) \right)}, \\
&= \tilde{\bigo}\left( \sqrt{T \beta_T \gamma_T(\btheta)} + T^{3/4} \beta_T^{1/2} \right).
\end{align}

Substituting the exploration parameter $\beta_T^{1/2}(\delta) = \tilde{\bigo} \left(B + \sqrt{\Gamma_{T-1} + \log\!\bigl(\mathcal{N}_{\varepsilon_T}(\bTheta, d) / \delta\bigr)}\right)$, and absorbing standard constants into the poly-log notation, we conclude that:
\begin{equation}
R_T \le \tilde{\bigo} \left( \left(B + \sqrt{\Gamma_{T - 1} + \log\!\bigl(\mathcal{N}_{\varepsilon_T}(\bTheta, d) / \delta\bigr)} \right) \left( \sqrt{T \gamma_{T}(\btheta)} + T^{3/4} \right) \right),
\end{equation}
completing the proof.

\section{Experimental Details}\label{sec:expdetails}

\subsection{Baselines}
\label{sec:baseline-details}

We compare OSCBO against several BO baselines that adapt uncertainty in different ways. Unless otherwise stated, all GP-based methods use the same observation model as OSCBO: fixed output variance, fixed observation noise, fixed constant mean, normalized inputs, and lengthscale-only adaptation. All methods use the same initial design, oracle-call budget, and acquisition-optimization settings when applicable.

\textbf{GP-UCB-MLL~\cite{srinivas2009gaussian}.}
The standard GP-UCB baseline with marginal-likelihood hyperparameter refitting. At each BO round, the GP lengthscale is refit by maximizing the marginal log likelihood on all observations collected so far,
whereas OSCBO updates the lengthscale through the sharpness--calibration game.
The next query is selected by maximizing
\[
\mu_t(\x) + \sqrt{\beta}\sigma_t(\x),
\]
with fixed $\beta=2.0$.
%For consistency with our experimental setting, we use a scalar lengthscale version and keep amplitude
%and noise fixed.

\textbf{OCBO~\cite{deshpande2024calibratedregressionadversaryregret}.}
An online calibration baseline that performs post-hoc recalibration of GP predictive quantiles.
At each BO round, the GP lengthscale is first refit by MLL. Leave-one-out predictive distributions on the observations collected so far are formed, and their probability integral transform values are computed
\[
u_i = \Phi\!\left(
\frac{y_i-\mu_{-i}(\x_i)}
{\sqrt{\sigma^2_{-i}(\x_i)+\sigma_n^2}}
\right).
\]
The calibrated upper-quantile level is set to the empirical $p$-quantile of these values, with $p=1-\delta$.
The acquisition is then
\[
\mu_t(\x) + \Phi^{-1}(q_t)\sigma_t(\x),
\]
where $\sigma_t(\x)$ is the latent posterior standard deviation. Thus, OCBO recalibrates the scalar UCB width through an empirical quantile map, while the GP lengthscale remains MLL-driven.

% \textbf{OCBO~\cite{deshpande2024calibratedregressionadversaryregret}.}  An online calibration baseline that adapts the UCB confidence level through a quantile update. The GP lengthscale is first refit by MLL, then leave-one-out predictive distributions are used to update an internal quantile $q_t$ toward the target coverage level $p=1-\delta$. The UCB multiplier is then set to
% \[
% \beta_t = \max\{(\Phi^{-1}(q_t))^2,0.1\}.
% \]
% Thus, OCBO changes the acquisition width through an adaptive scalar confidence multiplier, while the GP lengthscale itself remains MLL-driven.

\textbf{A-GP-UCB~\cite{berkenkamp2019no}.}
An adaptive lengthscale schedule baseline. It initializes the lengthscale by MLL on the initial design and then progressively shrinks it according to a deterministic schedule,
\[
\theta_t \leftarrow \max\{\theta_{0}/g_t,\theta_{\min}\}.
\]
Unlike OSCBO, this method does not use calibration feedback; the lengthscale evolution is driven only by the prescribed schedule.

\textbf{LB-GP-UCB~\cite{ziomek2024bayesian}.}
 An evolution of A-GP-UCB, which maintains a candidate set of lengthscales and eliminates candidates whose lower-confidence performance is not competitive. The initial lengthscale is obtained by MLL on the initial design, after which the method constructs shorter candidate lengthscales according to its schedule and selects among them using confidence-corrected empirical rewards. We use the same fixed-noise, fixed-amplitude GP model as in the rest of the experiments.

\textbf{SCGP~\cite{capone2023sharp}.}
A sharp calibrated GP baseline. It splits the available observations into a regression set and a calibration set. The regression set is used to fit the posterior mean lengthscale by MLL, while the calibration set is used to learn a separate scalar uncertainty lengthscale and calibrated upper-quantile multiplier. Although the model was not tested on BO in the original publication, we extend it to BO by plugging this calibrated uncertainty into UCB as
\[
\mu_t(\x) + \beta_\delta \sigma_t^{\mathrm{cal}}(\x).
\]
%
%For consistency with our experimental setting, we use a scalar lengthscale version and keep amplitude and noise fixed.

\textbf{TabICLv2~\cite{qu2026tabiclv2}.}
A tabular foundation model baseline. We use TabICLv2 as a tabular regression surrogate inside the BO loop. At each round, the current normalized BO dataset is passed to the TabICLv2 regressor, which returns predictive quantiles for candidate points. We convert these quantiles into an empirical predictive mean and variance, and then optimize a UCB acquisition,
\[
\widehat{\mu}_t(\x)+\sqrt{\beta}\widehat{\sigma}_t(\x),
\]
directly in the normalized input domain using Sobol initial points followed by Adam updates.

\subsection{Hyperparameters}\label{sec:hyperparameter-details}

\begin{table}[H]
\centering
\tiny
\setlength{\tabcolsep}{3.8pt}
\renewcommand{\arraystretch}{0.87}
\begin{tabular}{@{}ll@{}}
\toprule
\multicolumn{2}{l}{\textbf{General}} \\ \midrule
Kernel & Matérn, $\nu=2.5$\\
Lengthscale & isotropic \\
Output variance & fixed to $1$ \\
Mean & fixed constant mean $0$ \\
$\noisevar$ & $0.01$ \\
$\vert \dataset_0 \vert$ & $10$ \\
$T$ & $100$ BO steps, $110$ total oracle calls \\
Acquisition & UCB \\
UCB $\beta$ & $2.0$ unless stated otherwise \\
GP refit optimizer & Adam, lr $=0.01$, $50$ steps \\
\texttt{BoTorch} acquisition optimizer & $q=1$, $5$ restarts, $20$ raw samples \\

\midrule
\multicolumn{2}{l}{\textbf{OSCBO}} \\ \midrule
$\hat{\rho}$ & $0.5$ \\
$\delta$ & $0.1$ \\
$\beta$ & $2.0$ \\
Primal RM lr & $0.01$ \\
Dual RM lr & $0.001$ \\
$\tilde{\rho}$ & $\max(150^{-1/4}, \hat{\rho}/2)$ \\
Bregman divergence & negative entropy \\
FTPL perturbation & $\xi_j \sim \mathcal{N}(0, 0.1^2)$ i.i.d. for each lengthscale dimension \\

\midrule
\multicolumn{2}{l}{\textbf{GP-UCB-MLL~\cite{srinivas2009gaussian}}} \\ \midrule
$\beta$ & $2.0$ \\
Hyperparameter update & MLL refit at every BO step \\

\midrule
\multicolumn{2}{l}{\textbf{SCGP~\cite{capone2023sharp}}} \\ \midrule
Confidence level & $1-\delta=0.9$ \\
Train/calibration split & calibration fraction $0.2$ \\
Minimum train/calibration sizes & $3$ / $2$ \\
Regression lengthscale & MLL fit, Adam lr $=0.01$, $50$ steps \\
Calibration lengthscale & Adam lr $=0.05$, $50$ steps, $5$ restarts \\
Calibration variance & includes observation noise \\

% \midrule
% \multicolumn{2}{l}{\textbf{OCBO~\cite{deshpande2024calibratedregressionadversaryregret}}} \\ \midrule
% $p$ & $1-\delta=0.9$ \\
% $\eta$ & $1.0$ \\
% Quantile update optimizer & Adam, lr $=0.01$, $50$ steps \\
% UCB $\beta_t$ & $\max\{(\Phi^{-1}(q_t))^2,0.1\}$ \\

\midrule
\multicolumn{2}{l}{\textbf{OCBO~\cite{deshpande2024calibratedregressionadversaryregret}}} \\ \midrule
$p$ & $1-\delta=0.9$ \\
GP hyperparameter update & MLL refit at every BO step \\
Recalibration set & leave-one-out probability integral transform values on current BO data \\
Calibrated level & $q_t=\operatorname{Quantile}_{p}(\{u_i\})$ \\
UCB form & $\mu_t(x)+\Phi^{-1}(q_t)\sigma_t(x)$ \\
Minimum calibration size & $3$ \\
%Probability clipping & $[10^{-4},\,1-10^{-4}]$ \\

\midrule
\multicolumn{2}{l}{\textbf{LB-GP-UCB~\cite{ziomek2024bayesian}}} \\ \midrule
$\delta$ & $0.1$ \\
$t_0$ & $5$ \\
$B$ & $1.0$ \\
$\theta_{\min}$ & $10^{-4}$ \\
% Initial lengthscale & MLL fit from $\dataset_0$ \\
% $g_t$ & $1$ if $t\leq t_0$, and $\max\{t_0,\sqrt{t}\}$ otherwise \\
% Update rule & lower-confidence lengthscale elimination over candidate set $\Theta_t$ \\
Initial lengthscale & MLL fit from $\dataset_0$, denoted $\theta_0$ \\
Candidate grid & $q(i)=\theta_0 e^{-i/d}$ \\
Candidate addition & add $q(\ell+1)$ when $q(\ell+1)\leq \theta_0/g_t$ \\
Selection rule & $\theta_t=\arg\min_{\theta\in\Theta_t}R^\theta(|S_{t-1}^\theta|+1)$ \\
Update rule & eliminate underperforming candidates from $\Theta_t$ using the lower-confidence test \\

\midrule
\multicolumn{2}{l}{\textbf{A-GP-UCB~\cite{berkenkamp2019no}}} \\ \midrule
$t_0$ & $5$ \\
$\theta_{\min}$ & $10^{-4}$ \\
Initial lengthscale & MLL fit from $\dataset_0$, denoted $\theta_0$ \\
$g_t$ & $1$ if $t\leq t_0$, and $\sqrt{t}$ otherwise \\
Update rule & $\theta_t \leftarrow \max\{\theta_0/g_t,\theta_{\min}\}$ \\

\midrule
\multicolumn{2}{l}{\textbf{TabICLv2~\cite{qu2026tabiclv2}}} \\ \midrule
Surrogate & TabICLv2 regression quantile model \\
Checkpoint & \texttt{tabicl-regressor-v2-20260212.ckpt} \\
Acquisition & UCB \\
UCB $\beta$ & $2.0$ \\
Acquisition optimizer & Adam, lr $=0.05$, $50$ steps \\
Restarts / raw samples & $5$ / $20$ \\
Batch size & $32$ \\

\bottomrule
\end{tabular}
\caption{Hyperparameter settings used in the experiments. GP amplitude, mean, and observation noise are fixed for all GP-based methods; only lengthscale parameters are adapted.}
\label{tab:hyperparameters}
\end{table}

\subsection{Kernels}\label{app:ibnn}

\paragraph{Matérn kernel.}
In the main experiments, we consider the Matérn kernel
\[
k_{\mathrm{Mat\acute{e}rn}}(\x,\x')
=
\sigma_f^2
\frac{2^{1-\nu}}{\Gamma(\nu)}
\left(\sqrt{2\nu\, r^2(\x,\x')}\right)^\nu
K_\nu\!\left(\sqrt{2\nu\, r^2(\x,\x')}\right),
\]
where $\nu>0$ controls smoothness, $K_\nu$ is the modified Bessel function of the second kind, and
\[
r^2(\x,\x')
=
(\x-\x')^\top \boldsymbol{\Lambda}^{-1}(\x-\x').
\]
$\sigma_f^2$ is fixed. In the isotropic setting, $\boldsymbol{\Lambda}=\theta^2 \mathbf{I}$; with ARD, $\boldsymbol{\Lambda}=\mathrm{diag}(\theta_1^2,\ldots,\theta_d^2)$.

\subsection{Hardware}\label{sec:hardware-details}

All experiments were conducted on a compute cluster equipped with four NVIDIA V100 GPUs (32 GB memory each), as well as on a MacBook Pro with an Apple M4 Pro chip.

\section{Test functions}\label{sec:testfunc}

\paragraph{Preprocessing (input/output normalization).}
Across all experiments, we use BoTorch's standard transforms to normalize inputs and standardize outputs~\cite{balandat2020botorch}. Inputs are mapped to the unit hypercube via \texttt{Normalize} (applied per dimension), and observed function values are standardized via \texttt{Standardize} before GP fitting, as is common practice for stable hyperparameter learning and comparable kernel scales across tasks.

\subsection{Synthetic test functions.}\label{sec:sb}
We use three standard synthetic benchmarks.

\textbf{Levy ($d{=}5$).}
For $x\in[-10,10]^d$, define $w_i = 1 + \frac{x_i-1}{4}$. We use the negated Levy function as a maximization objective,
\begin{align}
f_{\mathrm{Levy}}(x)
&= -\Bigg[
\sin^2(\pi w_1)
+ \sum_{i=1}^{d-1} (w_i-1)^2\left[1+10\sin^2\!\big(\pi w_i+1\big)\right]
+ (w_d-1)^2\left[1+\sin^2(2\pi w_d)\right]
\Bigg].
\end{align}

\textbf{Hartmann 3 ($d{=}3$).}
For $x\in[0,1]^3$, the Hartmann function is
\begin{align}
f_{\mathrm{H3}}(x)
&= - \sum_{i=1}^{4}\alpha_i \exp\!\left(-\sum_{j=1}^{3} A_{ij}(x_j-P_{ij})^2\right),
\end{align}
with
\[
\alpha = (1.0,\,1.2,\,3.0,\,3.2),
\qquad
A=
\begin{pmatrix}
3 & 10 & 30\\
0.1 & 10 & 35\\
3 & 10 & 30\\
0.1 & 10 & 35
\end{pmatrix},
\]
\[
P=
\begin{pmatrix}
0.3689 & 0.1170 & 0.2673\\
0.4699 & 0.4387 & 0.7470\\
0.1091 & 0.8732 & 0.5547\\
0.03815& 0.5743 & 0.8828
\end{pmatrix}.
\]

\textbf{Hartmann 6 ($d{=}6$).}
For $x\in[0,1]^6$,
\begin{align}
f_{\mathrm{H6}}(x)
&= - \sum_{i=1}^{4}\alpha_i \exp\!\left(-\sum_{j=1}^{6} A_{ij}(x_j-P_{ij})^2\right),
\end{align}
with
\[
\alpha = (1.0,\,1.2,\,3.0,\,3.2),
\qquad
A=
\begin{pmatrix}
10 & 3  & 17 & 3.5 & 1.7 & 8\\
0.05& 10 & 17 & 0.1 & 8   & 14\\
3  & 3  & 1.7& 10  & 17  & 8\\
17 & 8  & 0.05&10  & 0.1 & 14
\end{pmatrix},
\]
\[
P=
\begin{pmatrix}
0.1312 & 0.1696 & 0.5569 & 0.0124 & 0.8283 & 0.5886\\
0.2329 & 0.4135 & 0.8307 & 0.3736 & 0.1004 & 0.9991\\
0.2348 & 0.1451 & 0.3522 & 0.2883 & 0.3047 & 0.6650\\
0.4047 & 0.8828 & 0.8732 & 0.5743 & 0.1091 & 0.0381
\end{pmatrix}.
\]

% \textbf{Styblinski--Tang ($d{=}2$).}
% For $x\in[-5,5]^d$, the Styblinski--Tang function is
% \begin{align}
% f_{\mathrm{ST}}(x)
% &= \frac{1}{2}\sum_{i=1}^{d}\left(x_i^4 - 16x_i^2 + 5x_i\right).
% \end{align}

\subsection{Real-world benchmarks}\label{sec:rwb}

\textbf{Lunar Lander ($d{=}12$)}. The goal of this task is to find an optimal 12-dimensional control policy that allows an
autonomous lunar lander to consistently land without crashing. The final objective value we optimize is the
reward obtained by the policy averaged over a set of 50 random landing terrains. For this task, we use the same
controller setup used by~\cite{eriksson2019scalable}.

\subsection{Real-world tabular benchmarks}\label{sec:rwtb}

\textbf{Continuous oracle construction.}
For the tabular real-world benchmarks, we construct a deterministic continuous oracle from the observed experimental table, following~\cite[Appendix C]{yuan2026unleashing}. The search bounds are set to the coordinatewise minimum and maximum values observed in the table. Given a query $\x$, we clip it to these bounds and normalize each coordinate to $[0,1]$. If the observed designs form a complete rectangular grid, we use multilinear interpolation. Otherwise, we use $k$-nearest-neighbor inverse-distance weighting in normalized input space:
\[
\widehat f(\x)
=
\frac{\sum_{i \in \mathcal{N}_k(x)} w_i(\x)y_i}
{\sum_{i \in \mathcal{N}_k(x)} w_i(\x)},
\qquad
w_i(\x)
=
\left(\|\tilde \x-\tilde \x_i\|_2+\varepsilon\right)^{-p},
\]
with $k=12$, $p=2$, and $\varepsilon=10^{-12}$. If the query coincides with an observed design, the oracle returns the averaged observed value at that design. We use this protocol for Crossbarrel, Concrete, and Material.

\textbf{Material ($d{=}5$)}.
This task is adapted from the benchmark introduced by~\cite{mekki2021two} and later used by~\cite{ziomek2024bayesian}. It originates from a microfluidic nanomaterial-synthesis setting for discovering chemical “recipes” that yield silver nanoparticles with desired optical properties (absorbance spectrum). In our benchmark, this setting is distilled into a five-dimensional material design problem derived from experimental measurements and posed in a pool-based regime with only 164 available candidate compositions. The objective models a noisy performance metric linked to material composition. Following the original setup, we treat the task as continuous optimization and report both simple and cumulative regret over sequential evaluations

\begin{table}[h]
\centering
\caption{Material benchmark specification.}
\label{tab:material-spec}
\begin{tabular}{llll}
\toprule
Variable & Type & Range & Description \\
\midrule
$Q_{\mathrm{AgNO_3}}$ & continuous & $[4.53, 42.8098]$ & silver nitrate flow-rate ratio (\%) \\
$Q_{\mathrm{PVA}}$ & continuous & $[9.9995, 40.0010]$ & polyvinyl alcohol flow-rate ratio (\%) \\
$Q_{\mathrm{TSC}}$ & continuous & $[0.5, 30.5]$ & trisodium citrate flow-rate ratio (\%) \\
$Q_{\mathrm{seed}}$ & continuous & $[0.4989, 19.5]$ & silver-seed flow-rate ratio (\%) \\
$Q_{\mathrm{total}}$ & continuous & $[200.0, 983.0]$ & total oil/aqueous flow rate ($\mu$L/min) \\
Spectrum score & objective & maximize & similarity to target absorbance spectrum \\
\bottomrule
\end{tabular}
\end{table}

\textbf{Concrete ($d{=}7$).}
This task uses the Concrete compressive-strength benchmark from the UCI repository, following~\cite[Appendix C]{yuan2026unleashing}. Each input is a continuous mixture-design variable describing the amount of a concrete component, measured in kg/m$^3$. We exclude curing age, leaving seven mixture variables, and maximize the compressive strength. Since the dataset is tabular rather than directly queryable at arbitrary continuous inputs, we use the same continuous-oracle construction as for the other tabular real-world tasks.

\begin{table}[h]
\centering
\caption{Concrete benchmark specification.}
\label{tab:concrete-spec}
\begin{tabular}{llll}
\toprule
Variable & Type & Range & Description \\
\midrule
Cement & continuous & $[102.0, 540.0]$ & kg/m$^3$ \\
Blast furnace slag & continuous & $[0.0, 359.4]$ & kg/m$^3$ \\
Fly ash & continuous & $[0.0, 200.1]$ & kg/m$^3$ \\
Water & continuous & $[121.8, 247.0]$ & kg/m$^3$ \\
Superplasticizer & continuous & $[0.0, 32.2]$ & kg/m$^3$ \\
Coarse aggregate & continuous & $[801.0, 1145.0]$ & kg/m$^3$ \\
Fine aggregate & continuous & $[594.0, 992.6]$ & kg/m$^3$ \\
Compressive strength & objective & maximize & concrete strength \\
\bottomrule
\end{tabular}
\end{table}

\textbf{Crossbarrel ($d{=}4$).}
This task is based on the crossed-barrel structural-design benchmark of~\cite{gongora2020bayesian}, following~\cite[Appendix C]{yuan2026unleashing}. The dataset describes 3D-printed crossed-barrel structures whose geometry is controlled by four variables: the number of hollow columns $n$, the twist angle $\theta$, the outer radius $r$, and the wall thickness $t$. After printing, each structure is tested under uniaxial compression, and the objective is the mechanical toughness, measured as the area under the resulting force--displacement curve. We maximize toughness.

\begin{table}[h]
\centering
\caption{Crossbarrel benchmark specification.}
\label{tab:crossbarrel-spec}
\begin{tabular}{llll}
\toprule
Variable & Type & Range & Description \\
\midrule
$n$ & continuous & $[6.0, 12.0]$ & number of hollow columns \\
$\theta$ & continuous & $[0.0, 200.0]$ & twist angle of the columns (degrees) \\
$r$ & continuous & $[1.5, 2.5]$ & outer radius of the columns (mm) \\
$t$ & continuous & $[0.7, 1.4]$ & wall thickness of the hollow columns (mm) \\
toughness & objective & maximize & mechanical toughness \\
\bottomrule
\end{tabular}
\end{table}

%%%%%%%%%%%%%%%%%%%%%%%%%%%%%%%%%%%%%%%%%%%%%%%%%%%%%%%%%%%%
\clearpage
\newpage

\end{document}